\documentclass[11pt]{article}

\PassOptionsToPackage{table,xcdraw}{xcolor}

\usepackage{arxiv}

\usepackage{comment}
\usepackage[caption=false,subrefformat=parens,labelformat=parens]{subfig}

\def\E{{\mathbb E}}
\def\Var{{\rm Var}}
\def\KL{{\rm{KL}}}
\def\Cov{{\rm Cov}}
\def\Corr{{\rm Corr}}
\def\Diag{{\rm Diag}}
\def\Bias{{\rm Bias}}

\def\P{\mathbb P}
\def\Tr{{\rm Tr}}
\def\dd{\mathrm d}

\usepackage{mathtools}
\newcommand{\defeq}{\vcentcolon=}

\usepackage{booktabs}
\usepackage{chemformula}
\usepackage{graphicx}%
\usepackage{multirow}%
\usepackage{amsmath,amssymb,amsfonts}%
\usepackage{mathrsfs}%
\usepackage{xcolor}%
\usepackage{textcomp}%
\usepackage{manyfoot}%
\usepackage{booktabs}%
\usepackage{listings}%
\usepackage[ruled]{algorithm2e}
\usepackage[T1]{fontenc}
\usepackage{amsthm}
\usepackage{newtxtext,newtxmath}
\usepackage{anyfontsize}
\usepackage{siunitx}
\usepackage{hyperref}
\usepackage[title,page]{appendix}
\usepackage{natbib}
\allowdisplaybreaks
 \bibpunct[, ]{(}{)}{,}{a}{}{,}%
\theoremstyle{plain}
\newtheorem{theorem}{Theorem}
\newtheorem{lemma}[theorem]{Lemma} 
\newtheorem{corollary}[theorem]{Corollary}
\newtheorem{proposition}[theorem]{Proposition}
\newtheorem{definition}[theorem]{Definition}
\theoremstyle{remark}
\newtheorem{remark}[theorem]{Remark}
\title{Variance Reduction Based Experience Replay for Policy Optimization}

\author{
  Hua Zheng\\
  Department of Mechanical and Industrial Engineering\\
  Northeastern University, Boston, USA\\
  \texttt{zheng.hua1@northeastern.edu}
  \And
  Wei Xie\thanks{Corresponding author. Email: \texttt{w.xie@northeastern.edu}}\\
  Department of Mechanical and Industrial Engineering\\
  Northeastern University, Boston, USA\\
  \texttt{w.xie@northeastern.edu}
  \AND
  M. Ben Feng\\
  Department of Statistics and Actuarial Science\\
  University of Waterloo, Waterloo, Canada\\
  \texttt{ben.feng@uwaterloo.ca}
  \And
  Keilung Choy\\
  Department of Mechanical and Industrial Engineering\\
  Northeastern University, Boston, USA\\
  \texttt{choy.k@northeastern.edu}
}

\date{}

\begin{document}

\maketitle

\begin{abstract}
Effective reinforcement learning (RL) for complex stochastic systems requires leveraging historical data to improve sample efficiency and accelerate policy optimization. However, classical experience replay treats all past observations uniformly and fails to account for their varying contributions to learning.
To address this limitation, we propose Variance Reduction Experience Replay (VRER), a principled framework that selectively reuses informative samples to reduce the variance of policy gradient estimates. VRER is algorithm-agnostic and can be integrated with existing policy optimization methods, yielding the sample-efficient off-policy algorithm, Policy Gradient with VRER (PG-VRER). 
To provide rigorous theoretical guarantees, we develop a novel analysis framework for experience replay that explicitly accounts for dependencies induced by Markovian dynamics and behavior-policy interactions. 
Using this framework, we establish finite-time convergence guarantees for PG-VRER and characterize a fundamental bias-variance trade-off: reusing older samples reduces gradient variance but may introduce greater estimation bias.
Extensive experiments show that VRER consistently accelerates learning and outperforms state-of-the-art policy optimization algorithms.

\end{abstract}


\section{Introduction}
\label{sec:introduction} 

In recent years, various policy optimization approaches are developed to solve challenging control problems.
These approaches often consider parametric policies and search for optimal solution through policy gradient approach~\citep{sutton2018reinforcement}, whose performance and convergence crucially depends on the accuracy of gradient estimation.
Reusing historical samples 
is one way to improve gradient estimation, especially when historical data is scarce.
In this paper, we address an important question in policy optimization methods: \textit{How to intelligently select and reuse historical samples to accelerate the learning of the optimal policy for complex stochastic systems?}

According to the different basic unit of historical samples to reuse, 
policy gradient (PG)
algorithms can be classified into episode-based and step-based approaches~\citep{metelli2020importance}.
Episode-based approaches reuse historical trajectories through  importance sampling (IS) strategy accounting for the distributional difference induced by the target and 
behavior policies.
In this case, the importance sampling weight is built on the product of likelihood ratios (LR) of state-action transitions occurring within each process trajectory. 
As a result, the LR-weighted observations can have large or even infinite variance, especially for problems with long planning horizons~\citep{
schlegel2019importance}. 
On the other hand, the step-based approaches take individual state-action transitions as the basic reuse units.
This overcomes the limitation of episode-based approaches, provides a more flexible reuse strategy, and supports process online control.
\textit{In this paper, we will develop a variance-reduction based experience replay framework applicable to step-based approaches.}

A noticeable limitation of current reinforcement learning (RL) techniques is the low sample efficiency \citep{botvinick2019reinforcement}. This challenge becomes especially pronounced in complex stochastic systems, such as healthcare~\citep{hall2012handbook, zheng2021personalized} and biopharmaceutical manufacturing~\citep{zheng2021policy,Zeng2026GlycoMAC,KeqiIVT}, where each experiment can be expensive.
Traditional on-policy methods only utilize newly generated samples to estimate the policy gradient for each policy update.
Ignoring the relevant information carried with historical samples can lead to low sample efficiency and high uncertainty in policy gradient estimation. In light of these considerations, it becomes important to fully utilize all available information when addressing RL optimization problems.

Fortunately, this information loss can be mitigated through a combination of experience replay (ER)~\citep{
mnih2015humanlevel,wang2017sample} and off-policy optimization methods, which store and replay past relevant experiences to accelerate policy improvement. To correct for the distribution mismatch between the behavior and target policies, off-policy methods typically rely on importance sampling (IS) \citep{Owen2013monte}. 
However, IS often suffers from high variance when the two policies differ substantially \citep{
zheng2020green}. Moreover, the theoretical foundations of ER—particularly the effects of buffer size, policy age, and mixing rates—remain insufficiently understood. To address this, we develop a novel framework that characterizes a fundamental bias--variance trade-off in replay-based learning: reusing older samples introduces bias but reduces gradient estimation variance. 
Building on this insight, we propose Variance Reduction Experience Replay (VRER). VRER selectively reuses historical samples to reduce gradient variance and accelerate convergence, while employing adaptive downsampling and controlled replay capacity to mitigate the bias typically caused by policy lag and correlated data.

The key contributions of this study are summarized as follows.
First, we propose a VRER framework for policy gradient optimization that accelerates policy learning through the selective reuse of informative historical samples. By prioritizing relevant past experiences, VRER improves sample efficiency while reducing the variance of gradient estimates.
Second, we develop a novel finite‑time convergence analysis for experience replay that explicitly accounts for sample dependencies induced by Markovian noise and behavior‑policy interdependence. This framework quantifies how behavior policy age, state-transition mixing dynamics, and replay buffer size jointly influence learning performance through the fundamental bias-variance tradeoff in policy gradient estimation. Leveraging this analysis, we establish rigorous finite-time convergence guarantees for the proposed PG-VRER algorithm.
Third, extensive empirical studies demonstrate that VRER effectively utilize historical samples to reduce policy-gradient variance and significantly accelerate convergence across a range of complex stochastic control problems.

The remainder of this paper is organized as follows. Section~\ref{sec:literatureReview} reviews the related literature. Section~\ref{sec:policyGradientRL} formulates the policy-gradient optimization problem for infinite-horizon Markov decision processes (MDPs) and introduces the regularity assumptions. Section~\ref{sec:MLR} presents the LR and CLR policy-gradient estimators for step-based policy optimization. Section~\ref{sec: theory of experience replay} analyzes replay-induced bias and the dependence-aware variance of policy-gradient estimators. 
Section~\ref{sec:vrer} develops the exact variance-based selection rule, its low-cost KL-based practical surrogate, and the proposed PG-VRER algorithm.
Section~\ref{sec: convergence analysis} establishes finite-time convergence guarantees for policy gradient optimization that account for sample correlation, Markovian mixing, replay buffer size, and replay-induced bias. Finally, Section~\ref{sec:empiricalStudy} presents an empirical evaluation of learning performance, reuse patterns, parameter sensitivity, gradient-variance behavior, the KL-based approximation, and the replay bias--variance trade-off. The implementation of VRER is available at \href{https://github.com/zhenghuazx/vrer\_policy\_gradient}{GitHub}.


\section{Related Works}
\label{sec:literatureReview}

RL seeks to learn an optimal policy through interactions with the systems of interest. Policy-gradient methods are a widely used approach for solving RL problems, beginning with REINFORCE (or vanilla policy gradient, VPG)~\citep{williams1992simple} and later extending to actor--critic frameworks that incorporate value-function estimation~\citep{Konda1999actor}.
Recent advances have focused on improving sample efficiency, scalability, and convergence speed, leading to methods such as 
Asynchronous Advantage Actor-Critic (A3C) \citep{mnih2016asynchronous}, 
trust region policy optimization (TRPO)~\citep{schulman2015trust}, and the proximal policy optimization (PPO)~\citep{schulman2017proximal}. 
A central principle underlying many modern policy-gradient algorithms is to prevent excessively large policy updates. TRPO enforces this through a trust-region constraint based on KL divergence, while PPO achieves a similar effect by clipping likelihood ratios between successive policies.

Importance sampling (IS) is a fundamental tool for off-policy evaluation and correction, but it often suffers from high variance. To mitigate this issue, prior work has proposed variance-reduction techniques such as importance-weight truncation~\citep{espeholt2018impala}, performance-aware surrogate objectives~\citep{metelli2020importance}, and multiple importance sampling for improved robustness~\citep{papini2019optimistic}. Despite these advances, applying IS to step-based policy optimization remains challenging due to two fundamental issues:  Markovian noise, which induces sample dependence, and behavior policy interdependencies, which introduce bias across optimization iterations \citep{eckman2018reusing,liu2020simulation}. 
Although \citet{lin2023reusing} studied the asymptotic convergence of policy reuse, their analysis assumes independent sampling from the occupancy measure and therefore does not capture the effects of Markovian noise.

Experience replay (ER), including prioritized experience replay~\citep{Schaul2016PrioritizedER}, is widely used in RL to improve sample efficiency and mitigate sample correlation. Prior studies have shown that replay-buffer design substantially affects learning performance, with factors such as buffer size, replay capacity, and policy age playing critical roles~\citep{zhang2017deeper,fedus2020revisiting}. More recent work has explored adaptive replay strategies to improve data utilization~\citep{sun2020attentive}. Despite these advances, the theoretical understanding of experience replay remains limited.

VRER differs from existing sample-reuse methods in its selection principle. Uniform and recency-based replay reuse samples according to availability or age, without explicitly considering their impact on policy-gradient variance. Prioritized ER ranks samples by temporal-difference (TD) error and is primarily designed for value-function learning \citep{Schaul2016PrioritizedER}. Methods such as ACER, V-trace, POIS, and multiple IS instead focus on stabilizing off-policy updates through importance weighting, truncation, or trust-region corrections \citep{metelli2018policy}. In contrast, VRER is an algorithm-agnostic replay framework that selectively reuses historical samples based on a variance-reduction criterion. As a result, it can be seamlessly integrated with a wide range of policy-optimization algorithms, including PPO, TRPO, and A2C \citep{wang2017sample}.


\section{Problem Description}
\label{sec:policyGradientRL}

\textbf{Notation}: Let $\rho(x,y)=\mbox{Cov}\left(x,y\right)/\left(\sqrt{\Var(x)}\sqrt{\Var(y)}\right)$ denote the correlation coefficient between random variable $x$ and $y$. Let $\Diag(x_1,\ldots,x_n)$ denote the $n\times n$ diagonal matrix with entries $x_1,\ldots,x_n$.
For any $A\in\mathbb{R}^{d\times d}$, $\Tr({A})=\sum_{i=1}^d A_{i,i}$. 
For $\pmb{x}\in\mathbb{R}^d$, let $\Vert\pmb{x}\Vert
\defeq
\Vert\pmb{x}\Vert_2 
=\sqrt{\sum_{i=1}^d x_i^2}$. Let $P$ and $Q$ be probability measures with densities $p$ and $q$, respectively.
Their total variation (TV) distance on $\mathcal{X}$ is $\Vert P-Q\Vert_{TV}= 1/2\int_{\mathcal{X}}|P(dx)- Q(dx)|=1/2\int_{\mathcal{X}}|p(x)- q(x)|dx$. Finally,
$\lceil x\rceil$ and $\lfloor x\rfloor$ denote the ceiling and floor functions, respectively, and $[N]\defeq\{1,\ldots,N\}$.

\subsection{Markov Decision Process}
\label{subsec:mdpFormulation}

We consider an infinite-horizon discounted Markov decision process (MDP): $(\mathcal{S},\mathcal{A}, r, \P, \gamma, \pmb{s}_1)$, where $\mathcal{S}$ and $\mathcal{A}$
denote the state and action spaces, $r: \mathcal{S}\times\mathcal{A} \rightarrow \mathbb{R}$ is the reward function, $\P$ is the state-transition kernel, $\gamma\in (0,1)$ is the discount factor, and $\pmb{s}_1$ is the initial state. 
At time $t$, the agent observes state $\pmb{s}_t$, takes action $\pmb{a}_t\sim \pi_{\pmb{\theta}}(\cdot|\pmb{s}_t)$ under a stochastic policy $\pi_{\pmb\theta}$ parameterized by $\pmb{\theta} \in \mathbb{R}^d$,
and receives 
reward $r(\pmb{s}_t,\pmb{a}_t)$.  
The state evolution is governed by the transition kernel $\P(\pmb{s}_{t+1}\in\cdot|\pmb{s}_t,\pmb{a}_t)$.
We denote by $\P(\pmb{s}_{t}\in\cdot|\pmb{s}_1,\pmb\theta)$ the $t$-step state distribution under a fixed policy $\pi_{\pmb\theta}$. By abuse of notation, $\P(\pmb{s}_{t}\in\cdot|\pmb{s}_1)$ denotes the $t$-step state distribution induced by the sequence of evolving policies $\pi_{\pmb\theta_1},\pi_{\pmb\theta_2},\ldots,\pi_{\pmb\theta_{t-1}}$.

The performance of a policy $\pi_{\pmb{\theta}}$ is measured by its discounted return. 
For a trajectory $\pmb{\tau}=(\pmb{s}_1,\pmb{a}_1,\pmb{s}_2,\ldots,\pmb{s}_{H+1})$ of length $H$, define the discounted return 
$R(\pmb{\tau})\defeq\sum_{t=1}^{H}\gamma^{t-1}r(\pmb{s}_t,\pmb{a}_t)$ and let $R_t^\gamma(H)\defeq\sum_{t^\prime=t}^{H}\gamma^{t^\prime-t}r(\pmb{s}_{t^\prime},\pmb{a}_{t^\prime})$ denote the discounted reward-to-go from time $t$. Our objective is to find an optimal policy $\pi_{\pmb{\theta}^*}$ that maximizes the normalized expected discounted return:
{
\begin{align} \label{eq: objective infinite horizon}
\pmb\theta^*\in\arg\max_{\pmb\theta\in\Theta}J(\pmb\theta),  ~~~
J(\pmb\theta)
=(1-\gamma)\E\left[\left.\sum_{t=1}^{\infty}\gamma^{t-1}r(\pmb{s}_t,\pmb{a}_t)\right|\pi_{\pmb\theta},\pmb{s}_1\right]
=\E_{\pmb{s}\sim d^{\pi_{\pmb\theta}},\,\pmb{a}\sim\pi_{\pmb\theta}(\cdot|\pmb{s})}[r(\pmb{s},\pmb{a})].
\end{align}}
Here, $\Theta$ denotes the policy parameter space and
$d^{\pi_{\pmb\theta}}(\pmb{s})=(1-\gamma)
\sum_{t=1}^\infty\gamma^{t-1}p(\pmb{s}_{t}=\pmb{s}|\pmb{s}_1;{\pi_{\pmb\theta}})
$ is the normalized discounted state-occupancy measure induced by $\pi_{\pmb{\theta}}$. The corresponding state-action occupancy measure is
$d^{\pi_{\pmb\theta}}(\pmb{s},\pmb{a})=\pi_{\pmb\theta}(\pmb{a}|\pmb{s})d^{\pi_{\pmb\theta}}(\pmb{s})$.

Similarly, for a policy $\pi_{\pmb{\theta}}$, we define the state-value function starting with state $\pmb{s}$, i.e.,
\begin{align}
 V^{\pi_{\pmb\theta}}(\pmb{s})&= \lim_{H\rightarrow\infty}\E\left[\left. R^\gamma_1(H)\right| \pmb{s}_1=\pmb{s};{\pi_{\pmb\theta}}\right]= \E\left[\left.\sum_{t=1}^{\infty} \gamma^{t-1}r(\pmb{s}_{t},\pmb{a}_{t})\right| \pmb{s}_1=\pmb{s};{\pi_{\pmb\theta}}\right].
 \nonumber 
\end{align}
The corresponding action-value function is defined as the expected discounted return starting from the state-action pair $(\pmb{s},\pmb{a})$:
\begin{align}
 Q^{\pi_{\pmb\theta}}(\pmb{s},\pmb{a}) &=\lim_{H\rightarrow\infty} \E\left[\left. R^\gamma_1(H)\right| \pmb{s}_1=\pmb{s},\pmb{a}_1=\pmb{a};{\pi_{\pmb\theta}}\right] = \E\left[\left.\sum_{t=1}^{\infty} \gamma^{t-1}r(\pmb{s}_{t},\pmb{a}_{t})\right| \pmb{s}_1=\pmb{s}, \pmb{a}_1=\pmb{a};\pi_{\pmb\theta}\right].
 \nonumber 
\end{align}

For the theoretical analysis of PG-VRER and its replay estimators in Sections~\ref{sec: theory of experience replay}--\ref{sec: convergence analysis}, we introduce the restart kernel
$\widetilde{\P}(\dd\pmb{s}'|\pmb{s},\pmb{a})=\gamma\P(\dd\pmb{s}'|\pmb{s},\pmb{a})+(1-\gamma)\delta_{\pmb{s}_1}(\dd\pmb{s}')$, 
which resets to the initial state $\pmb{s}_1$ with probability $1-\gamma$ at each transition.
While the original kernel $\P$ defines the MDP, objective, and value functions, the restart kernel $\widetilde{\P}$ is used solely for analysis. Its stationary distribution is the normalized discounted occupancy measure $d^{\pi_{\pmb{\theta}}}$, enabling discounted expectations to be analyzed through stationary-distribution and mixing arguments. Consequently, $\widetilde{\P}$ is an auxiliary construct used only in the proofs and does not change the underlying MDP or policy-gradient formulation.

\subsection{Generic Policy Gradient}
\label{subsec:generalPolicyGradient}

Stochastic policy gradient ascent is a widely used approach for solving the RL optimization problem~\eqref{eq: objective infinite horizon}. At each $k$-th iteration, the policy parameters are updated according to
\begin{equation} 
\label{eq: policy gradient update}
 \pmb\theta_{k+1} \leftarrow \pmb\theta_k + \eta_k \widehat{\nabla} J(\pmb\theta_k),
\end{equation}
where $\eta_k$ is learning rate or step size and $\widehat{\nabla} J(\pmb{\theta}_k)$ is an estimator of the policy gradient $\nabla  J(\pmb{\theta}_k)$.
Throughout the paper, $\nabla$ denotes the gradient with respect to $\pmb{\theta}$ unless otherwise specified.

Under standard regularity conditions that permit interchanging differentiation and expectation,
\citep{sutton1999policy}, 
the policy gradient of the normalized objective in Eq.~\eqref{eq: objective infinite horizon} can be expressed as
\begin{equation}
 \nabla J(\pmb{\theta})= \E_{(\pmb{s},\pmb{a})\sim d^{\pi_{\pmb\theta}}(\cdot,\cdot)}[g(\pmb{s},\pmb{a})|\pmb{\theta}]=\E_{(\pmb{s},\pmb{a})\sim d^{\pi_{\pmb\theta}}(\cdot,\cdot)}\left[A^{\pi_{\pmb\theta}}(\pmb{s},\pmb{a}) \nabla \log \pi_{\pmb{\theta}}(\pmb{a}|\pmb{s})\right]
 \label{eq: policy gradient}
\end{equation} 
 where $A^{\pi_{\pmb\theta}}(\pmb{s},\pmb{a}) \defeq Q^{\pi_{\pmb\theta}}(\pmb{s},\pmb{a})-V^{\pi_{\pmb\theta}}(\pmb{s})$ is called \textit{advantage} and it intuitively measures the extra reward that the agent can obtain by taking a particular action $\pmb{a}$ at state $\pmb{s}$. 
The corresponding scenario-based policy gradient estimator
 in (\ref{eq: policy gradient}) is
 \begin{equation}
g\left(\pmb{s},\pmb{a}|\pmb\theta_k\right) \defeq
 A^{\pi_{\pmb\theta_k}}(\pmb{s},\pmb{a}) \nabla \log \pi_{\pmb{\theta}_k}(\pmb{a}|\pmb{s}).
 \label{eq.scenariobasedGradient}
 \end{equation}

Modern policy optimization algorithms typically perform multiple environment interactions before updating the policy. Accordingly, the \textit{policy gradient (PG) estimator} at iteration $k$ is computed by the sample average of $n$ new generated samples, i.e., 
\begin{equation}
 \widehat{\nabla} J^{PG}_k \defeq
 \widehat{\nabla} J^{PG}(\pmb{\theta}_k) =\frac{1}{n} \sum_{j=1}^n g\left(\pmb{s}^{(k,j)},\pmb{a}^{(k,j)}|\pmb\theta_k\right),
 \label{eq.PG-estimator}
\end{equation}
where $\left({\pmb{s}}^{(k,j)},{\pmb{a}}^{(k,j)}\right)$ denotes the $j$-th sample collected under policy $\pi_{\pmb{\theta}_k}$. 
For notational simplicity, we consider the following general setting throughout the paper: 
at each $k$-th iteration, we collect $n$ new samples (i.e., fixed-length trajectory segments collected from one or multiple environments) 
by following the latest policy $\pi_{\pmb\theta_k}$ to estimate and optimize the objective \eqref{eq: objective infinite horizon}.

\subsection{Regularity Conditions for Policy Gradient Estimation}
\label{subsec:assumptions}

\begin{enumerate}
\renewcommand{\theenumi}{A.\arabic{enumi}}
\renewcommand{\labelenumi}{\textbf{\theenumi}}
\item Suppose the reward and policy functions satisfy the following regularity conditions.
\begin{enumerate}
    \item[(i)] The absolute value of the reward $r(\pmb{s}, \pmb{a})$ is bounded uniformly, i.e., there exists a constant, say $U_r>0$ such that  $|r(\pmb{s},\pmb{a})|\leq U_r$ for any $(\pmb{s},\pmb{a})\in \mathcal{S}\times\mathcal{A}$.
    \item[(ii)] The score function is assumed to be Lipschitz continuous with bounded norm, and the policy $\pi_{\pmb\theta}$ is differentiable and Lipschitz continuous in $\pmb\theta$. 
For any $(\pmb{s},\pmb{a})\in \mathcal{S}\times\mathcal{A}$, there exist positive constants $L_\Theta, U_{\Theta}$, $U_{\pi}< \infty$ such that
\begin{align}
    &\Vert \nabla\log\pi_{\pmb\theta_1}(\pmb{a}|\pmb{s})  -\nabla\log\pi_{\pmb\theta_2}(\pmb{a}|\pmb{s}) 
    \Vert
    \leq L_\Theta\Vert  \pmb{\theta}_1-\pmb{\theta}_2\Vert \text{ for any $\pmb\theta_1$ and $\pmb\theta_2$}; \label{eq: assumption item 1}\\
    &\Vert \nabla\log\pi_{\pmb\theta}(\pmb{a}|\pmb{s})\Vert \leq U_\Theta \text{ for any $\pmb\theta$};\label{eq: assumption item 2}
    \\
    &\Vert\pi_{\pmb\theta_1}(\cdot|\pmb{s})-\pi_{\pmb\theta_2}(\cdot|\pmb{s}) \Vert_{TV} \leq {U_\pi}\Vert\pmb\theta_1-\pmb\theta_2\Vert.
    \label{eq: assumption item 3}
\end{align}
\end{enumerate}
\label{assumption 2}

\item (Uniform Ergodicity)
For any fixed policy parameter $\pmb\theta$, let $d^{\pi_{\pmb\theta}}(\cdot)$ denote the stationary distribution of the Markov chain induced by $\pi_{\pmb{\theta}}$ and the restart kernel $\widetilde{\P}$. 
There exists a decreasing function $\varphi(t)>0$ such that: $\left\Vert\widetilde{\P}(\pmb{s}_t\in\cdot\mid\pmb{s}_1=\pmb{s})-d^{\pi_{\pmb\theta}}(\cdot)\right\Vert_{TV}\leq\varphi(t)$ for $t\geq1,\ \forall\pmb{s}\in\mathcal{S}$,
where $\varphi(t)=\kappa_0 \kappa^t$ for some constants $\kappa_0 > 0$ and $\kappa\in(0,1)$.
 \label{assumption 3}
\end{enumerate}
 
The bounded-reward in Assumption~\ref{assumption 2}(i) is natural for applications with bounded, clipped, or normalized rewards.
Assumption~\ref{assumption 2}(ii) ensures smoothness and bounded gradients of the score function, which are required for the bias, variance, selection-rule, and finite-time convergence analyses, while 
Assumption~\ref{assumption 3} imposes uniform ergodicity of the restart-kernel Markov chain, ensuring geometric mixing and controlling the dependence among sampled transitions. 
Such score-function regularity and geometric-mixing conditions are standard in finite-time analyses of policy-gradient and actor-critic methods~\citep{zhang2020global,wu2020finite,xu2020improving}.
A uniformly bounded likelihood ratio is not assumed in the main results; it is introduced only in Appendix~\ref{appendix sec: verification of variance estimator} as a sufficient higher-moment condition for establishing the consistency of the gradient-variance estimators used by the selection rule. Although some assumptions may be difficult to verify directly, our empirical results demonstrate the robustness of VRER in practice.

\begin{lemma}[\cite{zhang2020global}, Lemma 3.2]\label{lemma: Lipschitz continuity}
 Under Assumption \ref{assumption 2}, the policy gradient of objective, denoted by $\nabla  J(\pmb\theta)$, is Lipschitz continuous, i.e., for any policy parameters $\pmb\theta_1,\pmb\theta_2 \in\Theta$, 
 there exists a constant $L > 0$ s.t.
 $\Vert\nabla  J(\pmb{\theta}_1) -\nabla  J(\pmb{\theta}_2)\Vert\leq L\lVert \pmb{\theta}_1-\pmb{\theta}_2\rVert.$
\end{lemma}

 \begin{lemma}[Boundedness 
 of Stochastic Policy Gradients]
 \label{lemma: bounded of policy gradient} 
 For any $\pmb\theta$, the norm of the policy gradient $\nabla  J(\pmb\theta)$ and its scenario-based stochastic estimate $g(\pmb{s},\pmb{a}|\pmb\theta)$ is bounded, i.e., 
 $\Vert \nabla  J(\pmb\theta)\Vert \leq M 
 \text{ and } 
  \Vert g\left(\pmb{s},\pmb{a}|\pmb\theta\right)\Vert \leq M,
  $
 where $M=\frac{2U_r U_\Theta}{1-\gamma}$.
 \end{lemma}

The uniform boundedness of the reward function $r(\pmb{s},\pmb{a})$ in Assumption~\ref{assumption 2}(i) also implies that the absolute value of the Q-function is upper bounded
by $\frac{U_r}{1-\gamma}$, since by definition
\begin{equation}\label{eq: bound of action value function}
|Q^{\pi_{\pmb\theta}}(\pmb{s},\pmb{a})| \leq \sum^\infty_{t=1}\gamma^{t-1} U_r=\frac{U_r}{1-\gamma}, \text{ for any $(\pmb{s},\pmb{a})\in \mathcal{S}\times\mathcal{A}$}.
\end{equation}
The same bound also applies for $V^{\pi_{\pmb\theta}}(\pmb{s})$ for any $\pmb\theta$ and $\pmb{s}\in\mathcal{S}$, i.e., 
\begin{equation}\label{eq: bound of state value function}
   |V^{\pi_{\pmb\theta}}(\pmb{s})| =  |\E\left[Q^{\pi_{\pmb\theta}}(\pmb{s},\pmb{a})\right]| \leq \E\left[|Q^{\pi_{\pmb\theta}}(\pmb{s},\pmb{a})|\right] \leq \sum^\infty_{t=1}\gamma^{t-1} U_r=\frac{U_r}{1-\gamma}.
\end{equation}
The objective satisfies  $|J(\pmb\theta)|=(1-\gamma)|V^{\pi_{\pmb\theta}}(\pmb{s}_1)|\leq U_r\leq\frac{U_r}{1-\gamma}$, and thus it is bounded by $U_ J \defeq\frac{U_r}{1-\gamma}$.
A wide range of policy-gradient algorithms have been proposed, each employing a different gradient estimator; see \citet{sutton2018reinforcement} and \citet{schulman2015high} for reviews. A key advantage of VRER is its algorithm-agnostic design, enabling seamless integration with diverse policy-gradient methods. Section~\ref{sec:empiricalStudy} demonstrates this flexibility using three state-of-the-art algorithms.

\section{Likelihood Ratio Based Policy Gradient Estimation}
\label{sec:MLR}
In this section, we present an IS-based framework for improving the estimation of the policy gradient $\nabla J(\pmb\theta_k)$ through the reuse of selected historical samples.
Let $\mathcal{F}_k=\{\pmb\theta_i\}_{i=k-B_k+1}^{k}$ denote the finite first-in, first-out (FIFO) policy buffer at iteration $k$, where $B_k=\min(k,B)$ and $B$ is the maximum buffer capacity. Each policy $\pmb\theta_i\in\mathcal{F}_k$ is associated with a stored sample $\mathcal{T}_i$, and $\mathcal{D}_k=\bigcup_{\pmb\theta_i\in\mathcal{F}_k}\mathcal{T}_i$ denotes the corresponding historical dataset. We further define the reuse set $\mathcal{U}_k\subseteq\mathcal{F}_k$ as the collection of policies whose samples are selected for estimating $\nabla J(\pmb\theta_k)$, with $|\mathcal{U}_k|$ denoting its cardinality. Throughout this section, $\mathcal{U}_k$ is treated as given; its construction is presented in Section~\ref{sec:vrer}.

\subsection{Step-based Policy Optimization via Importance Sampling}
\label{subsec: step-based importance sampling}

\begin{sloppypar}
While IS is effective for episode-based finite-horizon MDPs \citep{zheng2020green,metelli2020importance}, extending it to step-based infinite-horizon settings is considerably more challenging. Unlike the i.i.d. nature of episodic reuse, step-based samples are sequentially correlated due to Markovian noise. As a result, the sampling distribution is generally not the target stationary distribution, and the standard IS identity does not hold. Consequently, the likelihood-ratio estimator is biased:
\begin{equation}\label{eq: IS identity not hold}
\E_{\pmb{s}\sim \P(\pmb{s}_i\in\cdot|\pmb{s}_1),\pmb{a}\sim\pi_{\pmb\theta_i}(\cdot|\pmb{s})}\left[\varrho_{i,k}(\pmb{s},\pmb{a})g(\pmb{s},\pmb{a}|\pmb\theta_k)\right]
=\E_{\pmb{s}\sim \P(\pmb{s}_k\in\cdot|\pmb{s}_1),\pmb{a}\sim\pi_{\pmb\theta_k}(\cdot|\pmb{s})}\left[g(\pmb{s},\pmb{a}|\pmb\theta_k)\right]
\neq \nabla J(\pmb\theta_k)
\end{equation}
where $\varrho_{i,k}(\pmb{s},\pmb{a})=\frac{\P(\pmb{s}_k=\dd\pmb{s}|\pmb{s}_1)\pi_{\pmb\theta_k}(\pmb{a}|\pmb{s})}{\P(\pmb{s}_i=\dd\pmb{s}|\pmb{s}_1)\pi_{\pmb\theta_i}(\pmb{a}|\pmb{s})}$ and $\nabla J(\pmb\theta_k)=\E_{(\pmb{s},\pmb{a})\sim d^{\pi_{\pmb\theta_k}}(\cdot,\cdot)}[g(\pmb{s},\pmb{a})|\pmb{\theta}_k]$ with  the scenario-based policy gradient estimate
{$g(\pmb{s},\pmb{a}|\pmb\theta_k)$}. We write the probability density function of $k$-step transition under a sequence of evolving behavior policies $(\pmb\theta_1,\ldots,\pmb\theta_{k-1})$ induced by the optimization search as $\P(\pmb{s}_k=\dd\pmb{s}|\pmb{s}_1)
\defeq \P(\pmb{s}_k=\dd\pmb{s}|\pmb{s}_1,\pmb\theta_1,\ldots,\pmb\theta_{k-1})$. Despite this bias, the simplified likelihood ratio $f_{i,k}(\pmb{s},\pmb{a})\defeq \frac{\pi_{\pmb\theta_k}(\pmb{a}|\pmb{s})}{\pi_{\pmb\theta_i}(\pmb{a}|\pmb{s})}$ is widely used in off-policy policy-gradient methods due to its computational tractability, despite ignoring the state-distribution mismatch~\citep{degris2012off}.
\end{sloppypar}

\subsection{Likelihood Ratio Policy Gradient Estimator}
The policy gradient $\nabla J(\pmb\theta_k)$ can be estimated using historical samples collected under a behavior policy $\pi_{\pmb{\theta}_i}$ with $i<k$ via the likelihood-ratio (LR) estimator
\begin{equation}
 \widehat{\nabla} J^{LR}_{i,k}=\frac{1}{n}\sum^{n}_{j=1}f_{i,k}\left(\pmb{s}^{(i,j)},\pmb{a}^{(i,j)}\right)
 g\left(\pmb{s}^{(i,j)},\pmb{a}^{(i,j)}|\pmb\theta_k\right)
 ~ \mbox{with} ~ f_{i,k}(\pmb{s},\pmb{a}) = \frac{\pi_{\pmb\theta_k}(\pmb{a}|\pmb{s})}{\pi_{\pmb\theta_i}(\pmb{a}|\pmb{s})}.
 \label{eq: individual likelihood ratio estimator}
\end{equation}
The likelihood ratio $f_{i,k}(\pmb{s},\pmb{a})$ weights the historical samples 
to account for the mismatch between the behavior and target policies specified by parameters $\pmb\theta_i$ and $\pmb\theta_k$.

One way to reuse all samples associated with the policies in the reuse set $\mathcal{U}_k$ is to average the corresponding individual LR estimators, yielding the \textit{likelihood ratio (LR) policy gradient estimator},
\begin{equation}
\label{eq.LR-gradient}
\widehat{\nabla} J^{LR}_{k}=\frac{1}{|\mathcal{U}_k|}\sum_{{\pmb\theta_i}\in \mathcal{U}_k}\widehat{\nabla} J^{LR}_{i,k} = 
 \frac{1}{|\mathcal{U}_k|n}
 \sum_{{\pmb\theta_i}\in \mathcal{U}_k}
 \sum^{n}_{j=1}
 \frac{\pi_{\pmb\theta_k}(\pmb{a}^{(i,j)}|\pmb{s}^{(i,j)})}
 {\pi_{\pmb\theta_i}(\pmb{a}^{(i,j)}|\pmb{s}^{(i,j)})}
 g\left(\pmb{s}^{(i,j)},\pmb{a}^{(i,j)}|\pmb\theta_k\right).
\end{equation}
For simplification, we assume that each policy evaluation generates the same number of samples $n$.

\subsection{Clipped Likelihood Ratio (CLR) Policy Gradient Estimator}
Although the LR estimators in \eqref{eq: individual likelihood ratio estimator} and \eqref{eq.LR-gradient} are widely used, they can exhibit high or even infinite variance when the likelihood ratio becomes large.
A technique for mitigating this problem is weight clipping~\citep{ionides2008truncated}, which truncates the LR 
by applying the operator 
$\min(f_{i,k}(\cdot,\cdot), U_f)$, i.e., 
\begin{align}
\widehat{\nabla} J^{CLR}_{k}&=\frac{1}{|\mathcal{U}_k|}\sum_{{\pmb\theta_i}\in \mathcal{U}_k} \widehat{\nabla} J^{CLR}_{i,k},
\label{eq.CLR-gradient} \\
\mbox{  with  } \widehat{\nabla} J^{CLR}_{i,k} &= 
 \frac{1}{n}
 \sum^{n}_{j=1}
 \min\left(\frac{\pi_{\pmb\theta_k}(\pmb{a}^{(i,j)}|\pmb{s}^{(i,j)})}
 {\pi_{\pmb\theta_i}(\pmb{a}^{(i,j)}|\pmb{s}^{(i,j)})},U_f\right)
g\left(\pmb{s}^{(i,j)},\pmb{a}^{(i,j)}|\pmb\theta_k\right).\label{eq: individual clipped likelihood ratio estimator}
\end{align}
The clipping threshold $U_f\geq 1$ limits extreme importance weights and thereby reduces variance. 
We refer to $\widehat{\nabla} J^{CLR}_{i,k}$ and $\widehat{\nabla} J^{CLR}_{k}$ in \eqref{eq: individual clipped likelihood ratio estimator} and \eqref{eq.CLR-gradient} as the individual and average clipped likelihood ratio (CLR) policy-gradient estimators.
In the remainder of the paper, both the LR estimator \eqref{eq.LR-gradient} and the CLR estimator \eqref{eq.CLR-gradient} are treated as biased estimators and analyzed theoretically and empirically.

\section{Variance-Bias Analysis in Experience Replay}
\label{sec: theory of experience replay}

Although prior studies have empirically investigated the effects of replay-related hyperparameters, such as policy age, buffer size, and replay capacity \citep{fedus2020revisiting}, a rigorous theoretical characterization of the variance--bias trade-off induced by sample reuse remains largely absent. To address this gap, we develop a unified analytical framework that quantifies how sample dependence arising from Markovian dynamics, experience replay, and policy updates affects the bias (Section~\ref{subsec: experience replay bias}) and variance (Section~\ref{subsec: correlation increase gradient variance}) of LR and CLR policy-gradient estimators. The analysis is independent of the specific variance-reduction selection rule introduced in Section~\ref{sec:vrer} and applies broadly to replay-based off-policy policy optimization methods.

\subsection{Policy Gradient Bias from Dependent Sample Replay}

\label{subsec: experience replay bias}


Despite their widespread use, PG methods face inherent challenges in obtaining unbiased gradient estimates for infinite-horizon MDPs. The Policy Gradient Theorem~\citep{sutton1999policy} requires state--action samples to be drawn from the stationary distribution induced by the target policy; see (\ref{eq: objective infinite horizon}). In practice, however, PG algorithms generate samples sequentially, and the resulting Markovian dependence introduces bias into the gradient estimator~\citep{zhang2020global}. ER further exacerbates this issue by reusing samples collected under earlier behavior policies, thereby increasing temporal and policy dependence. To quantify these effects, Theorem~\ref{theorem: bounded bias for LR graident} establishes an upper bound on the bias of the LR and CLR policy-gradient estimators; the proof is provided in Appendix~\ref{appendix sec: bias of policy gradient}.

Specifically, at iteration $k$, suppose the historical samples
$\pmb{x}^{(i,j)} \defeq (\pmb{s}^{(i,j)},\pmb{a}^{(i,j)})$ with
$j=1,2,\ldots,n$, collected at iteration $i$ under policy
$\pi_{\pmb{\theta}_i}$ are reused for gradient estimation.
Introducing a lag parameter $t$ to account for the dependence induced by Markovian dynamics, we derive a bias bound for the LR and CLR estimators in \eqref{eq.LR-gradient} and \eqref{eq.CLR-gradient} that holds for any integer $t<i$.
The bound becomes tighter as $k$ increases and the policy update step size $\eta_k$ decreases when $t$ is chosen as the mixing time $t_k^\star$ defined below.

\begin{definition}[Mixing Time, \cite{wu2020finite}]
For an ergodic MDP under policy $\pi_{\pmb\theta_k}$, the mixing time is defined as 
$t_k^\star \defeq \min\{t \geq 0|\varphi(nt) \leq  \eta_{k}
\}$, where $\eta_k$ is the step size at iteration $k$.
\end{definition}

\begin{theorem}[Gradient Biasedness]
\label{theorem: bounded bias for LR graident}
Suppose Assumptions~\ref{assumption 2} and \ref{assumption 3} hold. Consider the policy update rule \eqref{eq: policy gradient update} with step size $\eta_k$. For either the LR or CLR estimator (\(R\in\{\mathrm{LR},\mathrm{CLR}\}\)) in Eq.~\eqref{eq.LR-gradient} and \eqref{eq.CLR-gradient} and any integer $0\le t<i$, the bias of policy gradient estimators satisfies
\begin{equation*}
     \left\Vert\E\left[\widehat{\nabla} J^{R}_k\right] - \E[\nabla  J(\pmb\theta_k)]\right\Vert \leq \frac{Z_1^{R} \cdot t}{|\mathcal{U}_k|}
 \sum_{{\pmb\theta_i}\in \mathcal{U}_k}\sum^i_{\ell = i-t}\eta_\ell + 2M\varphi(nt) + \frac{Z_2^{R}}{|\mathcal{U}_k|}
 \sum_{{\pmb\theta_i}\in \mathcal{U}_k} \sum^k_{\ell=i-t}\eta_\ell, 
\end{equation*}
where $Z_1^{LR}=2nM^2 U_\pi$, $Z_1^{CLR}= 2nM^2 U_\pi U_f$, $Z_2^{LR}=M(2MC_d+2L_g+MU_\pi)$ and $Z_2^{CLR}=M(2MC_d+ (U_f+1)L_g+MU_\pi)$. The Lipschitz constants $C_d$ and $L_g$ are defined in Lemmas~\ref{lemma: lipchitz continuity of stationary distribution} and~\ref{lemma: lipchitz continuity of sample gradient estimate} respectively in Appendix~\ref{appendix subsec: lipschitz continuity}.
\end{theorem}

The difference between the LR- and CLR-specific constants in Theorem~\ref{theorem: bounded bias for LR graident} makes the role of clipping explicit in the total-bias analysis. In particular, the dependence of $Z_1^{CLR}$ and $Z_2^{CLR}$ on $U_f$ arises from controlling clipped likelihood-ratio weights in the proof; this dependence provides a qualitative indication of the effect of clipping bias.
These estimator-specific constants characterize the overall deviation between the expected replay gradient and the current-policy gradient and are subsequently used in the convergence analysis of Theorem~\ref{convergence theorem}.

By setting $i=k$ in Theorem~\ref{theorem: bounded bias for LR graident}, corresponding to the case without sample reuse, we obtain an upper bound on the bias of the classical PG estimator \eqref{eq.PG-estimator} as shown in Corollary~\ref{corollary: bounded bias for vanila graident}. 

\begin{corollary}
\label{corollary: bounded bias for vanila graident}
Suppose Assumptions~\ref{assumption 2} and \ref{assumption 3} hold. 
For any integer $0\le t<k$, the bias of policy gradient estimator (\ref{eq.PG-estimator}) can be bounded by 
\begin{equation*}
     \left\Vert\E\left[\widehat{\nabla} J^{PG}_k\right] - \E[\nabla  J(\pmb\theta_k)]\right\Vert \leq {Z_1^{LR} (t+1)}\sum^k_{\ell = k-t}\eta_\ell + 2M\varphi(nt) +{Z_2^{LR}}\sum^k_{\ell=k-t}\eta_\ell, 
\end{equation*}
where $Z_1^{LR}$ and $Z_2^{LR}$ are specified in Theorem~\ref{theorem: bounded bias for LR graident}.
\end{corollary}

Now consider a diminishing step size $\eta_k=\eta_1 k^{-r}$ with $r\in (0,1)$. Since {$\sum^i_{i-t}\eta_\ell \leq t\eta_{i-t}$ and $\sum^{k}_{\ell=i-t}\eta_\ell\leq (k-i+t)\eta_{i-t}$} due to $k\geq i> t$, the bias bound in Theorem~\ref{theorem: bounded bias for LR graident} can be simplified to
\begin{equation}\label{eq: big O of bias}
\left\Vert\E\left[\widehat{\nabla} J^{R}_k\right] - \E[\nabla  J(\pmb\theta_k)]\right\Vert
\leq \mathcal{O}\left(\frac{1}{|\mathcal{U}_k|}
 \sum_{{\pmb\theta_i}\in \mathcal{U}_k}\frac{(t+1)t}{(i-t)^r}\right)+\mathcal{O}\left(\frac{1}{|\mathcal{U}_k|}
 \sum_{{\pmb\theta_i}\in \mathcal{U}_k}\frac{k-i+t}{(i-t)^r}\right)+\mathcal{O}\left(\varphi(nt)\right).
\end{equation}
Equation~\eqref{eq: big O of bias} shows that the bias of LR/CLR policy gradient estimator increases as older policies are reused (smaller $i$), reflecting the growing mismatch between the behavior and target policies. Since $k-i\le B_k$, the bound can be further expressed in terms of the replay-buffer size:
\begin{align}
\left\Vert\E\left[\widehat{\nabla} J^{R}_k\right] - \E[\nabla  J(\pmb\theta_k)]\right\Vert  \leq \mathcal{O}\left(\frac{(t+1)t}{(k-B_k-t)^r}\right)+\mathcal{O}\left(\frac{B_k+t}{(k-B_k-t)^r}\right)+\mathcal{O}\left(\varphi(nt)\right)\nonumber\\
= \mathcal{O}\left(\frac{\frac{r^2}{n^2}\log_{\kappa}^2 k}{(k-B_k+\frac{r}{n}\log_{\kappa} k)^r}\right)+\mathcal{O}\left(\frac{B_k-\frac{r}{n}\log_{\kappa} k}{(k-B_k+\frac{r}{n}\log_{\kappa} k)^r}\right)+\mathcal{O}\left(\frac{1}{k^r}\right). \label{eq: big O of bias with buffer size}
\end{align}
Choosing the lag term $t$ to be the mixing time, $t_k^*=\log_\kappa k^{-r/n}$, yields
$\varphi(nt)=\mathcal{O}(k^{-r})$, illustrating how both policy age and replay-buffer size contribute to the bias induced by experience replay.

Therefore, the preceding analysis leads to the following conclusions:
\begin{itemize}
\item \textbf{Asymptotically Unbiased}. 
The LR and CLR estimators in \eqref{eq.LR-gradient} and \eqref{eq.CLR-gradient} are asymptotically unbiased, i.e., $\Vert\E[\widehat{\nabla} J^{R}_k] - \E[\nabla  J(\pmb\theta_k)]\Vert \rightarrow 0$ as $k\rightarrow \infty$, provided that $B_k = \mathcal{O}(k^{r})$. A constant replay-buffer size is a special case satisfying this condition.

\item \textbf{Bias from Replaying Older Samples}. The bias increases with policy age and replay-buffer size. In particular, if $B_k$ grows faster than $\mathcal{O}(k^r)$, the LR/CLR policy gradient estimators might have an unbounded bias. The smallest bias is attained in the no-replay case, $\mathcal{U}_k=\{\pmb\theta_k\}$.
In such case, the bias is bounded by $\mathcal{O}\left({(t+1)t}{(k-t)^{-r}}\right)$ by applying (\ref{eq: big O of bias}).

\item \textbf{Role of Mixing}. The bias convergence rate depends on the mixing rate of the underlying Markov process. Faster mixing yields faster decay of the bias.
\end{itemize}

\subsection{Sample Correlation Increase Gradient Variance} \label{subsec: correlation increase gradient variance}

Given a reuse set $\mathcal{U}_k$, Proposition~\ref{prop: gradient variance decomposition} characterizes the variance of the LR and CLR policy-gradient estimators in \eqref{eq.LR-gradient} and \eqref{eq.CLR-gradient}; the proof is provided in Appendix~\ref{appendix sec: prop 1}.

 \begin{proposition}[Gradient Variance] 
 \label{prop: gradient variance decomposition}
Suppose Assumptions~\ref{assumption 2} and \ref{assumption 3} hold. At iteration $k$, for either the LR or CLR estimator ($R\in\{\mathrm{LR},\mathrm{CLR}\}$), the total variance can be decomposed as
\begin{equation}\label{eq: variance of LR estimaor}
     \Tr\left(\Var
 \left[\widehat{\nabla} J^{R}_{k} \right]\right)=\frac{1}{|\mathcal{U}_k|^2}\sum_{\pmb\theta_i\in\mathcal{U}_k}\sum_{\pmb\theta_{i^\prime}\in\mathcal{U}_k} \left(\pmb\sigma^{R}_{i,k}\right)^\top \mathbf{P}_{i,i^\prime,k} \left(\pmb\sigma^{R}_{i^\prime,k}\right),
\end{equation}
where $\pmb\sigma^{R}_{i,k}=\left(\sqrt{\Var\left[\widehat{\nabla} J^{R,(1)}_{i,k} \right]}, \ldots,\sqrt{\Var\left[\widehat{\nabla} J^{R,(d)}_{i,k} \right]}\right)^\top$ and  $\mathbf{P}_{i,i^\prime,k}=\Diag\left(\Corr^{(1)}_{i,i^\prime,k},\ldots,\Corr^{(d)}_{i,i^\prime,k}\right)$ with correlation coefficients of gradient estimate pairs $\Corr_{i,i^\prime,k}^{(\ell)}=\Corr\left(\widehat{\nabla} J^{R,(\ell)}_{i,k},\widehat{\nabla} J^{R,(\ell)}_{i^\prime,k}\right)$ for each $\ell$-th dimension of policy parameters with $\ell=1,2,\ldots,d$. By using the greatest element of $\mathbf{P}$, the total variance~\eqref{eq: variance of LR estimaor} is bounded by
\begin{equation}\label{eq: decomposed total variance}
     \Tr\left(\Var
 \left[\widehat{\nabla} J^{R}_{k} \right]\right)\leq \frac{1}{|\mathcal{U}_k|^2}\sum_{\pmb\theta_i\in\mathcal{U}_k}\sum_{\pmb\theta_{i^\prime}\in\mathcal{U}_k} \max_{\ell=1,\ldots,d}|\Corr^{(\ell)}_{i,i^\prime,k}|\left(\pmb\sigma^{R}_{i,k}\right)^\top\left(\pmb\sigma^{R}_{i^\prime,k}\right).
\end{equation}
 \end{proposition}

Theorem~\ref{theorem: bounded bias for LR graident} and Proposition~\ref{prop: gradient variance decomposition} highlight the importance of properly selecting the replay set and buffer size 
to balance variance control and replay-induced bias in policy gradient estimation. 
To address the variance component, Section~\ref{sec:vrer} proposes an experience selection strategy. Furthermore, the effects of buffer size and experience selection on policy optimization convergence are investigated both theoretically in Section~\ref{sec: convergence analysis} and empirically in Section~\ref{sec:empiricalStudy}. 
Since replay-induced bias is difficult to estimate in practice, designing replay policies that directly minimize the mean squared error (MSE) of policy-gradient estimators remains an important direction for future research.

\section{Variance-Reduced Replay for Policy Optimization}

\label{sec:vrer}
 
Effective experience replay requires selectively reusing historical samples to balance computational efficiency, stability, and variance reduction. To this end, Section~\ref{subsec: selection rule} introduces a variance-based selection criterion for LR/CLR estimators, supported by the variance estimators developed in Section~\ref{subsec:covarianceEst}.
To improve scalability as replay history grows, Section~\ref{subsec: gradient variance ratio approximation} derives a computationally efficient practical surrogate of this criterion. These components are integrated into the Variance-Reduced Experience Replay (VRER) framework and the generic PG-VRER algorithm in Section~\ref{subsec: GS-PG algorithm}, which applies to a broad class of policy optimization methods.

\subsection{
Selection Rule for Reducing Policy Gradient Variance} \label{subsec: selection rule}

Our goal is to construct a reuse set $\mathcal{U}_k$ that controls gradient-estimation variance by prioritizing historical samples with stable likelihood ratios. Theorem~\ref{thm:Online selection rule} establishes a selection criterion that bounds the variance of each individual LR/CLR estimator, $\widehat{\nabla} J^R_{i,k}$ ($R\in\{\mathrm{LR},\mathrm{CLR}\}$), by $c>1$ times that of the classical PG estimator $\widehat{\nabla} J^{PG}_k$. The proof is provided in Appendix~\ref{appendix sec: proofs for selection rule theorem}.

 \begin{theorem} 
 \label{thm:Online selection rule}
 At the $k$-th iteration with target policy $\pi_{\pmb\theta_k}$, the reuse set $\mathcal{U}_k$ is created to include the behavioral distributions, i.e., $\pi_{\pmb\theta_i}$ with $\pmb{\theta}_i \in \mathcal{F}_k$, whose total variance of individual LR/CLR policy gradient estimators in \eqref{eq: individual likelihood ratio estimator} and 
\eqref{eq: individual clipped likelihood ratio estimator} is no greater than $c$ times the total variance of the classical PG estimator for some constant $c>1$. Mathematically, for $R\in$\{LR,CLR\},
\begin{equation} \label{eq: selection rule online}
\text{\textbf{Selection Rule 1:}} ~~~ 
 \Tr\left(\Var
 \left[\widehat{\nabla} J^{R}_{i,k} \right]\right)\leq c\Tr\left(\Var
 \left[ \widehat{\nabla} J^{PG}_k \right]\right).
\end{equation}
Then, for the averaged LR/CLR estimators in \eqref{eq.LR-gradient} and \eqref{eq.CLR-gradient},
\begin{equation} 
 \Tr\left(\Var\left[\widehat{\nabla} J^{R}_{k} \right]\right)
 \leq \frac{c}{|\mathcal{U}_k|^2}\Tr\left(\Var\left[
 \widehat{\nabla} J^{PG}_{k} \right]\right) \sum_{\pmb\theta_i\in\mathcal{U}_k}\sum_{\pmb\theta_{i^\prime}\in\mathcal{U}_k}\max_{\ell=1,\ldots,d}|\Corr^{(\ell)}_{i,i^\prime,k}|.
 \label{eq: variance reduction (online)}
\end{equation} 
where $\Corr_{i,i^\prime,k}^{(\ell)}=\Corr\left(\widehat{\nabla} J^{R,(\ell)}_{i,k},\widehat{\nabla} J^{R,(\ell)}_{i^\prime,k}\right)$. That is, the total variance is bounded by the PG variance scaled by the average maximum absolute correlation among the individual LR/CLR gradient estimates. Moreover,
\begin{align*}
&\E\left[\left\Vert\widehat{\nabla} J^{R}_{k}\right\Vert^2\right]\leq \frac{c}{|\mathcal{U}_k|^2}\sum_{\pmb\theta_i\in\mathcal{U}_k}\sum_{\pmb\theta_{i^\prime}\in\mathcal{U}_k}\max_{\ell=1,\ldots,d}|\Corr^{(\ell)}_{i,i^\prime,k}|\E\left[\left\Vert\widehat{\nabla} J^{PG}_{k}\right\Vert^2\right] \nonumber\\
&\qquad\qquad\qquad\qquad + 2\left\Vert\E\left[\nabla  J(\pmb\theta_k)\right]\right\Vert^2+
2\left\Vert\E\left[\widehat{\nabla} J^{R}_{k}\right]-\E[\nabla J(\pmb\theta_k)]\right\Vert^2.
\end{align*}
 \end{theorem}

Eq.~\eqref{eq: variance reduction (online)} 
accounts for dependence among replay-gradient estimators.
Define the average correlation factor
$
\begin{aligned}
\bar\rho_k^{R}
\defeq \frac{1}{|\mathcal{U}_k|^2}
\sum_{\pmb\theta_i\in\mathcal{U}_k}
\sum_{\pmb\theta_{i^\prime}\in\mathcal{U}_k}
\left|
\max_{\ell=1,\ldots,d}
\Corr\!\left(
\widehat{\nabla}J^{R,(\ell)}_{i,k},
\widehat{\nabla}J^{R,(\ell)}_{i^\prime,k}
\right)
\right|.
\end{aligned}
$
Then Theorem~\ref{thm:Online selection rule} implies
$
\frac{
\Tr\!\left(\Var[\widehat{\nabla}J^R_k]\right)
}{
\Tr\!\left(\Var[\widehat{\nabla}J^{PG}_k]\right)
}
\leq c\bar\rho_k^{R}.
$
Hence, the conditional variance is reduced whenever
$
c\bar\rho_k^{R}<1.
$
The selection rule acts only on conditional sampling variance under a fixed policy and replay configuration and does not directly control variability across configurations.
Together with Theorem~\ref{theorem: bounded bias for LR graident}, this reveals the fundamental bias-variance tradeoff: weakly correlated replay gradients improve variance reduction, whereas older behavior policies may induce larger replay bias.


 \subsection{Variance Estimation for Policy Gradient Estimators}
 \label{subsec:covarianceEst}

To implement Selection Rule~\eqref{eq: selection rule online}, we must estimate the variances of the PG estimator and the individual LR/CLR estimators in \eqref{eq: individual likelihood ratio estimator} and \eqref{eq: individual clipped likelihood ratio estimator}. A key challenge is the temporal dependence among scenario-based gradient observations.
To account for this dependence, we employ the Moving Block Bootstrap (MBB) method \citep{kunsch1989jackknife,liu1992moving}. Unlike conventional bootstrap procedures that resample individual observations, MBB resamples blocks of consecutive observations, thereby preserving the dependence structure of the original data.

Let $g_{i,j} \defeq \frac{
\pi_{\pmb\theta_k}\left(\pmb{a}^{(i,j)}|\pmb{s}^{(i,j)}\right)
}{\pi_{\pmb\theta_i}\left(\pmb{a}^{(i,j)}|\pmb{s}^{(i,j)}\right)}g\left(\pmb{s}^{(i,j)},\pmb{a}^{(i,j)}|\pmb\theta_k\right)$ denote the scenario-based LR policy gradient
observation.  For a block length $l\in[1,n]$, define $X_j=(g_{i,j},\ldots,g_{i,j+l-1})$ with $j=1,2,\ldots,N$, where $N = n-l+1$. 
MBB generates bootstrap samples by drawing blocks with replacement from $\{X_1,X_2,\ldots,X_N\}$. 
Let $\bar{g}_{i,j}=(g_{i,j}+g_{i,j+1}+\ldots+ g_{i,j+l-1})/l$ denote the average gradient in block $X_j$, and let $\bar{g}_i =N^{-1}\sum^N_{j=1}\bar{g}_{i,j}$. 
Then, the MBB variance estimator for the individual LR estimator (\ref{eq: individual likelihood ratio estimator}) \cite[Chapter 3.2.1]{lahiri2003resampling}
is
\begin{equation} 
\widehat{\Var}\left[
\widehat{\nabla} J^{LR}_{i,k}\right]=\frac{l}{n}\left[\frac{1}{N}\sum^N_{j=1}\bar{g}_{i,j}\bar{g}_{i,j}^\top - \bar{g}_i\bar{g}_i^\top\right],
~~~
\Tr\left(\widehat{\Var}\left[
\widehat{\nabla} J^{LR}_{i,k}\right]\right)=\frac{l}{n}\left[\frac{1}{N}\sum^N_{j=1}\Vert \bar{g}_{i,j}\Vert^2 - \Vert\bar{g}_{i}\Vert^2\right].
\label{eq.LR-VarEst}
\end{equation}

The same construction applies to the CLR estimator \eqref{eq: individual clipped likelihood ratio estimator} by replacing the likelihood ratio in $g_{i,j}$ with its clipped counterpart.
Similarly, the variance estimator for the PG estimator is 
\begin{equation}
\widehat{\Var}\left[
\widehat{\nabla} J^{PG}_{k}\right]=\frac{l}{n}\left[\frac{1}{N}\sum^N_{j=1}\bar{g}_{k,j}\bar{g}_{k,j}^\top - \bar{g}_k\bar{g}_k^\top\right],
~~~
\Tr\left(\widehat{\Var}\left[
\widehat{\nabla} J^{PG}_{k}\right]\right)=\frac{l}{n}\left[\frac{1}{N}\sum^N_{j=1}\Vert \bar{g}_{k,j}\Vert^2 - \Vert\bar{g}_{k}\Vert^2\right].
\label{eq.PG-VarEst}
\end{equation}

The consistency of the MBB variance estimators is well established. \citet[Theorem 3.1]{lahiri2003resampling} showed that, under mild moment and strong-mixing conditions on the process $\{g_{i,j}\}_{j=1}^n$, the estimators in \eqref{eq.LR-VarEst} and \eqref{eq.PG-VarEst} are consistent for a broad range of block lengths. Consistency holds when $l\to\infty$ and $l/n\to0$ as $n\to\infty$. This result justifies the use of MBB for variance estimation in policy gradient settings; verification of the required regularity conditions is provided in Appendix~\ref{appendix sec: verification of variance estimator}.


\subsection{Gradient Variance Ratio Approximation} \label{subsec: gradient variance ratio approximation}

\begin{sloppypar}
Selection Rule~\eqref{eq: selection rule online} becomes costly when the policy buffer $\mathcal{F}_k$ is large because it requires estimating the bootstrap variance of every LR estimator in \eqref{eq: individual likelihood ratio estimator}.
To enable efficient online implementation, we derive a local second-order approximation of the variance ratio
$\Tr(\Var
 [\widehat{\nabla} J^{LR}_{i,k}])
 /\Tr(\Var
[\widehat{\nabla} J^{PG}_k])$ for step-based algorithms. 
Proposition~\ref{prop: approximation of variance ratio} presents the result, whose derivation is deferred to Appendix~\ref{appendix sec: justification for vr approximation}. Replacing the likelihood ratio with its clipped counterpart yields the same efficient policy-distance filter for CLR.
\end{sloppypar}

\begin{proposition}\label{prop: approximation of variance ratio}
Under Assumptions \ref{assumption 2} and \ref{assumption 3}, let $\eta_k=\eta_1 k^{-r}$ with $\eta_1\in(0,\frac{1}{4L}]$ and $r\in(0,1)$.
For any $t<i\le k$, if $t=o(k^{r/2})$ and the buffer size satisfies $B_k=o(k^{r/2})$, then
\begin{equation}\label{eq: approximator of variance ratio}
     \frac{\Tr\left(\Var
 \left[ \widehat{\nabla} J^{LR}_{i,k} \right]\right)}{\Tr\left(\Var\left[\widehat{\nabla} J^{PG}_k\right]\right)}\approx e^{\E\left[\KL\left(\pi_{\pmb\theta_k}(\cdot|\pmb{s})\Vert \pi_{\pmb\theta_i}(\cdot|\pmb{s})\right)\right]}\left(1+\zeta_k^{-1}\right)-\zeta^{-1}_k,
\end{equation}
where 
$\zeta_k={\Tr(\Var[\widehat{\nabla} J^{PG}_k])}/{\Vert\E[\widehat{\nabla} J^{PG}_k]\Vert^2}$ is the relative variance.
\end{proposition}

Directly estimating $\zeta_k$ is computationally expensive. Instead, we leverage the first- and second-moment estimates maintained by Adam \citep{kingma2015adam},
$
m_k=\frac{\tilde m_k}{1-\beta_1^{k+1}}$ and $
v_k=\frac{\tilde v_k}{1-\beta_2^{k+1}},
$
where
$
\tilde m_k=\beta_1\tilde m_{k-1}+(1-\beta_1)g_k$ and
$
\tilde v_k=\beta_2\tilde v_{k-1}+(1-\beta_2)g_k^2
$
with $\beta_1, \beta_2 \in (0,1)$.
Here, $g_k$, $m_k$, and $v_k$ are vectors with $d$ components. For a coordinate index $\ell\in\{1,\ldots,d\}$, $m_{k,\ell}$ and $v_{k,\ell}$ denote the corresponding components of $m_k$ and $v_k$; the square in Adam's update is taken componentwise.
Following \citet{balles2018dissecting}, the quantity $v_{k,\ell}-m_{k,\ell}^2$ provides an estimate of the variance of gradient coordinate $\ell$. Summing across the $d$ coordinates yields a scalar approximation of the trace-based relative variance measure $\zeta_k$:
\[
\hat\zeta_k=\frac{\sum_{\ell=1}^{d}(v_{k,\ell}-m_{k,\ell}^2)}{\sum_{\ell=1}^{d}m_{k,\ell}^2}.
\]
Because these quantities are already computed by Adam, estimating $\hat\zeta_k$ incurs essentially no additional cost, leaving policy-KL evaluation as the only extra computation.

Substituting $\hat\zeta_k$ for $\zeta_k$ in \eqref{eq: approximator of variance ratio} and rearranging yields:
\begin{equation}\label{eq: simple selection}
\mbox{\textbf{Selection Rule 2:}} ~~~
\E\left[\mbox{KL}\left(\pi_{\pmb\theta_k}(\cdot|\pmb{s})\Vert \pi_{\pmb\theta_i}(\cdot|\pmb{s})\right)\right] \leq \log\left(1 + (c-1)\frac{\hat\zeta_k}{\hat\zeta_k+1}\right).
\end{equation}
Selection Rule~2 imposes a dynamic threshold on the distance between behavior and target policies. As the estimated relative variance $\hat\zeta_k$ increases, the admissible KL divergence expands, permitting more aggressive replay when variance reduction is most valuable. When $c=1$, only samples generated by the current policy ($\pmb\theta_i=\pmb\theta_k$) satisfy the rule. 

Selection Rule~2 should be viewed as a computationally efficient surrogate for Selection Rule~1. 
The approximation in Proposition~\ref{prop: approximation of variance ratio} is local and asymptotic rather than a uniform finite-sample guarantee. For every retained behavior policy in the buffer, i.e., $i\geq k-B_k+1$, since $\eta_\ell=\eta_1\ell^{-r}$ is decreasing, Lemma~\ref{auxillary lemma: parametric distance} in Appendix ~F gives
\begin{align}
\E\!\left[\Vert\pmb\theta_k-\pmb\theta_i\Vert\right]
&\leq M\sum_{\ell=i}^{k}\eta_\ell
\leq M\eta_1(k-i+1)i^{-r} 
\nonumber\\
&\leq M\eta_1B_k(k-B_k+1)^{-r}
=\mathcal{O}\!\left(B_k k^{-r}\right)
=o\!\left(k^{-r/2}\right),
\label{eq: retained-policy-distance}
\end{align}
since $B_k=o(k^{r/2})$. Thus, retained behavior policies become locally close to the target policy as iteration increases. 
Nevertheless, the approximation may be less accurate when policy differences are large, importance weights are heavy-tailed, second-moment terms are highly variable, or estimates of $\hat\zeta_k$ and the empirical KL divergence are noisy, particularly near the selection threshold.

Because Selection Rules 1 and 2 are used only to construct the reuse set $\mathcal U_k$, approximation errors affect sample selection rather than the LR/CLR gradient estimators. Section~\ref{subsec: kl approximation diagnostic} shows that, after the initial transient, the KL-based surrogate produces rankings and replay sets that closely match those obtained from the variance-based criterion.


\subsection{VRER Assisted Generic Policy Optimization Algorithm} 
\label{subsec: GS-PG algorithm}

Let $\mathcal{F}_k=\{\pmb\theta_i\}_{i=k-B_k+1}^{k}$ denote the first-in, first-out (FIFO) policy buffer at iteration $k$, where $B_k=\min(k,B)$ and $B$ is the maximum buffer size. Each policy $\pmb\theta_i\in\mathcal{F}_k$ is paired with a trajectory segment $\mathcal{T}_i$ containing $n$ observations collected under $\pi_{\pmb\theta_i}$. The replay dataset is $\mathcal{D}_k=\bigcup_{\pmb\theta_i\in\mathcal{F}_k}\mathcal{T}_i$. A selection rule identifies a reuse set $\mathcal{U}_k\subseteq\mathcal{F}_k$ and the replay batch $\mathcal{D}_k^{n_0}$ is formed by drawing $n_0\le n$ observations with replacement from each selected trajectory $\mathcal{T}_i$.

Algorithm~\ref{algo: online} summarizes the generic \textbf{PG-VRER} procedure. At iteration $k$, we collect $n$ samples under $\pi_{\pmb\theta_k}$, update the buffer $\mathcal{F}_k$ and replay dataset $\mathcal{D}_k$, and construct the reuse set $\mathcal{U}_k$ using Selection Rule~1 or 2. To reduce dependence among replayed samples, we form $\mathcal{D}^{n_0}_k=\bigcup_{\pmb\theta_i\in\mathcal{U}_k}\{\pmb{x}^{(i,j)}\sim\mathcal{T}_i\}_{j=1}^{n_0}$, and then optimize $J(\pmb\theta)$ over $K_{\text{off}}$ epochs using mini-batch gradient ascent on $\mathcal{D}^{n_0}_k$. Sampling with replacement does not remove dependence; $n_0$ simply controls the replay contribution of each selected trajectory and the associated computational cost.
PG-VRER is compatible with a wide range of policy optimization algorithms. Data collection can be parallelized, and $K_{\text{off}}$ can be adapted to specific actor-critic updates. In practice, the approximation in \eqref{eq: simple selection} is preferable to the bootstrap-based rule \eqref{eq: selection rule online}, since it leverages KL-divergence and gradient-variance estimates already available from Adam \citep{kingma2015adam} at essentially no additional cost.

Replay is implemented without modifying the underlying learning algorithm. Each stored trajectory $\mathcal{T}_i$ retains $(\pmb{s}_t, \pmb{a}_t, r_t, \pmb{s}_{t+1}, \log \pi_{\pmb{\theta}_i}(\pmb{a}_t|\pmb{s}_t))$, which allows historical transitions to be evaluated under the current policy.
Historical advantage estimates are not stored because they correspond to the behavior policy $\pi_{\pmb{\theta}_i}$ rather than the target policy $\pi_{\pmb{\theta}_k}$.
When a trajectory is replayed at iteration $k$, advantages are recomputed as target-policy quantities,
$
A^{\pi_{\pmb{\theta}_k}}(\pmb{s}_t,\pmb{a}_t)
=Q^{\pi_{\pmb{\theta}_k}}(\pmb{s}_t,\pmb{a}_t)
-V^{\pi_{\pmb{\theta}_k}}(\pmb{s}_t).
$
In practice, actor-critic methods use the estimate
$
\widehat A_t^{(k)}
=\widehat Q^{\pi_{\pmb{\theta}_k}}(\pmb{s}_t,\pmb{a}_t)
-\widehat V^{\pi_{\pmb{\theta}_k}}(\pmb{s}_t),
$
computed from stored rewards and next states using the current policy, e.g., TD returns or Generalized Advantage Estimation. Thus, replay changes only the source of transitions; the score function and advantage in $g(\pmb{s},\pmb{a}|\pmb\theta_k)$ remain defined with respect to $\pi_{\pmb{\theta}_k}$, while stored behavior-policy log-likelihoods are used to compute LR or CLR weights.

For Selection Rule~2, the term
$\mathbb{E}[KL(\pi_{\pmb{\theta}_k}(\cdot|\pmb{s}) || \pi_{\pmb{\theta}_i}(\cdot|\pmb{s}))]$ is estimated empirically using states from the candidate trajectory $\mathcal{T}_i$, rather than an analytically specified stationary distribution.
For discrete action spaces, we compute the exact KL divergence from the softmax action distributions at each sampled state. For continuous-control tasks with Gaussian policies and state-independent standard deviations, we use the standard closed-form KL divergence between multivariate Gaussians.

 \begin{algorithm}[ht!]
 \small
\SetAlgoLined
\textbf{Input}: the selection threshold constant $c$; the maximum number of iterations $K$; buffer size $B$; the number of replications per iteration (batch size) $n$; the number of iterations in offline optimization $K_{\text{off}}$; 
the set of historical samples 
$\mathcal{D}_0$; the set of policy parameters visited so far $\mathcal{F}_0$.
\\
\textbf{Initialize} policy $\pmb\theta_1$, set $\pmb\theta_1^0=\pmb\theta_1$, initialize $\mathcal{F}_1=\{\pmb\theta_1\}$, and initialize $\mathcal{D}_0=\emptyset$\;
 \For{$k=1,2,\ldots, K$ 
 }{

 1. \textbf{Execution:} 

 (a) Generate $n$ new samples by following $\pi_{\pmb\theta_k}$, i.e., $\mathcal{T}_k \defeq
\left\{ \left(\pmb{s}^{(k,j)},\pmb{a}^{(k,j)}\right); j=1,\ldots,n
\right\}$; 
 
(b) Store $\mathcal{T}_k$ in the replay dataset: $\mathcal{D}_k\leftarrow\mathcal{D}_{k-1}\cup\mathcal{T}_k$\;

 2. \textbf{Selection:} 
 
 (a) Initialize $\mathcal{U}_k=\emptyset$ and the set $\mathcal{D}_k^{n_0}=\emptyset$\;
 \For{$\pmb\theta_i \in \mathcal{F}_k$}{

 (b) Verify the selection criterion \eqref{eq: selection rule online} or \eqref{eq: simple selection};}
 




 \If{Selection Criterion is Satisfied}{ (c) Update the reuse set: $\mathcal{U}_k \leftarrow \mathcal{U}_k\cup \{ \pmb{\theta}_i\}$; \\
 (d) Sample $n_0$ observations from $\mathcal{T}_i$ with replacement to form $\mathcal{T}_i^{n_0}$, and update
$\mathcal{D}_k^{n_0}=\mathcal{D}_k^{n_0}\cup\mathcal{T}_i^{n_0}$; 
 }
 }
 3. \textbf{Offline Optimization (Algorithm-Specific):} \\
 \For{$h=0, ..., K_{\text{off}}-1$}{
 (a) Sample a mini-batch from $\mathcal{D}_k^{n_0}$ to compute the policy gradient estimates $\widehat{\nabla} J^{LR,h}_{k}$ or $\widehat{\nabla} J^{CLR,h}_{k}$ defined in Equations (\ref{eq.LR-gradient}) and (\ref{eq.CLR-gradient})\;
 (b) Update the policy by
 $\pmb{\theta}_{k}^{h+1} \leftarrow \pmb{\theta}_k^h+ \eta_k\widehat{\nabla} J_{k}^h$
 using the gradient estimate $\widehat{\nabla} J_{k}^h$ from Step 3(a).

 }
 4. Set $\pmb\theta_{k+1} = \pmb\theta_{k}^{K_{\text{off}}}$ 
 and  $\pmb\theta_{k+1}^0 = \pmb\theta_{k+1}$. 
Update the FIFO buffer $\mathcal{F}_{k+1}=\mathcal{F}_k\cup\{\pmb\theta_{k+1}\}$. If $|\mathcal{F}_{k+1}|>B$, remove the oldest policy and its associated trajectory from the replay dataset. 
 

\caption{VRER Assisted Generic Policy Gradient Algorithm (\textbf{PG-VRER})}
\label{algo: online}
\end{algorithm}

\begin{remark}[Downsampling]
The parameter $n_0$ determines how many observations are sampled from each selected historical trajectory in Step~2(d) of Algorithm~\ref{algo: online}. It therefore controls the amount of replay data and associated computational cost, but does not affect the variance-ratio ranking used to construct $\mathcal{U}_k$. Sampling with replacement does not eliminate dependence among replay observations. Historical observations within a trajectory, and replay gradients from nearby behavior policies, may remain correlated. Accordingly, the convergence analysis accounts for this dependence through correlation and mixing terms. Downsampling only limits the amount of correlated or stale replay observations without ensuring independence.
 \end{remark}

\begin{remark}[Buffer Size] The \textit{replay buffer size}, specified by $B$, can be fixed or dynamic. An example of a dynamic approach is setting the buffer size to be proportional to the number of iterations, such as selecting the $B=k\xi$ most recent samples, where $\xi$ represents the ratio of the total number of reusable iterations to the overall iterations. 
\textit{There is always a trade-off on specifying the replay buffer size to balance exploration and exploitation.} 
As~\cite{fedus2020revisiting} pointed out, there is an interplay between: (1)
the improvements caused by increasing the replay capacity and covering large state-action space;
and (2) the deterioration caused by having older policies in the
buffer or overfitting to some ``out-dated"
samples. The magnitude of both effects depends on the particular settings of these quantities. We will delve into the theoretical analysis of buffer size in Section~\ref{sec: convergence analysis} and explore its empirical impact in Section~\ref{sec:empiricalStudy}.
\end{remark}

 \section{Finite-Time Convergence Analysis of PG-VRER}\label{sec: convergence analysis}

We now analyze the asymptotic properties of PG-VRER, focusing on finite-time convergence and buffer management. This analysis addresses the dual challenges of sample dependence induced by Markovian noise and behavior policy interdependence, which render many conventional proofs ineffective. Theorem~\ref{convergence theorem} establishes the convergence of the average gradient norm over $K$ iterations \citep{zhang2019convergence,zhou2017convergence}. The result demonstrates that VRER improves the convergence rate compared to classical PG and explicitly characterizes the impact of learning rate $\eta_k$, mixing rate $\varphi(nt)$, and buffer size $B_K$. Detailed proofs are provided in Appendix~\ref{appendix: sec convergence}.

\begin{sloppypar}
\begin{theorem}[Convergence of PG-VRER]\label{convergence theorem}
Suppose Assumptions \ref{assumption 2} and \ref{assumption 3} hold. Let $\eta_k=\eta_1 k^{-r}$ denote the learning rate used in the $k$-th iteration with two constants $\eta_1\in(0, \frac{1}{4L}]$ and $r\in (0,1)$. By running Algorithm ~\ref{algo: online} with the replay buffer of size $B_K$, for both LR and CLR policy gradient estimators in Eq.~\eqref{eq.LR-gradient} and \eqref{eq.CLR-gradient} and for $K \ge 2t+2B_K$, we have the rate of convergence
 \begin{align*}
    \frac{1}{K}\sum^K_{k=1}\E\left[\Vert\nabla  J(\pmb\theta_k)\Vert^2\right]  & \leq   \frac{8U_ J /\eta_1}{K^{1-r}} + \frac{4cLM^2}{K}\sum_{k=1}^K{\eta_k} \bar{\rho}_k+
    {4M^2}\varphi(nt) + \frac{2^{r+1}C_3}{(1-r)K^r}  \nonumber\\
    &\quad  + \frac{2^{r+1}C_2\eta_1 (t+1)t}{(1-r)K^r} + \frac{2^{r+1}C_1\eta_1}{1-r}\frac{B_K+t}{K^r} + M^2\frac{B_K+t}{K}
 \end{align*}
where $C_1=C^\Gamma_1+C^\Gamma_2$, $C_2={2nM^3U_\pi  U_f}$, $C_3=\sup_{k\geq 1}\Vert \Bias_k\Vert^2$ with $\Bias_k=\E\left[\widehat{\nabla} J^{R}_{k}\right]-\E\left[\nabla J(\pmb\theta_k)\right]$ and $\bar{\rho}_{k}=\frac{1}{|\mathcal{U}_k|^2}\sum_{\pmb\theta_i\in\mathcal{U}_k}\sum_{\pmb\theta_{i^\prime}\in\mathcal{U}_k}\max_{\ell=1,\ldots,d}|\Corr^{(\ell)}_{i,i^\prime,k}|$. Here $R\in\{LR,CLR\},$  $C_1^\Gamma = {LM^2\sqrt{2c+1} (\sqrt{2c+1}+2)}$, $C_2^\Gamma = M^2(U_f L_g+M U_\pi)$,
and $c$ is the selection constant defined in Theorem~\ref{thm:Online selection rule}. Using $\mathcal{O}(\cdot)$ notation gives
 \begin{equation}\label{eq: convergence upper bound}
    \frac{1}{K}\sum^K_{k=1}\E\left[\Vert\nabla  J(\pmb\theta_k)\Vert^2\right] 
    = \mathcal{O}\left(\frac{1}{K^{1-r}}\right) +
     \mathcal{O}\left(\frac{\sum_{k=1}^K{\eta_k}\bar{\rho}_k}{K} \right)+
    \mathcal{O}\left(\varphi(nt)\right) + \mathcal{O}\left(\frac{t^2}{K^r}\right) + \mathcal{O}\left(\frac{B_K+t}{K^r}\right)  
 \end{equation}
\noindent where $n$ is the number of steps in each iteration. The notation
$\mathcal{O}(\cdot)$ hides constants $c$, $L$, $\eta_1$, $M$, $U_ J$, $n$, $U_f$, $U_\pi$, $\kappa_0$, $\kappa$ and $L_g$.
\end{theorem}
Theorem~\ref{convergence theorem} guarantees the local convergence of the proposed PG-VRER algorithm if LR or CLR policy gradient estimators are used in policy updates. It can be seen that the rate of optimal convergence depends on three key factors: (1) sample correlation $\bar{\rho}_k$; (2) mixing rate $\varphi(nt)=\kappa_0 \kappa^{nt}$; and (3) buffer size $B_K$. In short, low sample correlations between reused samples and a faster mixing rate of the environment (i.e., smaller $\kappa$) would improve the convergence. 
\end{sloppypar}

Since $\bar{\rho}_k\leq 1$, when the learning rate $\eta_k=\eta_1 k^{-r}$ with $\eta_1\in(0, \frac{1}{4L}]$ and $r\in (0,1)$, the second term of Eq.~\eqref{eq: convergence upper bound} can be simplified to $\frac{1}{K}\sum_{k=1}^K{\eta_k}\bar{\rho}_k\leq\frac{1}{K}\sum_{k=1}^K{\eta_k}\leq\frac{1}{1-r} K^{-r}=\mathcal{O}\left(\frac{1}{K^r}\right).$
\begin{sloppypar}
\begin{corollary}\label{corollary: convergence rate}
    Suppose Assumptions \ref{assumption 2} and \ref{assumption 3} hold. Under the same configurations as Theorem~\ref{convergence theorem}, we have the rate of convergence
 \begin{equation*}
    \frac{1}{K}\sum^K_{k=1}\E\left[\Vert\nabla  J(\pmb\theta_k)\Vert^2\right] 
    = \mathcal{O}\left(\frac{1}{K^{1-r}}\right) + \mathcal{O}\left(\frac{1}{K^r}\right) + 
    \mathcal{O}\left(\varphi(nt)\right) + \mathcal{O}\left(\frac{t^2}{K^r}\right) + \mathcal{O}\left(\frac{B_K+t}{K^r}\right).  
 \end{equation*}
\end{corollary} 
\end{sloppypar}

In our detailed analysis of the convergence rate outlined in Theorem~\ref{convergence theorem}, it can be observed that the VRER contributes to the stabilization of the gradient, as evidenced by the term $\bar{\rho}_k$. Nonetheless, based on the insights from Corollary~\ref{corollary: convergence rate}, this impact, while notable, does not constitute a limiting factor in the rate of convergence, when compared to the effect of sample-dependence, including mixing rate $\mathcal{O}(\varphi(nt))$ and gradient bias terms $\mathcal{O}\left(\frac{t^2}{K^r}\right) + \mathcal{O}\left(\frac{B_K+t}{K^r}\right)$.



\begin{sloppypar}
\begin{corollary}\label{cor: convergence rate} 
    Suppose Assumptions \ref{assumption 2} and \ref{assumption 3} hold. Under the same configurations as Theorem~\ref{convergence theorem}, by setting $t=t^\star_K=\log_{\kappa}K^{-r/n}$, we have the rate of convergence
 \begin{equation*}
    \frac{1}{K}\sum^K_{k=1}\E\left[\Vert\nabla  J(\pmb\theta_k)\Vert^2\right] 
    \leq \mathcal{O}\left(\frac{1}{K^{1-r}}\right) + \mathcal{O}\left(\frac{(t_K^\star)^2}{K^r}\right) + \mathcal{O}\left(\frac{B_K+t_K^\star}{K^r}\right)  
 \end{equation*}
\noindent where the notation
$\mathcal{O}(\cdot)$ hides constants $c$, $L$, $\eta_1$, $M$, $U_ J$, $n$, $U_f$, $U_\pi$, $\kappa_0$, $\kappa$ and $L_g$.
\end{corollary} 
\end{sloppypar}

Corollary~\ref{cor: convergence rate} establishes a convergence rate for the average squared norm of the policy gradient for the PG-VRER algorithm when the lag term equals the mixing time $t=t_K^\star$. Its proof can be found in Appendix~\ref{appendix subsec: corollary convergence}. By letting 
$q=\min\{1-r,r\}$, the \textit{convergence rate} can be expressed as $\mathcal{O}\left(\frac{B_K+ (t_K^\star)^2}{K^q}\right)$ influenced by both the buffer size $B_K$ and a term related to the mixing time $t_K^\star$. Corollary~\ref{cor: convergence rate}  indicates that the convergence is only guaranteed when the buffer size scales at a rate lower than $\mathcal{O}(K^{r})$, acknowledging that replaying old samples introduces extra bias.

\section{Empirical Study}
\label{sec:empiricalStudy}

\begin{sloppypar}
We evaluate PG-VRER on benchmark reinforcement learning tasks from Gym \citep{brockman2016openai} and PyBullet \citep{benelot2018}. These environments are widely used for assessing RL algorithms. Our objective is to assess the finite sample performance of the proposed PG-VRER with three 
state-of-the-art (SOTA) policy optimization approaches and examine its stability and sensitivity. Throughout the experiments, we use the simplified selection rule in \eqref{eq: simple selection} to construct the reuse set efficiently. For the benchmark experiments, learning curves report the across-run mean with 95\% confidence intervals (CIs), while tables report mean $\pm$ standard deviation. 
For the benchmark comparisons, we selected moderate replay windows ($B=200$--$400$; Table~3, Appendix~A) based on the sensitivity analysis in Figure~\ref{fig: sensitivity buffer}, balancing replay opportunities against the slower convergence observed for larger windows. We also used a common selection constant $c=1.05$, guided by Figure~\ref{fig:sensitivity-selection}, as a robust shared setting rather than tuning $c$ separately for each algorithm.
We use Adam \citep{kingma2015adam} with a learning rate of 0.0003 and a discount factor of $\gamma=0.99$ for all tasks. Additional implementation details are provided in Appendix~\ref{sec:experiment_details}.

\end{sloppypar}

\subsection{Comparison of Policy Optimization with and without VRER}
\label{subsec: performance}

We evaluate VRER on four classical control benchmarks using
TRPO~\citep{schulman2015trust}, PPO~\citep{schulman2017proximal}, and A2C 
(Advantage Actor-Critic), 
the synchronous counterpart of A3C \citep{mnih2016asynchronous}. We select A2C because it has been shown to perform comparably to or better than A3C \citep{schulman2017proximal}.
For A2C and PPO, the policy and value functions share a multilayer perceptron with two hidden layers of 64 units and tanh activations. TRPO uses separate actor and critic networks, each with two hidden layers of 32 units. For discrete-action tasks, the actor outputs action probabilities through a softmax layer. For continuous-action tasks, we adopt a Gaussian policy with fixed standard deviations following \cite{metelli2020importance}. The historical sample selection threshold is fixed at $c=1.05$.

\begin{figure}[h]
\vspace{-0.in}
	\centering
	\includegraphics[width=1\textwidth]{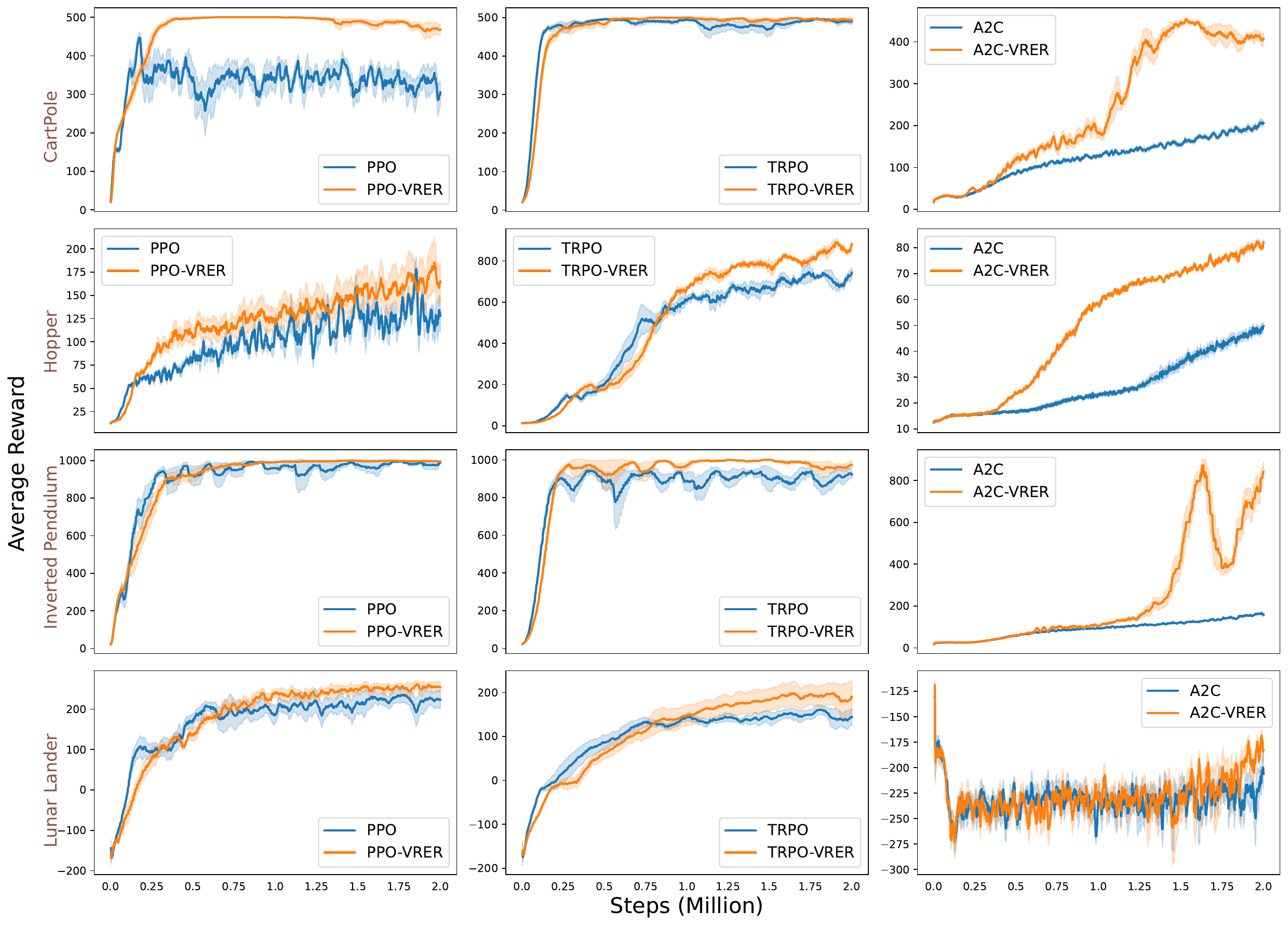}
	\vspace{-0.1in}
	\caption{Convergence results for the various PG algorithms with and without using VRER.}
	\label{fig: convergence comparison}
  \vspace{-0.1in}
\end{figure}

Figure~\ref{fig: convergence comparison} shows the mean learning curves and 95\% confidence intervals (CIs) across all tasks. Overall, VRER consistently improves the performance of policy-gradient methods. In the case of A2C, the incorporation of VRER leads to a remarkable enhancement in both convergence speed and stability, except for the Lunar Lander task, where A2C faced challenges in convergence even with VRER assistance over 2 million steps. For the experiments involving the PPO algorithm, PPO-VRER shows not only enhanced stability but also accelerated convergence when compared to its non-VRER counterparts. Across all three scenarios, PPO-VRER exhibits a noteworthy improvement in performance. While the gains for TRPO are smaller, TRPO-VRER consistently outperforms standard TRPO and achieves higher final rewards with smoother learning dynamics.

\begin{table}[h]\centering \small
\vspace{-0.in}
\caption{Average reward (mean $\pm$ SD) over the final 10,000 timesteps across five random seeds. Bold values indicate statistically significant differences from the corresponding baseline.
} \label{table: performance}
\begin{tabular}{@{}lcccc@{}}
\toprule
          & CartPole      & Hopper         & Inverted Pendulum & Lunar Lander     \\ \midrule
PPO       & 327.84 ± 8.97 & 135.72 ± 7.92  & 977.48 ± 10.02    & 224.43 ± 8.08   \\
PPO-VRER  & \textbf{468.19 ± 5.06} & \textbf{168.96 ± 12.49} & \textbf{996.35 ± 1.41}     & \textbf{255.57 ± 3.44}   \\ \midrule
TRPO      & 489.34 ± 2.62 & 700.52 ± 11.64 & 916.39 ± 11.55    & 139.17 ± 4.22   \\
TRPO-VRER & \textbf{494.09 ± 3.14} & \textbf{864.12 ± 9.79}  & \textbf{962.76 ± 3.63}     & \textbf{184.16 ± 9.26}   \\ \midrule
A2C       & 197.4 ± 3.35  & 47.75 ± 0.62   & 158.86 ± 3.71     & -221.99 ± 11.25 \\
A2C-VRER  & \textbf{411.46 ± 6.4}  & \textbf{80.58 ± 0.67}   & \textbf{737.22 ± 19.62}    & \textbf{-186.53 ± 4.88}  \\ \bottomrule
\end{tabular}
\vspace{-0.in}
\end{table}

Table~\ref{table: performance} summarizes the mean performance over the final 10,000 timesteps, following \cite{schulman2017proximal}. The benefits of VRER are evident across all tasks. On CartPole, VRER improves average returns by 43.0\% for PPO, 1.0\% for TRPO, and 108.4\% for A2C. On Hopper, the gains are 24.5\%, 23.3\%, and 68.8\%, respectively. On Inverted Pendulum, improvements reach 1.9\%, 5.1\%, and 364.1\%, while on Lunar Lander they are 13.9\%, 32.3\%, and 16.0\%, respectively. We also compare PG-VRER with ACER \citep{wang2017sample}, an off-policy actor-critic method designed to improve sample efficiency via experience replay. Owing to implementation limitations in commonly used libraries such as Stable-Baselines3 \citep{stable-baselines3}, this comparison is restricted to discrete-action environments. As reported in Appendix~\ref{append sec: additional experimental result}, PG-VRER consistently outperforms ACER.

Overall, Figure~\ref{fig: convergence comparison} and Table~\ref{table: performance} demonstrate that VRER consistently enhances policy-gradient methods by accelerating convergence, improving stability, and increasing final performance across a diverse set of control tasks.

\subsection{Sensitivity Analysis of Buffer Size and Selection Constant}
\label{subsec: sensitivity}
We investigate the sensitivity of VRER to the buffer size $B$ and selection constant $c$.
Figure~\ref{fig: sensitivity buffer} shows the learning curves of PPO-VRER, TRPO-VRER, and A2C-VRER under different buffer sizes. Overall, VRER is robust to the choice of $B$. 
For PPO-VRER, convergence becomes noticeably slower when the buffer size increases to $B=1000$ and $1500$. In contrast, TRPO-VRER exhibits similar performance across all buffer sizes. For A2C-VRER, larger buffers improve convergence, with $B=1000$ and $1500$ achieving the best performance.
The number of steps equals to $(B\times n\times \text{num-envs})$, where ``num-envs" denotes the number of parallel environments (see Table~\ref{table:merged_hyperparameters} in Appendix~\ref{sec:experiment_details}).

\begin{figure}[h]
\vspace{-0.in}
	\centering
	\includegraphics[width=1\textwidth]{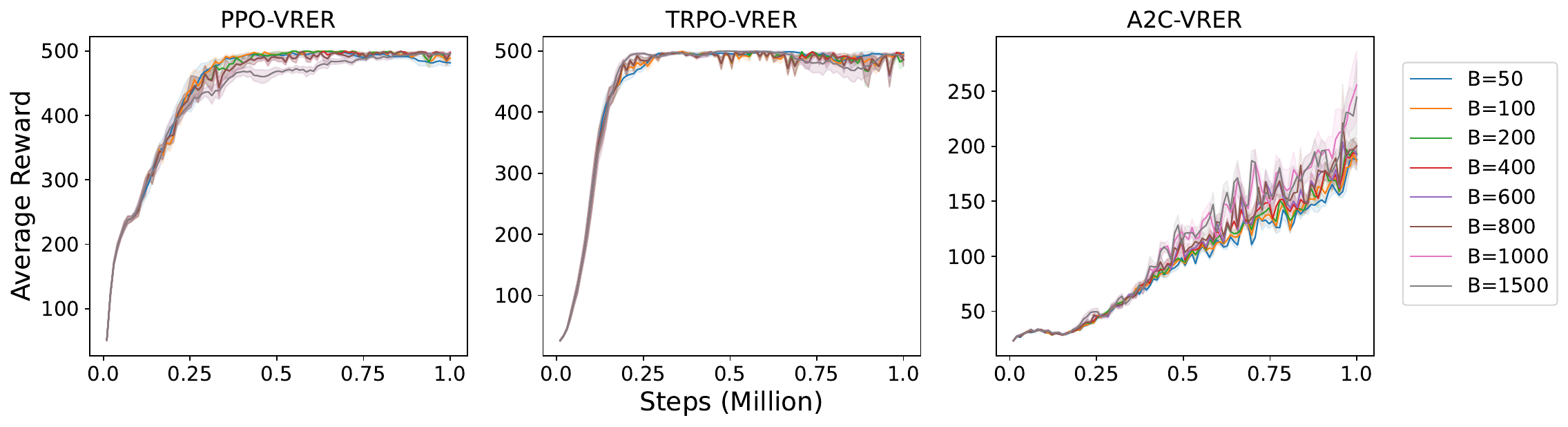}
    \vspace{-0.2in}
	\caption{Sensitivity analysis of buffer sizes. The average rewards achieved by PPO-VRER, TRPO-VRER, and A2C-VRER in the CartPole task with different buffer sizes.}
 \vspace{-0.in}
	\label{fig: sensitivity buffer}
\end{figure}

\begin{figure}[h]
\vspace{-0.in}
	\centering
	\includegraphics[width=1\textwidth]{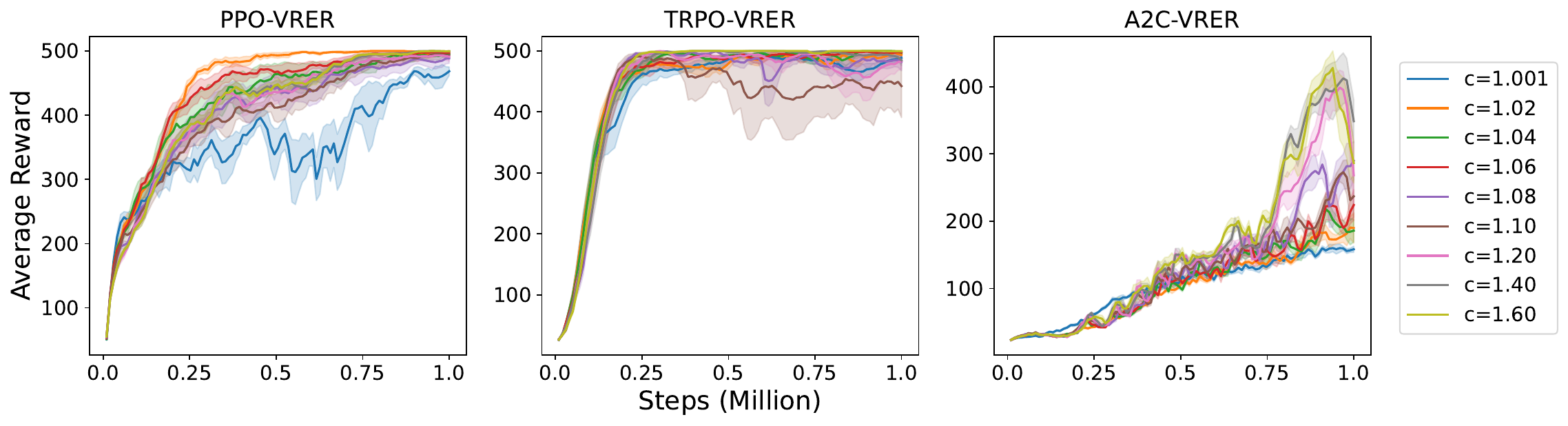}
    \vspace{-0.2in}
	\caption{Sensitivity of the Selection Constant $c$. The average rewards achieved by PPO-VRER, TRPO-VRER, and A2C-VRER in the CartPole task with different $c$.}
	\label{fig:sensitivity-selection}
\end{figure}
Figure~\ref{fig:sensitivity-selection} examines the effect of the selection constant $c$. Results are reported as the mean and 95\% CI over one million timesteps.
The impact of $c$ varies across algorithms: PPO-VRER converges fastest at $c=1.02$, A2C-VRER at $c=1.6$, whereas TRPO-VRER exhibits comparable performance across the tested values. These findings are consistent with the reuse patterns reported in Table~\ref{table: reuse ratio} (Appendix~\ref{append sec: additional experimental result}). Since $c$ determines the reuse ratio $|\mathcal{U}_k|/B$, 
it controls the trade-off between exploiting historical samples and limiting stale experiences.
PPO benefits from moderate reuse, A2C from more extensive reuse, and TRPO appears comparatively insensitive to the reuse level.

\subsection{Gradient Variance Reduction}
\label{subsec: reuse and variance}

We assess the ability of VRER to reduce policy-gradient variance.
Table~\ref{table: gradient variance reduction} reports the average reduction in relative gradient variance, ${\Tr(\Var[\widehat{\nabla} J^{PG}_k])}/{\Vert\E[\widehat{\nabla} J^{PG}_k]\Vert^2}-{\Tr(\Var[\widehat{\nabla} J^{CLR}_k])}/{\Vert\E[\widehat{\nabla} J^{CLR}_k]\Vert^2}$, for PPO, TRPO, and A2C on the CartPole task under different selection constants. The estimated reduction is positive for all three algorithms across most tested values of $c$, indicating that VRER generally lowers relative gradient variance. The only exception occurs for TRPO at $c=1.001$, where reuse is negligible ($0.00\pm0.00$) and the variance-reduction estimate is $-5.58\pm11.33$. The reported standard deviation exceeds the magnitude of the negative mean, so these summary statistics do not establish a clear variance increase. At larger values of $c$, the estimated reduction is positive for all three algorithms and is most pronounced for PPO.

\begin{table}[h]\centering
\vspace{-0.in}
\caption{Relative gradient estimation variance reduction from VRER for PPO, TRPO, and A2C in the CartPole task for different selection constants, reported as mean $\pm$ standard deviation.
} \label{table: gradient variance reduction}
\footnotesize
\begin{tabular}{@{}l|c|c|c@{}}
\toprule
\begin{tabular}[c]{@{}l@{}}Selection\\ Constant ($c$) \end{tabular} & PPO - (PPO-VRER)  & TRPO - (TRPO-VRER)&  A2C - (A2C-VRER) \\ \midrule
1.001                                                         & 6.94 ± 6.30   & -5.58 ± 11.33     & 1.09 ± 0.30   \\
1.02                                                         & 16.07 ± 2.60  & 2.23 ± 6.62     & 1.51 ± 0.18   \\
1.04                                                         & 24.42 ± 3.51  & 3.62 ± 4.68     & 1.37 ± 0.20   \\
1.06                                                         & 25.32 ± 3.39  & 4.81 ± 2.10     & 1.47 ± 0.16   \\
1.08                                                         & 24.01 ± 3.44  & 3.14 ± 6.45     & 1.45 ± 0.12   \\
1.10                                                         & 23.80 ± 4.07  & 5.84 ± 5.35     & 1.45 ± 0.17   \\
1.20                                                         & 24.92 ± 3.75  & 6.70 ± 5.72     & 1.48 ± 0.16   \\
1.40                                                         & 23.84 ± 3.80  & 7.91 ± 5.60     & 1.50 ± 0.17   \\
1.60                                                         & 23.48 ± 3.33  & 9.04 ± 4.99     & 1.49 ± 0.18   \\ \bottomrule
\end{tabular}
\vspace{-0.in}
\end{table}

\subsection{Validation of the KL-Based Selection Approximation}
\label{subsec: kl approximation diagnostic}

We compare the rankings and replay selections produced by Rules 1 and 2 using a high-precision offline implementation of Rule 1, since its exact variance criterion is computationally impractical for online replay screening.
At iteration $k$, Rule~1 evaluates each candidate $i\in\mathcal F_k$ using the exact variance ratio
$
R_{i,k}
\defeq
\frac{\Tr\!\left(\Var[\widehat{\nabla}J^{LR}_{i,k}]\right)}
{\Tr\!\left(\Var[\widehat{\nabla}J^{PG}_{k}]\right)}$,
whereas Rule~2 uses the KL-based approximation $\widetilde R_{i,k}$ from Proposition~\ref{prop: approximation of variance ratio}. We assess agreement using two metrics.
The first is the within-buffer Spearman rank correlation,
$
\rho_k^{\mathrm S}
=\Corr\!(
\operatorname{rank}\{R_{i,k}:i\in\mathcal{F}_k\},
\operatorname{rank}\{\widetilde R_{i,k}:i\in\mathcal{F}_k\}
)$, which measures ranking consistency.
The second is the matched-budget Jaccard overlap, 
$
\operatorname{Jaccard}_k
=
\frac{|\mathcal{S}^{\mathrm{LHS}}_k\cap\mathcal{S}^{\mathrm{RHS}}_k|}
{|\mathcal{S}^{\mathrm{LHS}}_k\cup\mathcal{S}^{\mathrm{RHS}}_k|},
$
which measures agreement between reuse sets.
Here, $\mathcal S_k^{\mathrm{LHS}}$ is the set selected by Rule~1 with $c=1.05$, and $\mathcal S_k^{\mathrm{RHS}}$ contains the $|\mathcal S_k^{\mathrm{LHS}}|$ candidates with the smallest $\widetilde R_{i,k}$ values. 

\begin{figure}[h]
\centering
\includegraphics[width=0.85\textwidth]{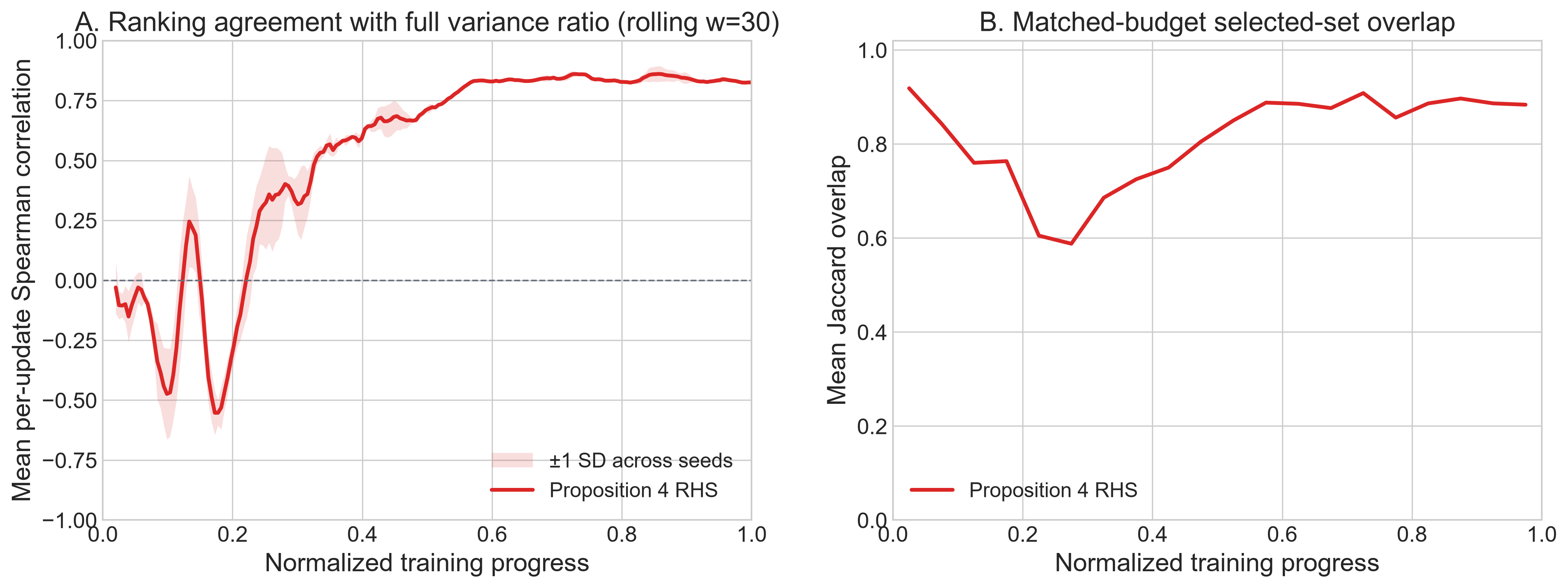}
\vspace{-0.in}
\caption{Ranking and selection agreement between Selection Rules 1 and 2.
}
\label{fig:kl_ranking_overlap_main}
\end{figure}


We evaluate both metrics on CartPole-v1 over $4$ million environment steps using three random seeds. Figure~\ref{fig:kl_ranking_overlap_main} shows that agreement is unstable during the early training phase but becomes consistently high thereafter.
The rolling Spearman correlation stabilizes around $0.82$--$0.87$, while the Jaccard overlap increases from roughly $0.58$--$0.60$ to mostly $0.85$--$0.90$. These results indicate that the low-cost KL proxy yields rankings and reuse sets that closely match those obtained from the computationally expensive variance-ratio criterion. Overall, the findings support Selection Rule~2 as an effective practical surrogate for Selection Rule~1 after the initial training phase, while recognizing that the approximation is local rather than exact (see Section~\ref{subsec: gradient variance ratio approximation}).

\subsection{Bias--Variance Diagnostic}
\label{subsec: bias variance diagnostic}

Variance alone does not determine replay-estimator quality because historical-policy mismatch, state-distribution mismatch, Markovian dependence, and CLR clipping can introduce bias. At each iteration, we therefore use a large on-policy batch to obtain a reference gradient and decompose repeated replay-gradient estimates into bias and variance components for both LR and CLR. Figure~\ref{fig:bias_variance_main} reports this diagnostic for the proposed variance-selected replay method.

\begin{figure}[htbp]
\centering
\includegraphics[width=1\textwidth]{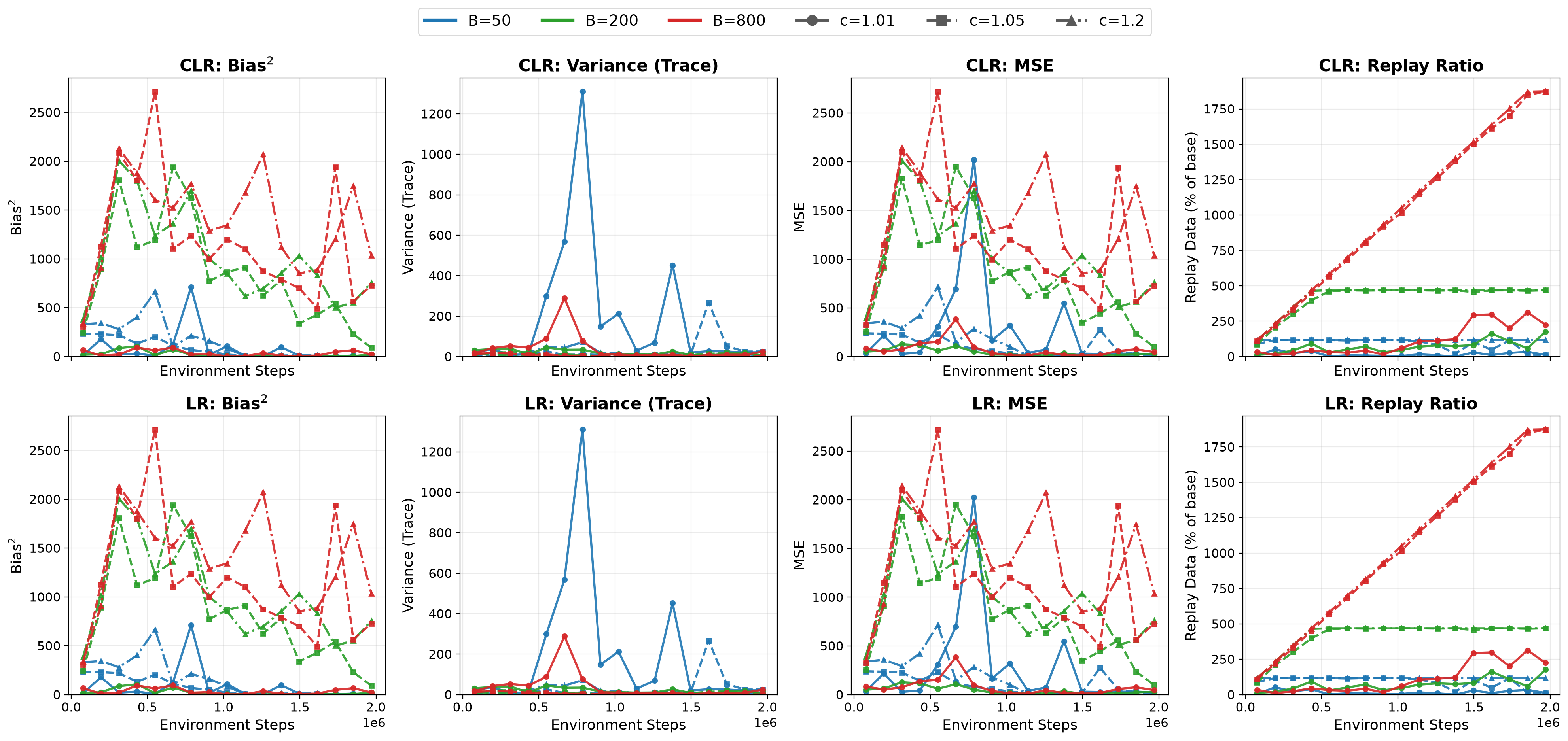}
\vspace{-0.1in}
\caption{Bias--variance diagnostics for variance-selected LR and CLR replay across $B\in\{50,200,800\}$ and $c\in\{1.01,1.05,1.2\}$. Colors indicate $B$; marker shapes indicate $c$.}
\label{fig:bias_variance_main}
\end{figure}

The results show that $B$ and $c$ govern the empirical bias--variance trade-off. Very strict selection ($B=50$, $c=1.01$) admits too little replay data to achieve stable variance reduction, yielding high and intermittent trace variance but low bias. 
In contrast, moderate and permissive settings generally maintain low variance while reusing substantially more data, at the cost of increased bias.
For larger buffers ($B=200$ and especially $B=800$), permissive thresholds ($c=1.05$ and $c=1.2$) admit considerably more historical data, producing a high-bias, bias-dominated MSE regime.
Larger buffers do not inherently increase reuse of old samples, but they make older policies available, and permissive thresholds are more likely to accept these potentially stale samples.
Consequently, increasing both the amount and age of replay samples raises the risk of stale-policy bias. We do not claim a monotonic relationship between replay ratio and bias, however, since replay sets of similar size can differ substantially in age distribution and policy mismatch.


\section{Conclusion}
We introduced a VRER replay framework that improves the sample efficiency of policy optimization by combining experience replay with a variance-based sample selection rule. By selectively reusing relevant historical samples and weighting them according to their likelihood under the current policy, VRER controls the variance of individual replay estimators. The aggregate estimator is also affected by cross-estimator correlations, while stale-policy mismatch and clipping can introduce bias. Consequently, our theoretical analysis and empirical diagnostics support VRER as a controllable bias--variance replay mechanism rather than an unconditional guarantee of variance reduction.

More broadly, experience replay in reinforcement learning lacks a rigorous understanding of its finite-time convergence and asymptotic behavior. Our theoretical framework helps address this gap by explicitly accounting for Markovian sample dependence and behavior-policy mismatch. The analysis explains why reusing overly stale samples can degrade policy optimization and how buffer size affects performance through the bias--variance trade-off of policy-gradient estimation.
Using this framework, we further establish finite-time convergence results for PG-VRER, providing theoretical support for replay-based policy optimization.


Finally, our buffer-dependent bias bounds and bias--variance diagnostics show that replay quality cannot be assessed through variance alone:
aggressive reuse may reduce variance while increasing squared bias and overall MSE. Since replay-induced bias is difficult to estimate reliably during training, developing replay strategies that directly optimize the mean squared error of policy-gradient estimators remains an important direction for future research.

\begin{sloppypar}
\section*{Funding Statement}
The authors acknowledge support from the National Institute of Standards and Technology (Grants 70NANB24H293, 70NANB17H002, 70NANB21H086) and the National Science Foundation (CMMI-2442970).
\end{sloppypar}

\bibliography{wscbib}

@article{KeqiIVT,
author = {Wang, Keqi and Choy, Keilung and Reiser, Eli and Pei, Jinxiang and Zheng, Hua and Dasgupta, Aparajita and Cheng, Fuqiang and Dong, Guogang and Mulukutla, Bhanu Chandra and Mannheimer, Joshua and Huang, Carolyn and Farsani, Hooman and Xie, Wei},
title = {A Modular Mechanistic In Silico Model for In Vitro Transcription Process Yield and Product Quality Prediction},
journal = {Biotechnology and Bioengineering},
volume = {123},
number = {7},
pages = {1836-1870},
doi = {https://doi.org/10.1002/bit.70222},
year = {2026}
}

@article{Zeng2026GlycoMAC,
  author = {Yuming Zeng and Sarah W. Harcum and Jinxiang Pei and Wei Xie},
  title = {{GlycoMAC}: A Multiscale Metabolic-Glycosylation Framework for Predicting Glycosylation Across Conditions in Mammalian Cell Cultures},
  journal = {Biotechnology and Bioengineering},
  year = {2026},
  doi = {http://doi.org/10.1002/bit.70393},
  url = {http://doi.org/10.1002/bit.70393}
}

@inproceedings{mnih2016asynchronous,
  title={Asynchronous methods for deep reinforcement learning},
  author={Mnih, Volodymyr and Badia, Adria Puigdomenech and Mirza, Mehdi and Graves, Alex and Lillicrap, Timothy and Harley, Tim and Silver, David and Kavukcuoglu, Koray},
  booktitle={International Conference on Machine Learning},
  pages={1928--1937},
  year={2016},
  organization={PMLR}
}

@article{Williams1992simple,
author = {Williams, Ronald J.},
title = {Simple Statistical Gradient-Following Algorithms for Connectionist Reinforcement Learning},
year = {1992},
issue_date = {May 1992},
publisher = {Kluwer Academic Publishers},
address = {USA},
volume = {8},
number = {3–4},
issn = {0885-6125},
doi = {10.1007/BF00992696},
journal = {Machine learning},
month = may,
pages = {229–256},
numpages = {28}
}

@inproceedings{kingma2015adam ,
  author    = {Diederik P. Kingma and
               Jimmy Ba},
  editor    = {Yoshua Bengio and
               Yann LeCun},
  title     = {Adam: {A} Method for Stochastic Optimization},
  booktitle={International Conference on Learning Representations},
  year      = {2015},
  bibsource = {dblp computer science bibliography, https://dblp.org}
}

@inproceedings{Schaul2016PrioritizedER,
  title={Prioritized Experience Replay},
  author={Tom Schaul and John Quan and Ioannis Antonoglou and David Silver},
  booktitle={International Conference on Learning Representations},
  year={2016}
}

@book{sutton2018reinforcement,
author = {Sutton, Richard S. and Barto, Andrew G.},
title = {Reinforcement Learning: An introduction},
year = {2018},
isbn = {0262039249},
publisher = {A Bradford Book},
address = {Cambridge, MA, USA}
}

@article{mnih2015humanlevel,
  author = {Mnih, Volodymyr and Kavukcuoglu, Koray and Silver, David and Rusu, Andrei A. and Veness, Joel and Bellemare, Marc G. and Graves, Alex and Riedmiller, Martin and Fidjeland, Andreas K. and Ostrovski, Georg and Petersen, Stig and Beattie, Charles and Sadik, Amir and Antonoglou, Ioannis and King, Helen and Kumaran, Dharshan and Wierstra, Daan and Legg, Shane and Hassabis, Demis},
  description = {Human-level control through deep reinforcement learning - nature14236.pdf},
  issn = {00280836},
  journal = {Nature},
  month = feb,
  number = 7540,
  pages = {529--533},
  publisher = {Nature Publishing Group, a division of Macmillan Publishers Limited. All Rights Reserved.},
  title = {Human-level control through deep reinforcement learning},
  volume = 518,
  year = 2015
}

@article{zheng2021personalized,
  title={Personalized Multimorbidity Management for Patients with Type 2 Diabetes Using Reinforcement Learning of Electronic Health Records},
  author={Zheng, Hua and Ryzhov, Ilya O and Xie, Wei and Zhong, Judy},
  journal={Drugs},
  pages={1--12},
  year={2021},
  publisher={Springer}
}

@inproceedings{schulman2015high,
  author       = {John Schulman and
                  Philipp Moritz and
                  Sergey Levine and
                  Michael I. Jordan and
                  Pieter Abbeel},
  editor       = {Yoshua Bengio and
                  Yann LeCun},
  title        = {High-Dimensional Continuous Control Using Generalized Advantage Estimation},
  booktitle={International Conference on Learning Representations},
  year         = {2016}
}

@book{nesterov2003introductory,
  title={Introductory lectures on convex optimization: A basic course},
  author={Nesterov, Yurii},
  volume={87},
  year={2003},
  publisher={Springer Science \& Business Media}
}

@INPROCEEDINGS{zheng2020green,
  author={Zheng, Hua and Xie, Wei and Feng, M. Ben},
  booktitle={2020 Winter Simulation Conference (WSC)}, 
  title={Green Simulation Assisted Reinforcement Learning With Model Risk for Biomanufacturing Learning and Control}, 
  year={2020},
  volume={},
  number={},
  pages={337-348},
  doi={10.1109/WSC48552.2020.9384107}}

@inproceedings{espeholt2018impala,
  title={Impala: Scalable distributed deep-rl with importance weighted actor-learner architectures},
  author={Espeholt, Lasse and Soyer, Hubert and Munos, Remi and Simonyan, Karen and Mnih, Vlad and Ward, Tom and Doron, Yotam and Firoiu, Vlad and Harley, Tim and Dunning, Iain and others},
  booktitle={International Conference on Machine Learning},
  pages={1407--1416},
  year={2018},
  organization={PMLR}
}

@article{degris2012off,
  title={Off-policy actor-critic},
  author={Degris, Thomas and White, Martha and Sutton, Richard S},
  journal={arXiv preprint arXiv:1205.4839},
  year={2012}
}

@inproceedings{wang2017sample,
  author    = {Ziyu Wang and
               Victor Bapst and
               Nicolas Heess and
               Volodymyr Mnih and
               R{\'{e}}mi Munos and
               Koray Kavukcuoglu and
               Nando de Freitas},
  title     = {Sample Efficient Actor-Critic with Experience Replay},
  booktitle={International Conference on Learning Representations},
  publisher = {OpenReview.net},
  year      = {2017},
}

@inproceedings{schulman2015trust,
  title={Trust region policy optimization},
  author={Schulman, John and Levine, Sergey and Abbeel, Pieter and Jordan, Michael and Moritz, Philipp},
  booktitle={International Conference on Machine Learning},
  pages={1889--1897},
  year={2015},
  organization={PMLR}
}

@article{zheng2021policy,
  title={Policy Optimization in Dynamic Bayesian Network Hybrid Models of Biomanufacturing Processes},
  author={Zheng, Hua and Xie, Wei and Ryzhov, Ilya O and Xie, Dongming},
  journal={INFORMS Journal on Computing},
  volume={35},
  number={1},
  pages={66--82},
  year={2023},
  publisher={INFORMS}
}

@article{brockman2016openai,
  title={Openai gym},
  author={Brockman, Greg and Cheung, Vicki and Pettersson, Ludwig and Schneider, Jonas and Schulman, John and Tang, Jie and Zaremba, Wojciech},
  journal={arXiv preprint arXiv:1606.01540},
  year={2016}
}

@inproceedings{zhang2019convergence,
  title={Convergence and iteration complexity of policy gradient method for infinite-horizon reinforcement learning},
  author={Zhang, Kaiqing and Koppel, Alec and Zhu, Hao and Ba{\c{s}}ar, Tamer},
  booktitle={2019 IEEE 58th Conference on Decision and Control (CDC)},
  pages={7415--7422},
  year={2019},
  organization={IEEE}
}

@inproceedings{zhou2017convergence,
  title     = {On the Convergence Properties of a K-step Averaging Stochastic Gradient Descent Algorithm for Nonconvex Optimization},
  author    = {Fan Zhou and Guojing Cong},
  booktitle = {Proceedings of the Twenty-Seventh International Joint Conference on
               Artificial Intelligence, {IJCAI-18}},
  publisher = {International Joint Conferences on Artificial Intelligence Organization},             
  pages     = {3219--3227},
  year      = {2018},
  month     = {7},
}

@article{metelli2020importance,
  title={Importance Sampling Techniques for Policy Optimization.},
  author={Metelli, Alberto Maria and Papini, Matteo and Montali, Nico and Restelli, Marcello},
  journal={J. Mach. Learn. Res.},
  volume={21},
  pages={141--1},
  year={2020}
}

@article{schulman2017proximal,
  title={Proximal policy optimization algorithms},
  author={Schulman, John and Wolski, Filip and Dhariwal, Prafulla and Radford, Alec and Klimov, Oleg},
  journal={arXiv preprint arXiv:1707.06347},
  year={2017}
}

@article{ionides2008truncated,
  title={Truncated importance sampling},
  author={Ionides, Edward L},
  journal={Journal of Computational and Graphical Statistics},
  volume={17},
  number={2},
  pages={295--311},
  year={2008},
  publisher={Taylor \& Francis}
}

@inproceedings{fedus2020revisiting,
  title={Revisiting fundamentals of experience replay},
  author={Fedus, William and Ramachandran, Prajit and Agarwal, Rishabh and Bengio, Yoshua and Larochelle, Hugo and Rowland, Mark and Dabney, Will},
  booktitle={International Conference on Machine Learning},
  pages={3061--3071},
  year={2020},
  organization={PMLR}
}

@inproceedings{schlegel2019importance,
 author = {Schlegel, Matthew and Chung, Wesley and Graves, Daniel and Qian, Jian and White, Martha},
 booktitle = {Advances in Neural Information Processing Systems},
 editor = {H. Wallach and H. Larochelle and A. Beygelzimer and F. d\textquotesingle Alch\'{e}-Buc and E. Fox and R. Garnett},
 pages = {},
 publisher = {Curran Associates, Inc.},
 title = {Importance Resampling for Off-policy Prediction},
 volume = {32},
 year = {2019}
}

@book{hall2012handbook,
  title={Handbook of healthcare system scheduling},
  author={Hall, Randolph W and others},
  year={2012},
  publisher={Springer}
}

@inproceedings{sutton1999policy,
 author = {Sutton, Richard S and McAllester, David and Singh, Satinder and Mansour, Yishay},
 booktitle = {Advances in Neural Information Processing Systems},
 editor = {S. Solla and T. Leen and K. M\"{u}ller},
 pages = {},
 publisher = {MIT Press},
 title = {Policy Gradient Methods for Reinforcement Learning with Function Approximation},
 volume = {12},
 year = {1999}
}

@inproceedings{Konda1999actor,
 author = {Konda, Vijay and Tsitsiklis, John},
 booktitle = {Advances in Neural Information Processing Systems},
 editor = {S. Solla and T. Leen and K. M\"{u}ller},
 pages = {},
 publisher = {MIT Press},
 title = {Actor-Critic Algorithms},
 volume = {12},
 year = {1999}
}

@book{Owen2013monte,
   author = {Art B. Owen},
   year = 2013,
   title = {Monte Carlo theory, methods and examples}
}

@article{zhang2020global,
author = {Zhang, Kaiqing and Koppel, Alec and Zhu, Hao and Ba\c{s}ar, Tamer},
title = {Global Convergence of Policy Gradient Methods to (Almost) Locally Optimal Policies},
journal = {SIAM Journal on Control and Optimization},
volume = {58},
number = {6},
pages = {3586-3612},
year = {2020},
doi = {10.1137/19M1288012},
eprint = { 
        https://doi.org/10.1137/19M1288012
}
}

@inproceedings{wu2020finite,
 author = {Wu, Yue Frank and ZHANG, Weitong and Xu, Pan and Gu, Quanquan},
 booktitle = {Advances in Neural Information Processing Systems},
 editor = {H. Larochelle and M. Ranzato and R. Hadsell and M.F. Balcan and H. Lin},
 pages = {17617--17628},
 publisher = {Curran Associates, Inc.},
 title = {A Finite-Time Analysis of Two Time-Scale Actor-Critic Methods},
 volume = {33},
 year = {2020}
}

@inproceedings{xu2020improving,
 author = {Xu, Tengyu and Wang, Zhe and Liang, Yingbin},
 booktitle = {Advances in Neural Information Processing Systems},
 editor = {H. Larochelle and M. Ranzato and R. Hadsell and M.F. Balcan and H. Lin},
 pages = {4358--4369},
 publisher = {Curran Associates, Inc.},
 title = {Improving Sample Complexity Bounds for (Natural) Actor-Critic Algorithms},
 volume = {33},
 year = {2020}
}

@inproceedings{zou2019finite,
 author = {Zou, Shaofeng and Xu, Tengyu and Liang, Yingbin},
 booktitle = {Advances in Neural Information Processing Systems},
 editor = {H. Wallach and H. Larochelle and A. Beygelzimer and F. d\textquotesingle Alch\'{e}-Buc and E. Fox and R. Garnett},
 pages = {},
 publisher = {Curran Associates, Inc.},
 title = {Finite-Sample Analysis for SARSA with Linear Function Approximation},
 volume = {32},
 year = {2019}
}

@article{zhang2017deeper,
 title={A deeper look at experience replay},
author={Zhang, Shangtong and Sutton, Richard S},
journal={Deep Reinforcement Learning Symposium, NIPS},
year={2017}
}

@inproceedings{sun2020attentive,
  title={Attentive experience replay},
  author={Sun, Peiquan and Zhou, Wengang and Li, Houqiang},
  booktitle={Proceedings of the AAAI Conference on Artificial Intelligence},
  volume={34},
  number={04},
  pages={5900--5907},
  year={2020}
}

@inproceedings{balles2018dissecting,
  title={Dissecting adam: The sign, magnitude and variance of stochastic gradients},
  author={Balles, Lukas and Hennig, Philipp},
  booktitle={International Conference on Machine Learning},
  pages={404--413},
  year={2018},
  organization={PMLR}
}

@book{lahiri2003resampling,
  title={Resampling methods for dependent data},
  author={Lahiri, Soumendra Nath},
  year={2003},
  publisher={Springer Science \& Business Media}
}

@article{kunsch1989jackknife,
  title={The jackknife and the bootstrap for general stationary observations},
  author={Kunsch, Hans R},
  journal={The Annals of Statistics},
  pages={1217--1241},
  year={1989},
  publisher={JSTOR}
}

@article{liu1992moving,
  title={Moving blocks jackknife and bootstrap capture weak dependence},
  author={Liu, Regina Y and Singh, Kesar and others},
  journal={Exploring the Limits of Bootstrap},
  volume={225},
  pages={248},
  year={1992}
}

@inproceedings{bhandari2018finite,
  title={A finite time analysis of temporal difference learning with linear function approximation},
  author={Bhandari, Jalaj and Russo, Daniel and Singal, Raghav},
  booktitle={Conference on learning theory},
  pages={1691--1692},
  year={2018},
  organization={PMLR}
}

@misc{benelot2018, author = {Benjamin Ellenberger}, title = {PyBullet Gymperium}, howpublished = {\url{ https://github.com/benelot/pybullet-gym}} , year = {2018--2019}}

@article{botvinick2019reinforcement,
  title={Reinforcement learning, fast and slow},
  author={Botvinick, Matthew and Ritter, Sam and Wang, Jane X and Kurth-Nelson, Zeb and Blundell, Charles and Hassabis, Demis},
  journal={Trends in Cognitive Sciences},
  volume={23},
  number={5},
  pages={408--422},
  year={2019},
  publisher={Elsevier}
}

@inproceedings{metelli2018policy,
 author = {Metelli, Alberto Maria and Papini, Matteo and Faccio, Francesco and Restelli, Marcello},
 booktitle = {Advances in Neural Information Processing Systems},
 editor = {S. Bengio and H. Wallach and H. Larochelle and K. Grauman and N. Cesa-Bianchi and R. Garnett},
 pages = {},
 publisher = {Curran Associates, Inc.},
 title = {Policy Optimization via Importance Sampling},
 volume = {31},
 year = {2018}
}

@article{stable-baselines3,
  author  = {Antonin Raffin and Ashley Hill and Adam Gleave and Anssi Kanervisto and Maximilian Ernestus and Noah Dormann},
  title   = {Stable-Baselines3: Reliable Reinforcement Learning Implementations},
  journal = {Journal of Machine Learning Research},
  year    = {2021},
  volume  = {22},
  number  = {268},
  pages   = {1-8},
}

@article{eckman2018reusing,
  title={Reusing search data in ranking and selection: What could possibly go wrong?},
  author={Eckman, David J and Henderson, Shane G},
  journal={ACM Transactions on Modeling and Computer Simulation (TOMACS)},
  volume={28},
  number={3},
  pages={1--15},
  year={2018},
  publisher={ACM New York, NY, USA}
}

@inproceedings{liu2020simulation,
  title={Simulation optimization by reusing past replications: don’t be afraid of dependence},
  author={Liu, Tianyi and Zhou, Enlu},
  booktitle={2020 Winter Simulation Conference (WSC)},
  pages={2923--2934},
  year={2020},
  organization={IEEE}
}

@inproceedings{lin2023reusing,
  title={Reusing Historical Observations in Natural Policy Gradient},
  author={Lin, Yifan and Zhou, Enlu},
  booktitle={2023 Winter Simulation Conference (WSC)},
  pages={3071--3081},
  year={2023},
  organization={IEEE}
}

@book{bradley2007introduction,
  title={Introduction to Strong Mixing Conditions},
  author={Bradley, R.C.},
  number={v. 3},
  isbn={9780974042787},
  series={Introduction to Strong Mixing Conditions},
  year={2007},
  publisher={Kendrick Press}
}

@article{bradley2007basic,
author = {Richard C. Bradley},
title = {{Basic Properties of Strong Mixing Conditions. A Survey and Some Open Questions}},
volume = {2},
journal = {Probability Surveys},
number = {none},
publisher = {Institute of Mathematical Statistics and Bernoulli Society},
pages = {107 -- 144},
year = {2005},
doi = {10.1214/154957805100000104},
URL = {https://doi.org/10.1214/154957805100000104}
}

@inproceedings{papini2019optimistic,
  title={Optimistic policy optimization via multiple importance sampling},
  author={Papini, Matteo and Metelli, Alberto Maria and Lupo, Lorenzo and Restelli, Marcello},
  booktitle={International Conference on Machine Learning},
  pages={4989--4999},
  year={2019},
  organization={PMLR}
}
\clearpage

\begin{appendices}

\section{Experiment Details}
\label{sec:experiment_details}

In this appendix, we report the hyperparameters,
and additional experimental results. In all comparison experiments, each pair of baseline and VRER-based algorithm were run under the same settings. To reproduce the results and check out the implementation details, please visit our repository: \url{https://github.com/zhenghuazx/vrer_policy_gradient}.

\subsection{Hyperparameters}
\label{subsec: hyperparameters}
\noindent \textbf{Policy Configuration.} The benchmark convergence curves in Figure~\ref{fig: convergence comparison} and the buffer-size and selection-constant sensitivity curves in Figures~\ref{fig: sensitivity buffer} and~\ref{fig:sensitivity-selection} are based on eight seeds and report across-seed means with 95\% confidence intervals. Table~\ref{table: performance} uses five seeds and reports mean $\pm$ standard deviation. The KL-based selection-approximation study in Figure~\ref{fig:kl_ranking_overlap_main} uses three seeds, whereas the bias--variance diagnostic in Figure~\ref{fig:bias_variance_main} uses five seeds and reports across-seed medians. Regarding architecture, for $m$-dimensional discrete actions, we use a softmax policy $\pi_{\pmb\theta}(a_i|\pmb{s}) = \exp(\phi_{\pmb\theta}(\pmb{s},a_i)) / \sum_{j=1}^m \exp(\phi_{\pmb\theta}(\pmb{s}, a_j))$, where $\phi_{\pmb\theta}$ is a two-layer MLP with ReLU (hidden) and linear (output) activations. For continuous actions, we adopt a Gaussian policy where the mean is parameterized by a two-layer MLP with tanh (hidden) and linear (output) activations, while the variance is a state-independent identity matrix.

The VRER-related hyperparameters include selection constant $c$, the buffer size $B$, and the number of sampled observations per iteration $n_0$. The environmental condition includes ``num-envs" which represents the number of parallel environments (each iteration, each environment collect $n$ steps of data) and the learning rate for policy optimization. Table~\ref{table:merged_hyperparameters} presents the hyperparameters for A2C, PPO, TRPO, and their VRER variants. 
Common parameters include ``entropy coef'' and ``value loss coef,'' which denote the coefficients for entropy and value loss calculations, respectively. 
For PPO, ``batch size'' refers to the number of transitions per iteration, ``mini batches'' to the count of mini-batches, and ``Lambda'' to the General Advantage Estimation (GAE) parameter. The ``PPO iterations'' ($K_{\text{off}}$) represent the maximum actor-critic optimization steps per training step. 
Regarding TRPO, ``cg iterations'' indicates the maximum gradient conjugation iterations, and ``clip norm'' denotes the surrogate clipping coefficient. Similarly, ``actor/critic iterations'' ($K_{\text{off}}$) refer to the maximum optimization steps for the actor and critic networks.

\begin{table}[p]\small
\centering
\caption{Hyperparameters for A2C, PPO, TRPO, and their VRER variants}
\label{table:merged_hyperparameters}
\begin{tabular}{@{}lcccc@{}}
\toprule
\textbf{Parameter} & \textbf{CartPole} & \textbf{Hopper} & \textbf{Inv. Pendulum} & \textbf{Lunar Lander} \\ 
\midrule

\multicolumn{5}{c}{\textit{\textbf{A2C}}} \\ 
\midrule
Learning rate & 0.0003 & 0.0003 & 0.0003 & 0.0003 \\
Num-envs & 24 & 24 & 64 & 48 \\
n-steps & 16 & 16 & 16 & 16 \\
Clip norm & 0.2 & 0.2 & 0.2 & 0.2 \\
Entropy coef & 0.01 & 0.01 & 0.01 & 0.01 \\
Value loss coef & 0.5 & 0.5 & 0.5 & 0.5 \\
Grad norm & 0.5 & 0.5 & 0.5 & 0.5 \\
Buffer size & 400 & 400 & 300 & 200 \\
$n_0$ & 3 & 3 & 3 & 3 \\ 

\midrule
\multicolumn{5}{c}{\textit{\textbf{PPO}}} \\ 
\midrule
Learning rate & 0.0003 & 0.0003 & 0.0003 & 0.0003 \\
Num-envs & 12 & 6 & 12 & 12 \\
n-steps & 128 & 128 & 128 & 128 \\
Clip norm & 0.2 & 0.2 & 0.2 & 0.2 \\
Entropy coef & 0.01 & 0.01 & 0.01 & 0.01 \\
Lambda & 0.95 & 0.95 & 0.95 & 0.95 \\
Mini batch size & 128 & 128 & 128 & 128 \\
PPO iterations & 4 & 4 & 4 & 4 \\
Buffer size & 400 & 400 & 400 & 200 \\
$n_0$ & 3 & 3 & 3 & 3 \\ 

\midrule
\multicolumn{5}{c}{\textit{\textbf{TRPO}}} \\ 
\midrule
Learning rate & 0.0003 & 0.0003 & 0.0003 & 0.0003 \\
Num-envs & 12 & 6 & 12 & 12 \\
n-steps & 128 & 128 & 128 & 128 \\
Batch size & 512 & 512 & 512 & 512 \\
Clip norm & 0.2 & 0.2 & 0.2 & 0.2 \\
Entropy coef & 0.01 & 0.01 & 0.01 & 0.01 \\
Mini batches & 128 & 128 & 128 & 128 \\
Actor iterations & 10 & 10 & 10 & 10 \\
Critic iterations & 3 & 3 & 3 & 3 \\
CG iterations & 10 & 10 & 10 & 10 \\
Buffer size & 400 & 400 & 400 & 200 \\
$n_0$ & 3 & 3 & 3 & 3 \\ 
\bottomrule
\end{tabular}
\end{table}
\clearpage



\section{Additional Experimental Results and Replay Patterns}\label{append sec: additional experimental result}
To further validate our approach, we conduct experiments on two discrete-control tasks and compare the convergence performance of A2C-VRER and ACER. The results are shown in Figure~\ref{fig: convergence comparison with acer}. Both methods are implemented on top of the A2C framework. For a fair comparison, we use the same training settings across all methods, including a gradient clipping norm of 0.02, entropy coefficient of 0.01, 16-step rollouts, a learning rate of 0.0003, and 24 and 48 parallel environments for CartPole and LunarLander, respectively. 
We then tune ACER-specific hyperparameters, including replay-buffer size, replay ratio, and the importance-weight truncation threshold. Additional implementation details for ACER are available in our open-source codebase and \citet{wang2017sample}.


\begin{figure}[!ht]
 \centering
 \vspace{-0.1in}
 \subfloat[CartPole]{
 \centering
 \includegraphics[width=0.48\textwidth]{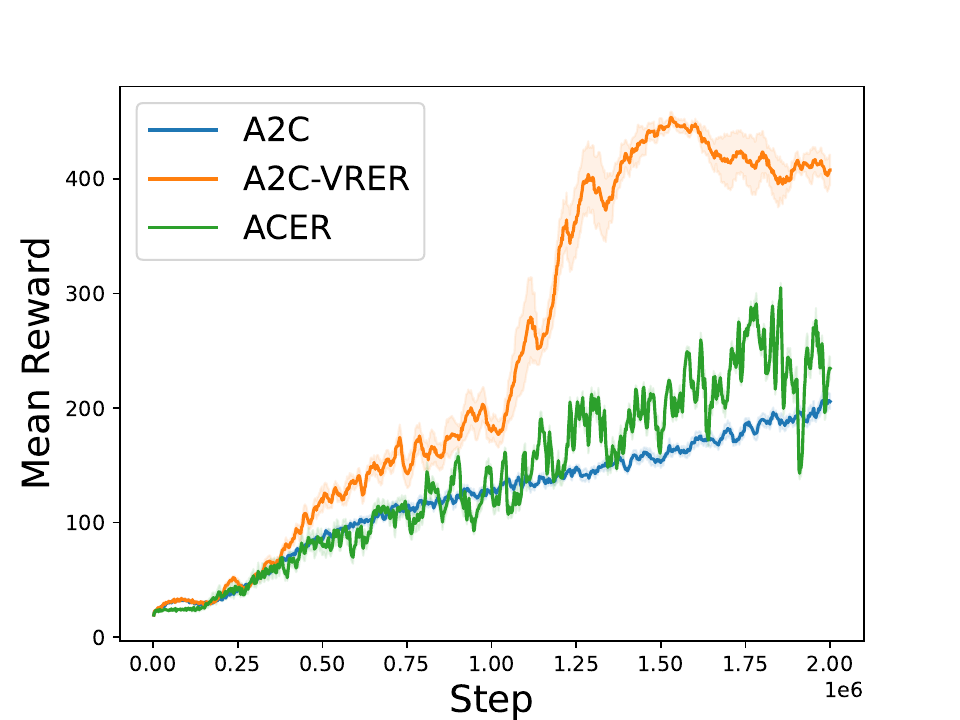}
 }
 \subfloat[Lunar Lander]{
 \centering
 \includegraphics[width=0.48\textwidth]{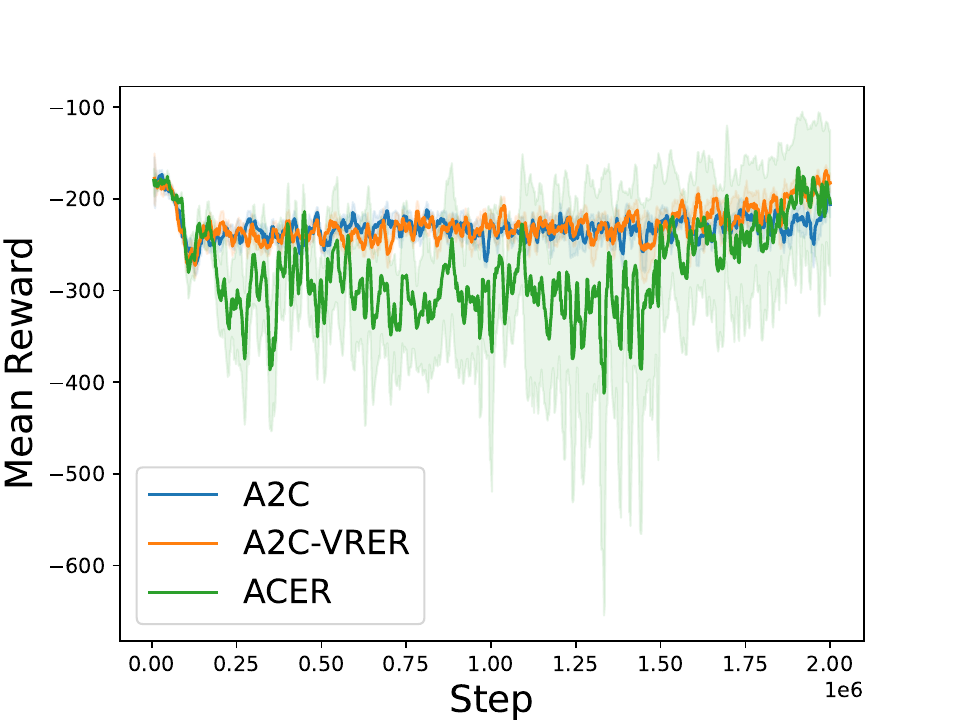}
 }
 \medskip
 \caption{Comparison between A2C-VRER and ACER approach in discrete control Tasks.
 }\label{fig: convergence comparison with acer}
 \vspace{-0.in}
\end{figure}




In addition, we study the impact of the value of selection constant $c$ on the reusing pattern of historical samples.
Table~\ref{table: reuse ratio} provides a comprehensive overview of the performance of VRER-based policy optimization algorithms in the context of the CartPole task under varying selection constants ($c$). We observe two key aspects: the reuse ratios ($|\mathcal{U}_k|/B$) and the average rewards obtained over one million steps. 
The reuse ratio, defined as the proportion of historical samples selected, rises consistently with the selection constant $c$ across the tested values, with PPO-VRER leveling off at $0.99$ for the largest values. At $c=1.001$, reuse is negligible for PPO-VRER, TRPO-VRER, and A2C-VRER.
At $c=1.6$, the ratios reach $0.99$, $0.30$, and $0.41$, respectively. Small values of $c$ impose strict KL divergence constraints per Eq.~\eqref{eq: simple selection}, resulting in little reuse. Conversely, large values of $c$ can degrade performance by admitting less relevant samples, as observed in PPO-VRER. Thus, $c$ must balance sample reuse against replay bias rather than be chosen solely to maximize reuse.

\begin{table}[tbp]\vspace{-0.1in}
\centering
\caption{The reuse ratios ($|\mathcal{U}_k|/B$) and average rewards for VRER-based policy optimization algorithms in the CartPole task under different selection constants. Values are represented as mean $\pm$ standard deviation.}\label{table: reuse ratio}
\footnotesize
\begin{tabular}{@{}c|ccc|ccc@{}}
\toprule
\multirow{2}{*}{\begin{tabular}[c]{@{}c@{}}Selection\\ Constant $c$ \end{tabular}} & \multicolumn{3}{c|}{Reuse Ratio $|\mathcal{U}_k|/B$}         & \multicolumn{3}{c}{\begin{tabular}[c]{@{}c@{}}Average Reward \\ of One Million Steps\end{tabular}} \\ \cmidrule(l){2-7} 
& PPO-VRER     & TRPO-VRER    & A2C-VRER     & PPO-VRER     & TRPO-VRER    & A2C-VRER                        \\ \midrule
1.001                                                                                & 0.00 ± 0.00 & 0.00 ± 0.00 & 0.01 ± 0.01 & 372.65  ± 17.98                        & 434.93  ± 14.57                        & 100.20  ± 3.51                       \\
1.02                                                                                & 0.34 ± 0.02 & 0.10 ± 0.00 & 0.13 ± 0.02 & 443.03  ± 5.52                        & 449.11  ± 5.66                        & 99.64     ± 5.01                     \\
1.04                                                                                & 0.86 ± 0.01 & 0.12 ± 0.01 & 0.22 ± 0.02 & 432.77 ± 11.13                        & 443.40   ± 4.92                       & 106.07   ± 6.82                      \\
1.06                                                                                & 0.94 ± 0.01 & 0.14 ± 0.02 & 0.24 ± 0.02 & 425.5  ±  8.44                       & 444.52  ± 6.13                        & 107.98     ± 6.66                    \\
1.08                                                                                & 0.95 ± 0.02 & 0.15 ± 0.02 & 0.28 ± 0.02 & 403.78 ± 20.90                        & 429.76   ± 17.28                       & 123.58   ± 9.34                      \\
1.10                                                                                & 0.96 ± 0.01 & 0.17 ± 0.03 & 0.29 ± 0.03 & 397.89  ± 14.99                        & 411.87  ±  37.18                       & 120.75   ± 8.88                      \\
1.20                                                                                & 0.99 ± 0.00 & 0.22 ± 0.01 & 0.38 ± 0.01 & 400.79  ± 11.05                       & 439.72   ± 10.60                       & 143.20 ± 10.05                        \\
1.40                                                                                & 0.99 ± 0.00 & 0.28 ± 0.05 & 0.39 ± 0.01 & 411.43   ± 6.18                       & 442.68     ± 7.64                     & 154.51  ± 11.59                       \\
1.60                                                                                & 0.99 ± 0.00 & 0.30 ± 0.02 & 0.41 ± 0.02 & 411.46   ± 6.12                   & 445.19    ± 3.75                      & 157.61   ± 12.38                     \\ \bottomrule
\end{tabular}
\end{table}
\clearpage

\section{Proof of Preliminary Lemmas}\label{appendix sec: proof of preliminary lemmas}

\subsection{Key Proof Technique}\label{appendix subsec: proof techniques}
The key proof techniques in Theorem~\ref{convergence theorem} rely on employing the uniform ergodicity Assumption~(\textbf{A.2}). This technique was
first introduced by \cite{bhandari2018finite} to address the Markovian noise in policy evaluation. 
\cite{zou2019finite} extended its usage to the Q-learning setting and \cite{wu2020finite} further considered the situation with the policy parameter changing. In this work, we take two steps further: (1) reusing historical samples collected under old behavior policies, thereby causing a complex dependence structure; and (2) using a simple but biased LR-based gradient estimate (i.e., ratio of target and behavior policies), which introduces additional theoretical difficulty.

Suppose that a historical sample $\pmb{x}^{(i,j)}=\left(\pmb{s}^{(i,j)},\pmb{a}^{(i,j)}\right)$ is selected and replayed at iteration $k$ with $i<k$. Due to Markovian noise and policy updates, all samples are correlated and their dependencies can be analyzed by using the Auxiliary Markov Chain (AMC) and Stationary Markov Chain (SMC) constructed below.
Throughout the convergence analysis, path probabilities are defined with respect to the restart kernel, whereas the value functions $Q^\pi$ and $V^\pi$ are defined under the original environment dynamics.
To elucidate the effects of Markovian noise, policy update, and historical sample reuse, AMC and SMC will be utilized along with uniform ergodicity and triangle inequality to bound the dependency effect induced by Markovian noise and experience reply in the proof of Theorem~\ref{convergence theorem}.

\vspace{0.05in}
\noindent\textbf{Original Markov Chain}: For reference, we first present the original Markov chain with policy update, i.e.,
\begin{gather}
    \pmb{s}^{(i-t,1)}\stackrel{\pmb\theta_{i-t}}{\longrightarrow} \pmb{a}^{(i-t,1)}
    \stackrel{\widetilde{\P}}{\longrightarrow} \cdots\stackrel{\widetilde{\P}}{\longrightarrow}
    {\pmb{s}^{(i-t, j)}\stackrel{\pmb\theta_{i-t}}{\longrightarrow} \pmb{a}^{(i-t, j)}\stackrel{\widetilde{\P}}{\longrightarrow}
    \cdots \stackrel{\widetilde{\P}}{\longrightarrow} \pmb{s}^{(i-t,n
    )} \stackrel{\pmb\theta_{i-t}}{\longrightarrow} {\pmb{a}}^{(i-t,n)}\stackrel{\widetilde{\P}}{\longrightarrow}} \nonumber\\
     {\cdots} \nonumber\\
     {\pmb{s}^{(i,1)}\stackrel{\pmb\theta_{i}}{\longrightarrow} \pmb{a}^{(i,1)}\stackrel{\widetilde{\P}}{\longrightarrow} 
     \cdots \stackrel{\widetilde{\P}}{\longrightarrow}
    \pmb{s}^{(i, j)}\stackrel{\pmb\theta_{i}}{\longrightarrow}} {\pmb{a}^{(i, j)}}\stackrel{\widetilde{\P}}{\longrightarrow}
    \cdots \stackrel{\widetilde{\P}}{\longrightarrow} \pmb{s}^{(i,n
    )}\stackrel{\pmb\theta_{i}}{\longrightarrow}{\pmb{a}}^{(i-t,n)}\stackrel{\widetilde{\P}}{\longrightarrow} \nonumber\\
    \cdots \nonumber\\
    \pmb{s}^{(k,1)}\stackrel{\pmb\theta_{k}}{\longrightarrow} \pmb{a}^{(k,1)}\stackrel{\widetilde{\P}}{\longrightarrow}  \cdots
    \pmb{s}^{(k, j+1)}\stackrel{\pmb\theta_{k}}{\longrightarrow} \pmb{a}^{(k, j+1)}\stackrel{\widetilde{\P}}{\longrightarrow}
    \cdots \stackrel{\widetilde{\P}}{\longrightarrow} \pmb{s}^{(k,n
    )} \stackrel{\pmb\theta_{i-t}}{\longrightarrow} \tilde{\pmb{a}}^{(k,n)}\stackrel{\widetilde{\P}}{\longrightarrow}\cdots \label{eq: folded original MC}
\end{gather}
where $\widetilde{\P}$ is the restart kernel defined in Section~3 and $t$ is an integer with $0\leq t<i$.

\vspace{0.05in}
\noindent\textbf{Auxiliary Markov Chain (AMC)}:  The analysis of Markovian noise relies on the
auxiliary Markov chain introduced by \cite{zou2019finite} which constructs a state-action sequence by following a fixed policy $\pi_{\pmb\theta_{i-t}}$ with $0\leq t < i$,
\begin{gather}
    \pmb{s}^{(i-t,1)}\stackrel{\pmb\theta_{i-t}}{\longrightarrow} {\pmb{a}}^{(i-t,1)}\stackrel{\widetilde{\P}}{\longrightarrow} 
     \cdots\stackrel{\widetilde{\P}}{\longrightarrow}
    \tilde{\pmb{s}}^{(i-t, j)}\stackrel{\pmb\theta_{i-t}}{\longrightarrow} \tilde{\pmb{a}}^{(i-t, j)}\stackrel{\widetilde{\P}}{\longrightarrow}
    \cdots \stackrel{\widetilde{\P}}{\longrightarrow} \tilde{\pmb{s}}^{(i-t,n
    )} \stackrel{\pmb\theta_{i-t}}{\longrightarrow} \tilde{\pmb{a}}^{i-t,n}\stackrel{\widetilde{\P}}{\longrightarrow} \nonumber\\
    \tilde{\pmb{s}}^{(i-t +1,1)}\stackrel{\pmb\theta_{i-t}}{\longrightarrow} \tilde{\pmb{a}}^{(i-t+1,1)}\stackrel{\widetilde{\P}}{\longrightarrow} 
    \cdots \stackrel{\widetilde{\P}}{\longrightarrow} \tilde{\pmb{s}}^{(i-t+1,n
    )} \stackrel{\pmb\theta_{i-t}}{\longrightarrow} \tilde{\pmb{a}}^{i-t,n}\stackrel{\widetilde{\P}}{\longrightarrow} \nonumber\\
     \cdots \nonumber\\
     \tilde{\pmb{s}}^{(i,1)}\stackrel{\pmb\theta_{i-t}}{\longrightarrow} \tilde{\pmb{a}}^{(i,1)}\stackrel{\widetilde{\P}}{\longrightarrow} 
     \cdots \stackrel{\widetilde{\P}}{\longrightarrow}
    \tilde{\pmb{s}}^{(i, j)}\stackrel{\pmb\theta_{i-t}}{\longrightarrow}\tilde{\pmb{a}}^{(i, j)}\stackrel{\widetilde{\P}}{\longrightarrow}
     \ldots  \label{eq: folded auxiliary MC}
\end{gather}
Since a fixed policy $\pi_{\pmb\theta_{i-t}}$ is used, the highlighted trajectory part associates with the difference comparing with original Markov chain with policy updates.



\vspace{0.05in}
\noindent\textbf{Stationary Markov Chain (SMC)}: The analysis of the interdependence of behavior policies relies on bounding the distance between their policy parameters. To achieve this, we construct a hypothetical sample path (highlighted trajectory part below)  wherein each sample in the path is independently drawn from the fixed stationary distribution $(\pmb{s},\pmb{a}) \sim d^{\pi_{\pmb\theta_{i-t}}}(\pmb{s},\pmb
a)$. It is worth noting that the chain we created is no longer a Markov chain but a sequence of independent samples.
\begin{gather}
{\pmb{s}}^{(i-t,1)}\stackrel{\pmb\theta_{i-t}}{\longrightarrow} {\pmb{a}}^{(i-t,1)}\stackrel{\widetilde{\P}}{\longrightarrow} 
     \cdots\stackrel{\widetilde{\P}}{\longrightarrow}
    \check{\pmb{s}}^{(i-t, j)}\stackrel{\pmb\theta_{i-t}}{\longrightarrow}\check{\pmb{a}}^{(i-t, j)}\stackrel{d^{\pi_{\pmb\theta_{i-t}}}}{\longrightarrow}
    \cdots\stackrel{d^{\pi_{\pmb\theta_{i-t}}}}{\longrightarrow} \check{\pmb{s}}^{(i-t, n)}\stackrel{\pmb\theta_{i-t}}{\longrightarrow} \check{\pmb{a}}^{(i-t, n)} \stackrel{d^{\pi_{\pmb\theta_{i-t}}}}{\longrightarrow} \nonumber\\
\check{\pmb{s}}^{(i-t+1,1)}\stackrel{\pmb\theta_{i-t+1}}{\longrightarrow} \check{\pmb{a}}^{(i-t+1,1)}\stackrel{d^{\pi_{\pmb\theta_{i-t}}}}{\longrightarrow}
    \cdots \stackrel{d^{\pi_{\pmb\theta_{i-t}}}}{\longrightarrow}\check{\pmb{s}}^{(i-t+1, n)}\stackrel{\pmb\theta_{i-t}}{\longrightarrow} \check{\pmb{a}}^{(i-t+1, n)}\stackrel{d^{\pi_{\pmb\theta_{i-t}}}}{\longrightarrow} \nonumber\\
     \cdots \nonumber\\
     \check{\pmb{s}}^{(i,1)}\stackrel{\pmb\theta_{i-t}}{\longrightarrow} \check{\pmb{a}}^{(i,1)}  \stackrel{d^{\pi_{\pmb\theta_{i-t}}}}{\longrightarrow}
     \cdots \stackrel{d^{\pi_{\pmb\theta_{i-t}}}}{\longrightarrow}
    \check{\pmb{s}}^{(i, j)}\stackrel{\pmb\theta_{i-t}}{\longrightarrow}\check{\pmb{a}}^{(i, j)}\stackrel{\widetilde{\P}}{\longrightarrow}
    \cdots  \label{eq: stationary transitions}
\end{gather}

We can concatenate the nested chains \eqref{eq: folded original MC} and \eqref{eq: folded auxiliary MC} by reindexing the trajectory $t\leftarrow nt$ and  $\tau\leftarrow n (i-1)+j$ such that \begin{align*}
    (\pmb{s}_{\tau-t},\pmb{a}_{\tau-t})\leftarrow\left(\pmb{s}^{(i-t,j)}, \pmb{a}^{(i-t,j)}\right) \text{  and  } (\pmb{s}_\tau,\pmb{a}_\tau)\leftarrow\left(\pmb{s}^{(i,j)}, \pmb{a}^{(i,j)}\right).
\end{align*} 
As a result, the subsequence from $\left(\pmb{s}^{(i-t,j)}, \pmb{a}^{(i-t,j)}\right)$ to $\left(\pmb{s}^{(i,j)}, \pmb{a}^{(i,j)}\right)$  in Markov chains \eqref{eq: folded auxiliary MC} and \eqref{eq: folded original MC} can be concatenated into a single chain shown below.

\vspace{0.05in}
\noindent\textbf{(Concatenated) Auxilary Markov Chain:}
    \begin{small}
    \begin{equation}\label{eq: auxiliary MC}
    \tilde{\pmb{s}}_{\tau-t}\stackrel{\pmb\theta_{\tau-t}}{\longrightarrow} \tilde{\pmb{a}}_{\tau-t}\stackrel{\widetilde{\P}}{\longrightarrow} \tilde{\pmb{s}}_{\tau-t+1}\stackrel{\pmb\theta_{\tau-t}}{\longrightarrow} \tilde{\pmb{a}}_{\tau-t+1}\stackrel{\widetilde{\P}}{\longrightarrow} \tilde{\pmb{s}}_{\tau-t+2}\stackrel{\pmb\theta_{\tau-t}}{\longrightarrow} \tilde{\pmb{a}}_{\tau-t+2}\stackrel{\widetilde{\P}}{\longrightarrow}\cdots \stackrel{\widetilde{\P}}{\longrightarrow} \tilde{\pmb{s}}_{\tau}\stackrel{\pmb\theta_{\tau-t}}{\longrightarrow} \tilde{\pmb{a}}_{\tau}\stackrel{\widetilde{\P}}{\longrightarrow}\tilde{\pmb{s}}_{\tau+1}.
\end{equation}
\end{small}

\noindent\textbf{(Concatenated) Original Markov Chain:}
\begin{small}
\begin{equation}\label{eq: original MC}
    \pmb{s}_{\tau-t}\stackrel{\pmb\theta_{\tau-t}}{\longrightarrow} \pmb{a}_{\tau-t}\stackrel{\widetilde{\P}}{\longrightarrow} \pmb{s}_{\tau-t+1}\stackrel{\pmb\theta_{\tau-t+1}}{\longrightarrow} \pmb{a}_{\tau-t+1}\stackrel{\widetilde{\P}}{\longrightarrow} \pmb{s}_{\tau-t+2}\stackrel{\pmb\theta_{\tau-t+2}}{\longrightarrow} \pmb{a}_{\tau-t+2}\stackrel{\widetilde{\P}}{\longrightarrow}\cdots \stackrel{\widetilde{\P}}{\longrightarrow} \pmb{s}_{\tau}\stackrel{\pmb\theta_{\tau}}{\longrightarrow} \pmb{a}_{\tau}\stackrel{\widetilde{\P}}{\longrightarrow}\pmb{s}_{\tau+1}.
\end{equation}
\end{small}

Therefore, the existing results on the concatenated AMC (single chain MDP) can be directly transferred to our setting. We list those results as supportive lemmas in Section~\ref{appendix subsec: lipschitz continuity} 

\subsection{Notations for Main Theorems} \label{appendix subsec: notations}

We use $\tilde{\pmb{x}}^{(i,j)}=(\tilde{\pmb{s}}^{(i,j)},\tilde{\pmb{a}}^{(i,j)})$ to denote a state-action pair sample generated from the Auxiliary Markov Chain (AMC) that constructs a state-action sequence by following a fixed policy $\pi_{\pmb\theta_{i-t}}$ with $0\leq t < i$.
Let $\check{\pmb{x}}^{(i,j)}=(\check{\pmb{s}}^{(i,j)},\check{\pmb{a}}^{(i,j)})$ denote a sample generated from the stationary distribution $d^{\pi_{\pmb\theta_{i-t}}}(\pmb{s},\pmb{a})$.
We define some notations here for clarification:
\begin{align*}
    \pmb{x}_t &= (\pmb{s}_t, \pmb{a}_t)  \\
    \widehat{\nabla} J^{LR}\left(\pmb{x}^{(i,j)}, \pmb\theta_i,\pmb\theta_k\right)&= \frac{\pi_{\pmb\theta_k}\left(\pmb{a}^{(i,j)}|\pmb{s}^{(i,j)}\right)}{\pi_{\pmb\theta_i}\left(\pmb{a}^{(i,j)}|\pmb{s}^{(i,j)}\right)} g\left(\pmb{x}^{(i,j)}|\pmb\theta_k\right) \\
     \widehat{\nabla} J^{LR}\left(\pmb{x}^{(i,j)}, \pmb\theta_i,\pmb\theta_{i-t}\right)&= \frac{\pi_{\pmb\theta_{i-t}}\left(\pmb{a}^{(i,j)}|\pmb{s}^{(i,j)}\right)}{\pi_{\pmb\theta_i}\left(\pmb{a}^{(i,j)}|\pmb{s}^{(i,j)}\right)} g\left(\pmb{x}^{(i,j)}|\pmb\theta_{i-t}\right) \\
    \widehat{\nabla} J^{CLR}\left(\pmb{x}^{(i,j)}, \pmb\theta_i,\pmb\theta_k\right)&= \min\left(\frac{\pi_{\pmb\theta_k}\left(\pmb{a}^{(i,j)}|\pmb{s}^{(i,j)}\right)}{\pi_{\pmb\theta_i}\left(\pmb{a}^{(i,j)}|\pmb{s}^{(i,j)}\right)}, U_f \right)g\left(\pmb{x}^{(i,j)}|\pmb\theta_k\right) \\
     \widehat{\nabla} J^{CLR}\left(\pmb{x}^{(i,j)}, \pmb\theta_i,\pmb\theta_{i-t}\right)&= \min \left(\frac{\pi_{\pmb\theta_{i-t}}\left(\pmb{a}^{(i,j)}|\pmb{s}^{(i,j)}\right)}{\pi_{\pmb\theta_i}\left(\pmb{a}^{(i,j)}|\pmb{s}^{(i,j)}\right)}, U_f\right) g\left(\pmb{x}^{(i,j)}|\pmb\theta_{i-t}\right) \\
    \widehat{\nabla} J\left(\tilde{\pmb{x}}^{(i,j)}, \pmb\theta_i,\pmb\theta_k\right)&= g\left(\tilde{\pmb{x}}^{(i,j)}|\pmb\theta_k\right)\\
    \widehat{\nabla} J\left(\check{\pmb{x}}^{(i,j)}, \pmb\theta_i,\pmb\theta_k\right)&= g\left(\check{\pmb{x}}^{(i,j)}|\pmb\theta_k\right)\\
    \widehat{\nabla} J\left(\check{\pmb{x}}^{(i,j)}, \pmb\theta_i,\pmb\theta_k\right)&= g\left(\check{\pmb{x}}^{(i,j)}|\pmb\theta_k\right)\\
    \Gamma\left(\pmb{x}^{(i,j)},\pmb\theta_i, \pmb\theta_{k}\right) &= \left\langle \nabla J(\pmb\theta_k),\widehat{\nabla} J^{LR}\left(\pmb{x}^{(i,j)}, \pmb\theta_i,\pmb\theta_k\right) -\nabla J(\pmb\theta_k) \right\rangle \\
\Gamma\left(\pmb{x}^{(i,j)},\pmb\theta_i, \pmb\theta_{i-t}\right) &= \left\langle \nabla J(\pmb\theta_{i-t}),\widehat{\nabla} J^{LR}\left(\pmb{x}^{(i,j)}, \pmb\theta_i,\pmb\theta_{i-t}\right) -\nabla J(\pmb\theta_{i-t}) \right\rangle \\
\Gamma\left(\tilde{\pmb{x}}^{(i,j)},\pmb\theta_i, \pmb\theta_{i-t}\right) &= \left\langle \nabla J(\pmb\theta_{i-t}),\widehat{\nabla} J\left(\tilde{\pmb{x}}^{(i,j)}, \pmb\theta_i,\pmb\theta_{i-t}\right) -\nabla J(\pmb\theta_{i-t}) \right\rangle \\
\Gamma\left(\check{\pmb{x}}^{(i,j)},\pmb\theta_i, \pmb\theta_{i-t}\right) &= \left\langle \nabla J(\pmb\theta_{i-t}),\widehat{\nabla} J\left(\check{\pmb{x}}^{(i,j)}, \pmb\theta_i,\pmb\theta_{i-t}\right) -\nabla J(\pmb\theta_{i-t}) \right\rangle \\
\end{align*}

\subsection{Preliminary Supportive Lemmas} \label{appendix subsec: lipschitz continuity}
Lemma~\ref{lemma: bounded of policy gradient} establishes the boundedness of policy gradient and its stochastic estimate.
 There exists a constant $M>0$ such that the $L_2$ norm of the scenario-based policy gradient estimate is bounded, i.e., $\Vert g(\pmb{s},\pmb{a}|\pmb\theta)\Vert \leq M$.
 
\noindent\textbf{Lemma}~\ref{lemma: bounded of policy gradient} (\textbf{Boundedness
 of Stochastic Policy Gradients})
 \textit{For any $\pmb\theta$, the norm of the policy gradient $\nabla  J(\pmb\theta)$ and its scenario-based estimate $g(\pmb{s},\pmb{a}|\pmb\theta)$ is bounded, i.e., 
 $$\Vert \nabla  J(\pmb\theta)\Vert \leq M 
 \text{ and } 
  \Vert g\left(\pmb{s},\pmb{a}|\pmb\theta\right)\Vert \leq M,
  $$
 where $M=\frac{2U_r U_\Theta}{1-\gamma}$.}
 \proof From Eq.~\eqref{eq.scenariobasedGradient}, we have the scenario-based gradient estimate
 \begin{equation}
 \Vert g\left(\pmb{s},\pmb{a}|\pmb\theta\right) \Vert\leq |Q^{\pi_{\pmb\theta}}(\pmb{s},\pmb{a}) -V^{\pi_{\pmb\theta}}(\pmb{s})| \Vert\nabla \log \pi_{\pmb{\theta}}(\pmb{a}|\pmb{s})\Vert \leq \frac{2U_r U_\Theta}{1-\gamma}.
 \label{eq: boundedness of policy gradient}
 \end{equation}
Step~\eqref{eq: boundedness of policy gradient} follows by applying Assumption~\ref{assumption 2}(ii) and the boundedness of value functions in Eq.~\eqref{eq: bound of action value function} and \eqref{eq: bound of state value function}. Let $M=\frac{2U_r U_\Theta}{1-\gamma}$. The boundedness of $\nabla  J(\pmb\theta)$ follows the fact that
 $$\Vert\nabla J(\pmb{\theta})\Vert=\Vert\E_{(\pmb{s},\pmb{a})\sim d^{\pi_{\pmb\theta}}(\cdot,\cdot)}[g(\pmb{s},\pmb{a}|\pmb\theta)]\Vert\stackrel{(\star)}{\leq} \E_{(\pmb{s},\pmb{a})\sim d^{\pi_{\pmb\theta}}(\cdot,\cdot)}[\Vert g(\pmb{s},\pmb{a}|\pmb\theta)\Vert]\stackrel{(\star\star)}{\leq}\E[M]=M,$$
 where step~($\star$) follows by applying Jensen's inequality and $(\star\star)$ follows by applying \eqref{eq: boundedness of policy gradient}.
 \endproof
 
The proof of the following lemma is similar to the proof of Lemma B.2 in \cite{wu2020finite} with the difference that \cite{wu2020finite} considers the case with finite action space.

\begin{lemma}[Lipschitz Continuity of Stationary Distribution]\label{lemma: relations of total variation measures}
Given time indexes $t$ and $\tau$ such that $\tau\geq t\geq 0$, we consider the concatenated original Markov chain defined in \eqref{eq: original MC} and auxiliary Markov chain \eqref{eq: auxiliary MC}. Conditioning on $\pmb{s}_{\tau-t+1}$ and $\pmb\theta_{\tau-t}$, we have
\begin{align}
    \Vert & p(\pmb{s}_{\tau+1}|\pmb{s}_{\tau-t+1},\pmb\theta_{\tau-t}) -p(\tilde{\pmb{s}}_{\tau+1}|\pmb{s}_{\tau-t+1},\pmb\theta_{\tau-t})\Vert_{TV} \nonumber\\
    &\leq  \Vert p(\pmb{s}_{\tau},\pmb{a}_\tau|\pmb{s}_{\tau-t+1},\pmb\theta_{\tau-t}) -p(\tilde{\pmb{s}}_\tau,\tilde{\pmb{a}}_\tau|\pmb{s}_{\tau-t+1},\pmb\theta_{\tau-t})\Vert_{TV}\label{eq:lemma total_variation 1}\\
    \Vert & p(\pmb{s}_{\tau},\pmb{a}_\tau|\pmb{s}_{\tau-t+1},\pmb\theta_{\tau-t}) -p(\tilde{\pmb{s}}_{\tau},\tilde{\pmb{a}}_\tau|\pmb{s}_{\tau-t+1},\pmb\theta_{\tau-t})\Vert_{TV} \nonumber\\
    & \leq     \Vert p(\pmb{s}_\tau|\pmb{s}_{\tau-t+1},\pmb\theta_{\tau-t}) -p(\tilde{\pmb{s}}_\tau|\pmb{s}_{\tau-t+1},\pmb\theta_{\tau-t})\Vert_{TV}+  U_\pi\E[\Vert \pmb\theta_\tau-\pmb\theta_{\tau-t}\Vert]\label{eq:lemma total_variation 3}
\end{align}

\end{lemma}
\proof
The proof for \eqref{eq:lemma total_variation 1} and \eqref{eq:lemma total_variation 3} is identical to that of Lemma~B.2 in \cite{wu2020finite}. The only difference is that the proof in \cite{wu2020finite} handles a discrete action space and uses $\sum_{a\in\mathcal{A}}$, but our lemma considers both discrete and continuous action space, employing the integral notation $\int_{\mathcal{A}}$ to traverse this space.

\endproof

Lemmas~\ref{lemma: lipchitz continuity of stationary distribution} and \ref{lemma: lipchitz continuity of action value} establish the Lipschitz continuity of the stationary distribution and action value function. Similar supportive lemmas appear in many policy search algorithms \citep{xu2020improving,zou2019finite,wu2020finite}. Readers are referred to \citet[Lemma 3 and 4]{xu2020improving} for detailed proofs of these two lemmas.

\begin{lemma}[\cite{xu2020improving}, Lemma 3]\label{lemma: lipchitz continuity of stationary distribution} 
Let $d^{\pi_{\pmb\theta}}(s,a)$ denote the discounted state--action occupancy measure, which is the stationary distribution of the restart kernel associated with policy $\pi_{\pmb\theta}$.
Then, for any $\pmb{\theta}_1,\pmb{\theta}_2\in\mathbb{R}^d$,
\begin{equation*}
    \Vert d^{\pi_{\pmb\theta_1}}(\cdot,\cdot)-d^{\pi_{\pmb\theta_2}}(\cdot,\cdot)\Vert_{TV}\leq C_d \Vert\pmb\theta_1-\pmb\theta_2\Vert
\end{equation*}
where $C_d=U_\pi(1+\lceil\log_{\kappa}\kappa_0^{-1}\rceil+\frac{1}{1-\kappa})$.
\end{lemma}
\proof

Let $\widetilde{\P}_{\pi_{\pmb\theta}}$ denote the restart kernel. Let $d_j=d^{\pi_{\pmb\theta_j}}$ be its stationary distribution for $j\in\{1,2\}$ and their balance equations give
$d_1-d_2=\gamma(d_1-d_2)\P_{\pi_{\pmb\theta_1}}+\gamma d_2(\P_{\pi_{\pmb\theta_1}}-\P_{\pi_{\pmb\theta_2}})$.
Total-variation contraction and Assumption~\ref{assumption 2}(ii) bound the distance of the discounted occupancy over states,
\[ \Vert d^{\pi_{\pmb\theta_1}}-d^{\pi_{\pmb\theta_2}}\Vert_{TV} \leq \frac{\gamma U_\pi}{1-\gamma}\Vert\pmb\theta_1-\pmb\theta_2\Vert.
\]
Then, by adding the action distributions $d^{\pi_{\pmb\theta}}(\pmb{s},\pmb{a})=\pi_{\pmb\theta}(\pmb{a}|\pmb{s})d^{\pi_{\pmb\theta}}(\pmb{s})$ and applying the triangle inequality, we have
\[ 
\Vert d^{\pi_{\pmb\theta_1}}(\cdot,\cdot)-d^{\pi_{\pmb\theta_2}}(\cdot,\cdot)\Vert_{TV}\leq
\frac{U_\pi}{1-\gamma}\Vert\pmb\theta_1-\pmb\theta_2\Vert
\leq C_d\Vert\pmb\theta_1-\pmb\theta_2\Vert
\]for $\kappa_0=1,\kappa=\gamma$.
\endproof

\begin{lemma}[\cite{xu2020improving}, Lemma 4]\label{lemma: lipchitz continuity of action value} 
Suppose Assumptions \ref{assumption 2} and \ref{assumption 3} hold. For any $\pmb\theta_1,\pmb\theta_2\in \mathbb{R}^d$ and any state-action pair $(\pmb{s},\pmb{a})\in \mathcal{S}\times\mathcal{A}$, we have
$$|Q^{\pi_{\pmb\theta_1}}(\pmb{s},\pmb{a})-Q^{\pi_{\pmb\theta_2}}(\pmb{s},\pmb{a})|\leq L_Q \Vert \pmb\theta_1-\pmb\theta_2\Vert$$
where $L_Q=\frac{2 U_r C_d}{1-\gamma}$
\end{lemma}
\proof
The original-kernel Bellman equations and Assumption~\ref{assumption 2}(ii) give
\[ \Vert Q^{\pi_{\pmb\theta_1}}-Q^{\pi_{\pmb\theta_2}}\Vert_\infty\leq
\frac{2\gamma U_rU_\pi}{(1-\gamma)^2}\Vert\pmb\theta_1-\pmb\theta_2\Vert.
\]
The stated $L_Q$ is larger because $C_d\geq U_\pi/(1-\gamma)$. 
\endproof

\begin{lemma}[Lipchitz Continuity of Advantage Function]\label{lemma: lipchitz continuity of advantage}
Suppose Assumptions \ref{assumption 2} and \ref{assumption 3} hold. For any $\pmb\theta_1,\pmb\theta_2\in \mathbb{R}^d$ and any state-action pair $(\pmb{s},\pmb{a})\in \mathcal{S}\times\mathcal{A}$, we have
$$|A^{\pi_{\pmb\theta_1}}(\pmb{s},\pmb{a})-A^{\pi_{\pmb\theta_2}}(\pmb{s},\pmb{a})|\leq L_A \Vert \pmb\theta_1-\pmb\theta_2\Vert$$
where $L_A=2L_Q+U_ J U_\pi$, $L_Q=\frac{2 U_r C_d}{1-\gamma}$, and $U_ J = U_r/(1-\gamma)$.
\end{lemma}
\proof
To begin with, we show the Lipchitz continuity of state value function,
\begin{align}
    |V^{\pi_{\pmb\theta_1}}(\pmb{s}) - V^{\pi_{\pmb\theta_2}}(\pmb{s})|
    &=\left|\int_{\mathcal{A}}Q^{\pi_{\pmb\theta_1}}(\pmb{s},\pmb{a})\pi_{\pmb\theta_1}(\pmb{a}|\pmb{s})\dd\pmb{a} - \int_{\mathcal{A}} Q^{\pi_{\pmb\theta_2}}(\pmb{s},\pmb{a}) \pi_{\pmb\theta_2}(\pmb{a}|\pmb{s})\dd\pmb{a}\right| \nonumber\\
    &\leq\int_{\mathcal{A}} |Q^{\pi_{\pmb\theta_1}}(\pmb{s},\pmb{a})-Q^{\pi_{\pmb\theta_2}}(\pmb{s},\pmb{a})|\pi_{\pmb\theta_1}(\pmb{a}|\pmb{s})\dd\pmb{a} \nonumber\\
    & \quad + \int_{\mathcal{A}} \left|Q^{\pi_{\pmb\theta_2}}(\pmb{s},\pmb{a})\right|\left|\pi_{\pmb\theta_1}(\pmb{a}|\pmb{s}) -\pi_{\pmb\theta_2}(\pmb{a}|\pmb{s}) \right|\dd\pmb{a}. \nonumber
\end{align}
By applying Lemma~\ref{lemma: lipchitz continuity of action value}, it holds that $$\int_{\mathcal{A}} |Q^{\pi_{\pmb\theta_1}}(\pmb{s},\pmb{a})-Q^{\pi_{\pmb\theta_2}}(\pmb{s},\pmb{a})|\pi_{\pmb\theta_1}(\pmb{a}|\pmb{s})\dd\pmb{a}\leq L_Q\Vert\pmb\theta_1-\pmb\theta_2\Vert \int_{\mathcal{A}} \pi_{\pmb\theta_1}(\pmb{a}|\pmb{s})\dd\pmb{a}=L_Q\Vert\pmb\theta_1-\pmb\theta_2\Vert.$$ 
By applying Eq.~\eqref{eq: bound of action value function} and Assumption~\ref{assumption 2}, it holds that
$$\int_{\mathcal{A}} \left|Q^{\pi_{\pmb\theta_2}}(\pmb{s},\pmb{a})\right|\left|\pi_{\pmb\theta_1}(\pmb{a}|\pmb{s}) -\pi_{\pmb\theta_2}(\pmb{a}|\pmb{s}) \right|\dd\pmb{a}\leq U_ J U_\pi \Vert\pmb\theta_1-\pmb\theta_2\Vert.$$
Therefore, we have $|V^{\pi_{\pmb\theta_1}}(\pmb{s}) - V^{\pi_{\pmb\theta_2}}(\pmb{s})|\leq (L_Q +U_ J U_\pi)\Vert\pmb\theta_1-\pmb\theta_2\Vert$. Now we prove the Lipchitz continuity of advantage function. By triangle inequality, it holds that
\begin{align}
    |A^{\pi_{\pmb\theta_1}}(\pmb{s},\pmb{a}) - A^{\pi_{\pmb\theta_2}}(\pmb{s},\pmb{a})| &= |Q^{\pi_{\pmb\theta_1}}(\pmb{s},\pmb{a})-V^{\pi_{\pmb\theta_1}}(\pmb{s}) - Q^{\pi_{\pmb\theta_2}}(\pmb{s},\pmb{a}) + V^{\pi_{\pmb\theta_2}}(\pmb{s})| \nonumber\\
    &\leq |Q^{\pi_{\pmb\theta_1}}(\pmb{s},\pmb{a}) - Q^{\pi_{\pmb\theta_2}}(\pmb{s},\pmb{a})| + |V^{\pi_{\pmb\theta_1}}(\pmb{s}) - V^{\pi_{\pmb\theta_2}}(\pmb{s})| \nonumber\\
    &\leq (2L_Q+U_ J U_\pi) \Vert \pmb\theta_1-\pmb\theta_2\Vert,
\end{align}
which completes the proof.
\endproof

\begin{lemma}[Lipchitz Continuity of Sample Gradient]\label{lemma: lipchitz continuity of sample gradient estimate}
Suppose Assumptions \ref{assumption 2} and \ref{assumption 3} hold. For any $\pmb\theta_1,\pmb\theta_2\in \mathbb{R}^d$ and any state-action pair $(\pmb{s},\pmb{a})\in \mathcal{S}\times\mathcal{A}$, we have
$$\Vert g(\pmb{s},\pmb{a}|\pmb\theta_1)-g(\pmb{s},\pmb{a}|\pmb\theta_2)\Vert \leq L_g \Vert \pmb\theta_1-\pmb\theta_2\Vert$$
where $L_g=U_\Theta L_A+2U_ J L_\Theta$.
\end{lemma}
\proof By applying Eq.~(\ref{eq.scenariobasedGradient}) and Minkowski's inequality, we have
\begin{align}
    \Vert g(\pmb{s},\pmb{a}|\pmb\theta_1)-g(\pmb{s},\pmb{a}|\pmb\theta_2)\Vert &= \Vert A^{\pi_{\pmb\theta_1}}(\pmb{s},\pmb{a}) \nabla\log\pi_{\pmb\theta_1}(\pmb{a}|\pmb{s}) -A^{\pi_{\pmb\theta_2}}(\pmb{s},\pmb{a}) \nabla\log\pi_{\pmb\theta_2}(\pmb{a}|\pmb{s}) \Vert \nonumber\\
 &= \Vert A^{\pi_{\pmb\theta_1}}(\pmb{s},\pmb{a}) \nabla\log\pi_{\pmb\theta_1}(\pmb{a}|\pmb{s}) - A^{\pi_{\pmb\theta_2}}(\pmb{s},\pmb{a}) \nabla\log\pi_{\pmb\theta_1}(\pmb{a}|\pmb{s}) \nonumber\\
 & \quad + A^{\pi_{\pmb\theta_2}}(\pmb{s},\pmb{a}) \nabla\log\pi_{\pmb\theta_1}(\pmb{a}|\pmb{s}) -A^{\pi_{\pmb\theta_2}}(\pmb{s},\pmb{a}) \nabla\log\pi_{\pmb\theta_2}(\pmb{a}|\pmb{s}) \Vert \nonumber\\
    &\leq |A^{\pi_{\pmb\theta_1}}(\pmb{s},\pmb{a}) -A^{\pi_{\pmb\theta_2}}(\pmb{s},\pmb{a})|\Vert \nabla\log\pi_{\pmb\theta_1}(\pmb{a}|\pmb{s})\Vert \nonumber\\
    &\quad + |A^{\pi_{\pmb\theta_2}}(\pmb{s},\pmb{a})| \Vert \nabla\log\pi_{\pmb\theta_1}(\pmb{a}|\pmb{s})-\nabla\log\pi_{\pmb\theta_2}(\pmb{a}|\pmb{s}) \Vert. \nonumber
\end{align}
Notice that $|A^{\pi_{\pmb\theta_2}}(\pmb{s},\pmb{a})|\leq 2U_J$ holds because both state- and action-value functions are bounded by $U_J$, as shown in Eq.~\eqref{eq: bound of action value function} and \eqref{eq: bound of state value function}. By applying Lemma~\ref{lemma: lipchitz continuity of advantage} and Assumption~\ref{assumption 2}, we have
\begin{equation*}
\Vert g(\pmb{s},\pmb{a}|\pmb\theta_1)-g(\pmb{s},\pmb{a}|\pmb\theta_2)\Vert \leq U_\Theta L_A\Vert\pmb\theta_1-\pmb\theta_2\Vert + 2U_ J L_\Theta \Vert\pmb\theta_1-\pmb\theta_2\Vert = L_g \Vert\pmb\theta_1-\pmb\theta_2\Vert
\end{equation*}
where $L_g=U_\Theta L_A+2U_ J L_\Theta$.
\endproof

\begin{lemma} \label{lemma: variance bound} \textit{At the $k$-th iteration with the target distribution $\pi_{\pmb\theta_k}$, given the replay buffer of policies $\mathcal{U}_k$ storing the behavioral policies $\pi_{\pmb\theta_i}$, 
we have, for $R\in$\{LR,CLR\},
\begin{align*}
&\E\left[\left\Vert\widehat{\nabla} J^{R}_{k}\right\Vert^2\right]\leq \Tr\left(\Var
 \left[\widehat{\nabla} J^{R}_{k} \right]\right) +2\left \Vert \E\left[\nabla  J(\pmb\theta_k)\right]\right\Vert^2+2\left\Vert\E\left[\widehat{\nabla} J^{R}_{k}\right]-\E[\nabla J(\pmb\theta_k)]\right\Vert^2.
\end{align*}
}
\end{lemma}

\proof The conclusion follows the following steps
\begin{eqnarray}
\lefteqn{
\Tr\left(\Var\left[\widehat{\nabla} J^{R}_{k}\right]\right) }
\nonumber \\
&=& \Tr\left(\E\left[\left(\widehat{\nabla} J^{R}_{k}\right)\left(\widehat{\nabla} J^{R}_{k}\right)^\top\right] - \E\left[\widehat{\nabla} J^{R}_{k}\right]\E\left[\widehat{\nabla} J^{R}_{k}\right]^\top\right)\nonumber\\
&=&\E\left[\left\Vert\widehat{\nabla} J^{R}_{k}\right\Vert^2\right] - \left\Vert\E\left[\widehat{\nabla} J^{R}_{k}\right]\right\Vert^2\nonumber\\
&\geq &\E\left[\left\Vert\widehat{\nabla} J^{R}_{k}\right\Vert^2\right] - 2 \Vert\E[\nabla  J(\pmb\theta_k)]\Vert^2-2\left\Vert\E\left[\widehat{\nabla} J^{R}_{k}\right]-\E\left[\nabla J(\pmb\theta_k)\right]\right\Vert^2
\label{eq.midstep100}
\end{eqnarray}
\endproof

\section{Proofs of Main Theorems}
\label{appendix sec: proof of main theorems}

\subsection{Proof of Theorem~\ref{theorem: bounded bias for LR graident}}
\label{appendix sec: bias of policy gradient}
\textbf{Theorem~\ref{theorem: bounded bias for LR graident}.}\textit{Suppose that Assumptions \ref{assumption 2} and \ref{assumption 3} hold. For both LR and CLR policy gradient estimators in Eq.~\eqref{eq.LR-gradient} and \eqref{eq.CLR-gradient}, we consider the policy update rule (\ref{eq: policy gradient update}) with the learning rate $\eta_k$. For any integer $t$ such that $0\leq t < i$, the bias of policy gradient estimators (R$\in$\{LR,CLR\}) can be bounded by 
\begin{equation*}
     \left\Vert\E\left[\widehat{\nabla} J^{R}_k\right] - \E[\nabla  J(\pmb\theta_k)]\right\Vert \leq \frac{Z_1^{R} \cdot t}{|\mathcal{U}_k|}
 \sum_{{\pmb\theta_i}\in \mathcal{U}_k}\sum^i_{\ell = i-t}\eta_\ell + 2M\varphi(nt) + \frac{Z_2^{R}}{|\mathcal{U}_k|}
 \sum_{{\pmb\theta_i}\in \mathcal{U}_k} \sum^k_{\ell=i-t}\eta_\ell, 
\end{equation*}
where $Z_1^{LR}=2nM^2 U_\pi$, $Z_1^{CLR}= 2nM^2 U_\pi U_f$, $Z_2^{LR}=M(2MC_d+2L_g+MU_\pi)$ and $Z_2^{CLR}=M(2MC_d+ (U_f+1)L_g+MU_\pi)$. Here, the Lipchitz constants $C_d$ and $L_g$ are defined in Lemma~\ref{lemma: lipchitz continuity of stationary distribution} and Lemma~\ref{lemma: lipchitz continuity of sample gradient estimate} respectively in Appendix~\ref{appendix subsec: lipschitz continuity}.}

\proof
\noindent\textbf{(i)} We start with the proof for the LR policy gradient estimator. Notice that
\begin{align}
    \left\Vert\E\left[\widehat{\nabla} J^{LR}_k\right] -\E[\nabla  J(\pmb\theta_k)]\right\Vert &= \left\Vert \frac{1}{|\mathcal{U}_k|n}
 \sum_{{\pmb\theta_i}\in \mathcal{U}_k}
 \sum^{n}_{j=1}\E\left[
\frac{\pi_{\pmb\theta_k}\left(\pmb{a}^{(i,j)}|\pmb{s}^{(i,j)}\right)}
 {\pi_{\pmb\theta_i}\left(\pmb{a}^{(i,j)}|\pmb{s}^{(i,j)}\right)} g\left(\pmb{x}^{(i,j)}|\pmb\theta_k\right)-\nabla J(\pmb\theta_k)\right]\right\Vert \nonumber\\
 &\leq \frac{1}{|\mathcal{U}_k|n}
 \sum_{{\pmb\theta_i}\in \mathcal{U}_k}
 \sum^{n}_{j=1}\left\Vert \E\left[
\frac{\pi_{\pmb\theta_k}\left(\pmb{a}^{(i,j)}|\pmb{s}^{(i,j)}\right)}
{\pi_{\pmb\theta_i}\left(\pmb{a}^{(i,j)}|\pmb{s}^{(i,j)}\right)} g\left(\pmb{x}^{(i,j)}|\pmb\theta_k\right)-\nabla J(\pmb\theta_k)\right]\right\Vert \nonumber\\
 & = \frac{1}{|\mathcal{U}_k|n}
 \sum_{{\pmb\theta_i}\in \mathcal{U}_k}
 \sum^{n}_{j=1}\left\Vert \Delta_{\textbf{(i)}}(\pmb{x}^{(i,j)},\pmb\theta_i,\pmb\theta_k)\right\Vert,
 \label{eq1: theorem bounded bias for LR graident}
\end{align}
where $\Delta_{\textbf{(i)}}(\pmb{x}^{(i,j)},\pmb\theta_i,\pmb\theta_k)=\E\left[
\frac{\pi_{\pmb\theta_k}\left(\pmb{a}^{(i,j)}|\pmb{s}^{(i,j)}\right)}
 {\pi_{\pmb\theta_i}\left(\pmb{a}^{(i,j)}|\pmb{s}^{(i,j)}\right)} g\left(\pmb{x}^{(i,j)}|\pmb\theta_k\right)-\nabla J(\pmb\theta_k)\right]$ can be decomposed by
\begin{align}
\left\Vert\Delta_{\textbf{(i)}}(\pmb{x}^{(i,j)},\pmb\theta_i,\pmb\theta_k)\right\Vert&\leq\left\Vert\E\left[\widehat{\nabla} J^{LR}\left({\pmb{x}}^{(i,j)}, \pmb\theta_i,\pmb\theta_{k}\right)-\widehat{\nabla} J^{LR}\left({\pmb{x}}^{(i,j)}, \pmb\theta_i,\pmb\theta_{i-t}\right) \right] \right\Vert \nonumber\\
&\quad +\left\Vert\E\left[\widehat{\nabla} J^{LR}\left({\pmb{x}}^{(i,j)}, \pmb\theta_i,\pmb\theta_{i-t}\right)-\widehat{\nabla} J\left(\tilde{\pmb{x}}^{(i,j)}, \pmb\theta_i,\pmb\theta_{i-t}\right)\right] \right\Vert \nonumber\\
&\quad + \left\Vert \E\left[\widehat{\nabla} J\left(\tilde{\pmb{x}}^{(i,j)}, \pmb\theta_i,\pmb\theta_{i-t}\right)-\widehat{\nabla} J\left(\check{\pmb{x}}^{(i,j)}, \pmb\theta_i,\pmb\theta_{i-t}\right)\right] \right\Vert \nonumber\\
&\quad + \left\Vert \E\left[\widehat{\nabla} J\left(\check{\pmb{x}}^{(i,j)}, \pmb\theta_i,\pmb\theta_{i-t}\right)-\widehat{\nabla} J\left(\check{\pmb{x}}^{(i,j)}, \pmb\theta_i,\pmb\theta_{k}\right)\right]\right\Vert \nonumber\\
&\quad + \left\Vert\E\left[\widehat{\nabla} J\left(\check{\pmb{x}}^{(i,j)}, \pmb\theta_i,\pmb\theta_{k}\right)-\nabla  J\left(\pmb\theta_k\right)\right] \right\Vert. \label{eq: bounded bias 5}
\end{align}
By applying Lemmas~\ref{appendix lemma: supportive lemma 1}, \ref{appendix lemma: supportive lemma 2}, \ref{appendix lemma: supportive lemma 3}, \ref{appendix lemma: supportive lemma 4} and \ref{appendix lemma: supportive lemma 5} to the terms in \eqref{eq: bounded bias 5}, we have 
\begin{equation}\label{eq2: theorem bounded bias for LR graident}
\left\Vert\Delta_{\textbf{(i)}}(\pmb{x}^{(i,j)},\pmb\theta_i,\pmb\theta_k)\right\Vert \leq 
2M\varphi(nt)+2nM^2 U_\pi t\sum^i_{\ell = i-t}\eta_\ell  + M(2MC_d+2L_g+M U_\pi)\sum^k_{\ell=i-t}\eta_\ell. 
\end{equation}
The conclusion follows by plugging the upper bound \eqref{eq2: theorem bounded bias for LR graident} to \eqref{eq1: theorem bounded bias for LR graident}.

\vspace{0.1in}
\noindent\textbf{(ii)} The proof for the CLR policy gradient estimator is similar to that for the LR estimator.
 \begin{align}
    &\left\Vert\E\left[\widehat{\nabla} J^{CLR}_k\right] -\E[\nabla  J(\pmb\theta_k)]\right\Vert \nonumber\\
    &= \left\Vert \frac{1}{|\mathcal{U}_k|n}
 \sum_{{\pmb\theta_i}\in \mathcal{U}_k}
 \sum^{n}_{j=1}\E\left[
\min\left(\frac{\pi_{\pmb\theta_k}\left(\pmb{a}^{(i,j)}|\pmb{s}^{(i,j)}\right)}
{\pi_{\pmb\theta_i}\left(\pmb{a}^{(i,j)}|\pmb{s}^{(i,j)}\right)},U_f\right) g\left(\pmb{x}^{(i,j)}|\pmb\theta_k\right)-\nabla J(\pmb\theta_k)\right]\right\Vert \nonumber\\
 &\leq \frac{1}{|\mathcal{U}_k|n}
 \sum_{{\pmb\theta_i}\in \mathcal{U}_k}
 \sum^{n}_{j=1}\left\Vert \Delta_{\textbf{(ii)}}(\pmb{x}^{(i,j)},\pmb\theta_i,\pmb\theta_k)\right\Vert \label{eq3: theorem bounded bias for LR graident}
\end{align}
where $\Delta_{\textbf{(ii)}}(\pmb{x}^{(i,j)},\pmb\theta_i,\pmb\theta_k)=\E\left[
\min\left(\frac{\pi_{\pmb\theta_k}\left(\pmb{a}^{(i,j)}|\pmb{s}^{(i,j)}\right)}
{\pi_{\pmb\theta_i}\left(\pmb{a}^{(i,j)}|\pmb{s}^{(i,j)}\right)},U_f\right) g\left(\pmb{x}^{(i,j)}|\pmb\theta_k\right)-\nabla J(\pmb\theta_k)\right]$. Then we decompose $\Delta_{\textbf{(ii)}}(\pmb{x}^{(i,j)},\pmb\theta_i,\pmb\theta_k)$ by
\begin{align*}
\left\Vert\Delta_{\textbf{(ii)}}(\pmb{x}^{(i,j)},\pmb\theta_i,\pmb\theta_k)\right\Vert&\leq\left\Vert\E\left[\widehat{\nabla} J^{CLR}\left({\pmb{x}}^{(i,j)}, \pmb\theta_i,\pmb\theta_{k}\right)-\widehat{\nabla} J^{CLR}\left({\pmb{x}}^{(i,j)}, \pmb\theta_i,\pmb\theta_{i-t}\right) \right] \right\Vert\\
&\quad +\left\Vert\E\left[\widehat{\nabla} J^{CLR}\left({\pmb{x}}^{(i,j)}, \pmb\theta_i,\pmb\theta_{i-t}\right)-\widehat{\nabla} J\left(\tilde{\pmb{x}}^{(i,j)}, \pmb\theta_i,\pmb\theta_{i-t}\right)\right] \right\Vert \\
&\quad + \left\Vert \E\left[\widehat{\nabla} J\left(\tilde{\pmb{x}}^{(i,j)}, \pmb\theta_i,\pmb\theta_{i-t}\right)-\widehat{\nabla} J\left(\check{\pmb{x}}^{(i,j)}, \pmb\theta_i,\pmb\theta_{i-t}\right)\right] \right\Vert\\
&\quad + \left\Vert \E\left[\widehat{\nabla} J\left(\check{\pmb{x}}^{(i,j)}, \pmb\theta_i,\pmb\theta_{i-t}\right)-\widehat{\nabla} J\left(\check{\pmb{x}}^{(i,j)}, \pmb\theta_i,\pmb\theta_{k}\right)\right]\right\Vert \\
&\quad + \left\Vert\E\left[\widehat{\nabla} J\left(\check{\pmb{x}}^{(i,j)}, \pmb\theta_i,\pmb\theta_{k}\right)-\nabla  J\left(\pmb\theta_k\right)\right] \right\Vert
\end{align*}

By applying Lemmas~\ref{appendix lemma: supportive lemma 1}, \ref{appendix lemma: supportive lemma 2}, \ref{appendix lemma: supportive lemma 3}, \ref{appendix lemma: supportive lemma 4} and \ref{appendix lemma: supportive lemma 5}, we have 
\begin{align}\label{eq4: theorem bounded bias for LR graident}
&\left\Vert\Delta_{\textbf{(ii)}}(\pmb{x}^{(i,j)},\pmb\theta_i,\pmb\theta_k)\right\Vert 
\nonumber\\
    &
\leq 2nM^2U_\pi U_f t\sum^i_{\ell=i-t}\eta_\ell + 2M\varphi(nt) + M(2MC_d+(U_f+1) L_g +M U_\pi)\sum^k_{\ell=i-t}\eta_\ell. 
\end{align}
The conclusion follows by obtained by plugging the upper bound \eqref{eq4: theorem bounded bias for LR graident} to \eqref{eq3: theorem bounded bias for LR graident}.
\endproof

\subsection{Proof of Proposition~\ref{prop: gradient variance decomposition}}\label{appendix sec: prop 1}
\textbf{Proposition~\ref{prop: gradient variance decomposition}} (Gradient Variance)\textbf{.}
\textit{Suppose that Assumptions \ref{assumption 2} and \ref{assumption 3} hold. At the $k$-th iteration with the target policy $\pi_{\pmb\theta_k}$, given the replay buffer of policies $\mathcal{U}_k$, the total variance of LR/CLR policy gradient estimators ($R\in$\{LR,CLR\}) in Eq.~\eqref{eq.LR-gradient} and \eqref{eq.CLR-gradient} can be decomposed to
\begin{equation*}
     \Tr\left(\Var
 \left[\widehat{\nabla} J^{R}_{k} \right]\right)=\frac{1}{|\mathcal{U}_k|^2}\sum_{\pmb\theta_i\in\mathcal{U}_k}\sum_{\pmb\theta_{i^\prime}\in\mathcal{U}_k} \left(\pmb\sigma^{R}_{i,k}\right)^\top \mathbf{P}_{i,i^\prime,k} \left(\pmb\sigma^{R}_{i^\prime,k}\right)
\end{equation*}
where $\pmb\sigma^{R}_{i,k}=\left(\sqrt{\Var\left[\widehat{\nabla} J^{R,(1)}_{i,k} \right]}, \ldots,\sqrt{\Var\left[\widehat{\nabla} J^{R,(d)}_{i,k} \right]}\right)^\top$ and  $\mathbf{P}_{i,i^\prime,k}=\Diag\left(\Corr^{(1)}_{i,i^\prime,k},\ldots,\Corr^{(d)}_{i,i^\prime,k}\right)$ with correlation coefficients of each gradient estimate pair $\Corr_{i,i^\prime,k}^{(\ell)}=\Corr\left(\widehat{\nabla} J^{R,(\ell)}_{i,k},\widehat{\nabla} J^{R,(\ell)}_{i^\prime,k}\right)$ 
for each $\ell$-th dimension of policy parameters with $\ell=1,2,\ldots,d$. 
By using the greatest element of $\mathbf{P}$, the total variance~\eqref{eq: variance of LR estimaor} is bounded by
\begin{equation*}
     \Tr\left(\Var
 \left[\widehat{\nabla} J^{R}_{k} \right]\right)\leq \frac{1}{|\mathcal{U}_k|^2}\sum_{\pmb\theta_i\in\mathcal{U}_k}\sum_{\pmb\theta_{i^\prime}\in\mathcal{U}_k} \max_{\ell=1,\ldots,d}|\Corr^{(\ell)}_{i,i^\prime,k}|\left(\pmb\sigma^{R}_{i,k}\right)^\top\left(\pmb\sigma^{R}_{i^\prime,k}\right).
\end{equation*}}
\proof
Due to the Markovian noise and policy update, the observations $\pmb{x}^{(i,j)}$ for any $\pmb\theta_i\in\mathcal{U}_k$ and $j=1,2,\ldots,n$ are dependent. Thus, we have
 \begin{eqnarray}
 \Tr\left(\Var\left[
 \widehat{\nabla} J^{R}_{k} \right]\right)&=&\frac{1}{|\mathcal{U}_k|^2}\Tr\left(\Var\left[ \sum_{\pmb\theta_i\in\mathcal{U}_k}\widehat{\nabla} J^{R}_{i,k} \right]\right)\nonumber\\
 &=&\frac{1}{|\mathcal{U}_k|^2}\sum_{\pmb\theta_i\in\mathcal{U}_k}\sum_{\pmb\theta_{i^\prime}\in\mathcal{U}_k} \Tr\left(\Cov\left[ \widehat{\nabla} J^{R}_{i,k},\widehat{\nabla} J^{R}_{i^\prime,k} \right]\right). 
 \label{eq: LR equality 1}
 \end{eqnarray}

Due to the sample dependence, policy gradient estimates are dependent. Let $\Corr_{i,i^\prime,k}^{(\ell)}=\Corr\left(\widehat{\nabla} J^{R,(\ell)}_{i,k},\widehat{\nabla} J^{R,(\ell)}_{i^\prime,k}\right)$ denote the correlation coefficient of the $\ell$-th element between two individual LR/CLR policy gradient estimates. Then the total variance \eqref{eq: LR equality 1} can be rewritten as 
\begin{small}
\begin{align}
   \Tr\left(\Var\left[
 \widehat{\nabla} J^{R}_{k} \right]\right)&= \frac{1}{|\mathcal{U}_k|^2}\sum_{\pmb\theta_i\in\mathcal{U}_k}\sum_{\pmb\theta_{i^\prime}\in\mathcal{U}_k} \sum^d_{\ell=1} \Corr^{(\ell)}_{i,i^\prime,k}\sqrt{\Var\left[ \widehat{\nabla} J^{R,(\ell)}_{i,k}\right]} \sqrt{\Var\left[\widehat{\nabla} J^{R,(\ell)}_{i^\prime,k} \right]} \nonumber\\
 &= \frac{1}{|\mathcal{U}_k|^2}\sum_{\pmb\theta_i\in\mathcal{U}_k}\sum_{\pmb\theta_{i^\prime}\in\mathcal{U}_k} \left(\pmb\sigma^{R}_{i,k}\right)^\top \mathbf{P}_{i,i^\prime,k} \left(\pmb\sigma^{R}_{i^\prime,k}\right)
\end{align}
\end{small}
where $\pmb\sigma^{R}_{i,k}=\left(\sqrt{\Var\left[\widehat{\nabla} J^{R,(1)}_{i,k} \right]}, \sqrt{\Var\left[\widehat{\nabla} J^{R,(2)}_{i,k} \right]},\ldots,\sqrt{\Var\left[\widehat{\nabla} J^{R,(d)}_{i,k} \right]}\right)^\top$ and \\ $\mathbf{P}_{i,i^\prime,k}=\Diag\left(\Corr^{(1)}_{i,i^\prime,k},\ldots,\Corr^{(d)}_{i,i^\prime,k}\right)=\begin{pmatrix}
    \Corr^{(1)}_{i,i^\prime,k} & 0&\ldots& 0\\
    0 &  \Corr^{(2)}_{i,i^\prime,k} & \ldots & 0 \\
    \vdots & \vdots & \ddots &  \vdots \\
    0 & 0 & \ldots &   \Corr^{(d)}_{i,i^\prime,k} 
\end{pmatrix}$.
\endproof

\subsection{Proof of Theorem~\ref{thm:Online selection rule}}
\label{appendix sec: proofs for selection rule theorem}

\textbf{Theorem~\ref{thm:Online selection rule}.} \textit{At the $k$-th iteration with the target distribution $\pi_{\pmb\theta_k}$, the reuse set $\mathcal{U}_k$ is created to include the behavioral distributions, i.e., $\pi_{\pmb\theta_i}$ with $\pmb{\theta}_i \in \mathcal{F}_k$, 
whose total variance of individual LR/CLR policy gradient estimators in \eqref{eq: individual likelihood ratio estimator} and 
\eqref{eq: individual clipped likelihood ratio estimator}
is no greater than $c$ times the total variance of the classical PG estimator for some constant $c>1$. Mathematically, for $R\in$\{LR,CLR\},
\begin{equation*} 
\text{\textbf{Selection Rule 1:}} ~~~ 
 \Tr\left(\Var
 \left[\widehat{\nabla} J^{R}_{i,k} \right]\right)\leq c\Tr\left(\Var
 \left[ \widehat{\nabla} J^{PG}_k \right]\right).
\end{equation*}
Then, based on such reuse set $\mathcal{U}_k$, the total variances of the average LR/CLR policy gradient estimators~\eqref{eq.LR-gradient} and \eqref{eq.CLR-gradient} are no greater than the total variance of the PG estimator scaled by the averaged maximum absolute correlation between individual LR/CLR policy gradient estimates,
\begin{align*} 
 \Tr\left(\Var\left[\widehat{\nabla} J^{R}_{k} \right]\right)
 \leq \frac{c}{|\mathcal{U}_k|^2}\Tr\left(\Var\left[
 \widehat{\nabla} J^{PG}_{k} \right]\right) \sum_{\pmb\theta_i\in\mathcal{U}_k}\sum_{\pmb\theta_{i^\prime}\in\mathcal{U}_k}\max_{\ell=1,\ldots,d}|\Corr^{(\ell)}_{i,i^\prime,k}|.
\end{align*} 
Moreover, it holds that
\begin{align*}
&\E\left[\left\Vert\widehat{\nabla} J^{R}_{k}\right\Vert^2\right]\leq \frac{c}{|\mathcal{U}_k|^2}\sum_{\pmb\theta_i\in\mathcal{U}_k}\sum_{\pmb\theta_{i^\prime}\in\mathcal{U}_k}\max_{\ell=1,\ldots,d}|\Corr^{(\ell)}_{i,i^\prime,k}|\E\left[\left\Vert\widehat{\nabla} J^{PG}_{k}\right\Vert^2\right] \nonumber\\
&\qquad\qquad\qquad\qquad +2\left \Vert \E\left[\nabla  J(\pmb\theta_k)\right]\right\Vert^2+2\left\Vert\E\left[\widehat{\nabla} J^{R}_{k}\right]-\E[\nabla J(\pmb\theta_k)]\right\Vert^2.
\end{align*}
}

\proof
Following Proposition~\ref{prop: gradient variance decomposition}, we have

 \begin{align}
\Tr\left(\Var
 \left[\widehat{\nabla} J^{R}_{k} \right]\right)&=\frac{1}{|\mathcal{U}_k|^2}\sum_{\pmb\theta_i\in\mathcal{U}_k}\sum_{\pmb\theta_{i^\prime}\in\mathcal{U}_k} \left(\pmb\sigma^{R}_{i,k}\right)^\top \mathbf{P}_{i,i^\prime,k} \left(\pmb\sigma^{R}_{i^\prime,k}\right)\nonumber\\
 &\leq \frac{1}{|\mathcal{U}_k|^2}\sum_{\pmb\theta_i\in\mathcal{U}_k}\sum_{\pmb\theta_{i^\prime}\in\mathcal{U}_k}\max_{\ell=1,\ldots,d}|\Corr^{(\ell)}_{i,i^\prime,k}|\left(\pmb\sigma^{R}_{i,k}\right)^\top\left(\pmb\sigma^{R}_{i^\prime,k}\right) \nonumber\\
&\leq \frac{c}{|\mathcal{U}_k|^2}\sum_{\pmb\theta_i\in\mathcal{U}_k}\sum_{\pmb\theta_{i^\prime}\in\mathcal{U}_k}\max_{\ell=1,\ldots,d}|\Corr^{(\ell)}_{i,i^\prime,k}|\Tr\left(\Var\left[
 \widehat{\nabla} J^{PG}_{k} \right]\right), 
 \label{eq: variance inequality}
 \end{align}
where \eqref{eq: variance inequality} holds because of selection rule \eqref{eq: selection rule online}. 
Thus, (\ref{eq: variance reduction (online)}) is proved.
Then, we have 
\begin{eqnarray}
\Tr\left(\Var\left[\widehat{\nabla} J^{R}_{k}\right]\right) 
&=& \Tr\left(\E\left[\left(\widehat{\nabla} J^{R}_{k}\right)\left(\widehat{\nabla} J^{R}_{k}\right)^\top\right] - \E\left[\widehat{\nabla} J^{R}_{k}\right]\E\left[\widehat{\nabla} J^{R}_{k}\right]^\top\right)\nonumber\\
&=&\E\left[\left\Vert\widehat{\nabla} J^{R}_{k}\right\Vert^2\right] - \left\Vert\E\left[\widehat{\nabla} J^{R}_{k}\right]\right\Vert^2\nonumber\\
&\geq &\E\left[\left\Vert\widehat{\nabla} J^{R}_{k}\right\Vert^2\right] - 2 \Vert\E[\nabla  J(\pmb\theta_k)]\Vert^2-2\left\Vert\E\left[\widehat{\nabla} J^{R}_{k}\right]-\E\left[\nabla J(\pmb\theta_k)\right]\right\Vert^2.
\label{eq.midstep100}
\end{eqnarray}
Similarly, we have
\begin{equation}
\Tr\left(\Var\left[\widehat{\nabla} J^{PG}_{k}
\right]\right)=\E\left[\left\Vert\widehat{\nabla} J^{PG}_{k}\right\Vert^2\right] - \left\Vert\E\left[\widehat{\nabla} J^{PG}_{k}\right]\right\Vert^2\leq \E\left[\left\Vert\widehat{\nabla} J^{PG}_{k}\right\Vert^2\right].
\label{eq.midstep101}
\end{equation}

Then by applying inequality \eqref{eq: variance inequality}--\eqref{eq.midstep101} and rearranging the both sides, we can show
 \begin{align} 
\E\left[\left\Vert\widehat{\nabla} J^{R}_{k}\right\Vert^2\right] &\leq   \frac{c}{|\mathcal{U}_k|^2}\sum_{\pmb\theta_i\in\mathcal{U}_k}\sum_{\pmb\theta_{i^\prime}\in\mathcal{U}_k}\max_{\ell=1,\ldots,d}|\Corr^{(\ell)}_{i,i^\prime,k}|\E\left[\left\Vert\widehat{\nabla} J^{PG}_{k}\right\Vert^2\right] \nonumber\\
&\quad +2\left\Vert \E\left[\nabla  J(\pmb\theta_k)\right]\right\Vert^2+ 2\left\Vert\E\left[\widehat{\nabla} J^{R}_{k}\right]-\E\left[\nabla J(\pmb\theta_k)\right]\right\Vert^2 \nonumber
 \end{align}
which completes the proof.

\endproof




\subsection{Proof of Theorem~\ref{convergence theorem}} \label{appendix: sec convergence}

\begin{lemma}\label{lemma: main lemma for convergence analysis}
Suppose Assumptions \ref{assumption 2} and \ref{assumption 3} hold. Let $\eta_k=\eta_1 k^{-r}$ denote the learning rate used in the $k$-th iteration with two constants $\eta_1 \in(0, \frac{1}{4L}]$ and $r\in (0,1)$, where $L$ is defined in Lemma~\ref{lemma: Lipschitz continuity}. By running Algorithm~\ref{algo: online}, for both LR and CLR policy gradient estimators, we have 
\begin{align*}
\E\left[\Vert\nabla  J(\pmb\theta_k)\Vert^2\right] &\leq \frac{2}{\eta_k}\left(\E\left[ J(\pmb\theta_{k+1})\right]-\E\left[ J(\pmb\theta_k)\right]\right) + 2(I_1 + I_2 + I_3)\nonumber\\
I_1 & =\frac{1}{|\mathcal{U}_k|}\sum_{{\pmb\theta_i}\in \mathcal{U}_k}
 \left(C_1(k-i+t)\eta_{i-t}+C_2 (t+1)t\eta_{i-t} +2M^2\varphi(nt)\right)\nonumber\\
I_2 & =\frac{cLM^2\eta_k}{|\mathcal{U}_k|^2}\sum_{\pmb\theta_i\in \mathcal{U}_k}\sum_{\pmb\theta_{i^\prime}\in\mathcal{U}_k}\max_{\ell=1,\ldots,d}|\Corr^{(\ell)}_{i,i^\prime,k}| \nonumber\\
  I_3 &=2L\eta_k\Vert \Bias_k\Vert^2
\end{align*}
where $C_1=C^\Gamma_1+C^\Gamma_2$, $C_2={2nM^3U_\pi  U_f}$,  and $\Bias_k=\E\left[\widehat{\nabla} J^{R}_{k}\right]-\E\left[\nabla J(\pmb\theta_k)\right]$ with $R\in\{LR,CLR\},$ $C_1^\Gamma = {LM^2\sqrt{2c+1} (\sqrt{2c+1}+2)}$ and $C_2^\Gamma = M^2(U_f L_g+M U_\pi)$,
and $c$ is the selection constant defined in Theorem~\ref{thm:Online selection rule}.
\end{lemma}
\proof
Lemma~\ref{lemma: Lipschitz continuity} implies the L-Lipschitz property of policy gradient, which by the definition of smoothness (\cite{nesterov2003introductory}, Lemma 1.2.3) is also equivalent to
 $$ J(\pmb\theta_k) -  J(\pmb\theta_{k+1})\leq \langle \nabla  J(\pmb\theta_k),\pmb\theta_k - \pmb\theta_{k+1}\rangle + L\Vert \pmb\theta_{k+1}-\pmb\theta_k\Vert^2.$$
 
Let $\widehat{\nabla} J^{R}_{k}$ represents either LR policy gradient estimator $\widehat{\nabla} J^{LR}_{k}$ or CLR policy gradient estimator $\widehat{\nabla} J^{CLR}_{k}$, (i.e. $R\in\{LR,CLR\}$). By applying the policy parameter update implemented in the proposed algorithm $\pmb{\theta}_{k+1} = \pmb{\theta}_k+ \eta_k\widehat{\nabla} J^{R}_{k}$,  we have
\begin{align} 
& J(\pmb \theta_{k}) -  J(\pmb{\theta}_{k+1}) \leq -\left\langle \nabla  J(\pmb\theta_k),\eta_k \widehat{\nabla} J^{R}_k(\pmb\theta_k)\right\rangle + L\Vert \pmb\theta_{k+1}-\pmb\theta_k\Vert^2 \nonumber\\
&= \eta_k \left\langle \nabla  J(\pmb\theta_k),\nabla  J(\pmb\theta_k)-\widehat{\nabla} J^{R}_k(\pmb\theta_k)\right\rangle -\eta_k\Vert\nabla  J(\pmb\theta_k)\Vert^2+ L\Vert \pmb\theta_{k+1}-\pmb\theta_k\Vert^2. \label{eq: theorem equation 0}
\end{align}
Then by taking the expectation of both sides of \eqref{eq: theorem equation 0}, we have
\begin{align}
 &\E\left[ J(\pmb\theta_k)\right]-\E\left[ J(\pmb\theta_{k+1})\right] \nonumber\\ &\leq \eta_k\left|\E\left[\left\langle \nabla J(\pmb\theta_k), \widehat{\nabla} J^{R}_k(\pmb\theta_k)-\nabla J(\pmb\theta_k)\right\rangle \right] \right| -{\eta_k}\E\left[\Vert\nabla  J(\pmb\theta_k)\Vert^2\right]  + L\E[\Vert\pmb\theta_{k+1}-\pmb\theta_k\Vert^2] \nonumber\\
 &\leq \eta_k\left|\E\left[\left\langle \nabla J(\pmb\theta_k), \widehat{\nabla} J^{R}_k(\pmb\theta_k)-\nabla J(\pmb\theta_k)\right\rangle \right] \right| -{\eta_k}\E\left[\Vert\nabla  J(\pmb\theta_k)\Vert^2\right] + L\eta_k^2\E\left[\left\Vert \widehat{\nabla} J^{R}_{k} \right\Vert^2\right].
\end{align}
By applying Theorem~\ref{thm:Online selection rule} and rearranging both sides, we have
\begin{align}
(1-2L\eta_k)\E\left[\left\Vert\nabla  J(\pmb\theta_k)\right\Vert^2\right]  &\leq  \frac{1}{\eta_k}\left(\E\left[ J(\pmb\theta_{k+1})\right]-\E\left[ J(\pmb\theta_k)\right]\right) \nonumber \\
&\quad +\underbrace{\left|\E\left[\left\langle \nabla J(\pmb\theta_k), \widehat{\nabla} J^{R}_{k}-\nabla J(\pmb\theta_k)\right\rangle \right] \right|}_{I_1} \nonumber \\
 &\quad+\underbrace{\frac{cL\eta_k}{|\mathcal{U}_k|^2}\sum_{\pmb\theta_i\in\mathcal{U}_k}\sum_{\pmb\theta_{i^\prime}\in\mathcal{U}_k}\max_{\ell=1,\ldots,d}|\Corr^{(\ell)}_{i,i^\prime,k}|\E\left[\left\Vert\widehat{\nabla} J^{PG}_{k}\right\Vert^2\right]}_{I_2} \nonumber\\
&\quad +\underbrace{2L\eta_k\left\Vert\E\left[\widehat{\nabla} J^{R}_{k}\right]-\nabla J(\pmb\theta_k)\right\Vert^2}_{I_3}. \label{eq: expected squared norm bound}
\end{align}

\noindent For term $I_1$, we have
$$I_1 =\left|\frac{1}{|\mathcal{U}_k|n}
 \sum_{{\pmb\theta_i}\in \mathcal{U}_k}
 \sum^{n}_{j=1}\E\left[\left\langle \nabla J(\pmb\theta_k), 
\frac{\pi_{\pmb\theta_k}\left(\pmb{a}^{(i,j)}|\pmb{a}^{(i,j)}\right)}
 {\pi_{\pmb\theta_i}\left(\pmb{a}^{(i,j)}|\pmb{a}^{(i,j)}\right)} g\left(\pmb{x}^{(i,j)}|\pmb\theta_k\right)-\nabla J(\pmb\theta_k)\right\rangle\right]\right|.$$
 By applying Jensen's inequality, we have
\begin{align}
I_1&\leq \frac{1}{|\mathcal{U}_k|}
 \sum_{{\pmb\theta_i}\in \mathcal{U}_k}
 \left |\frac{1}{n}\sum^{n}_{j=1}\E\left[\left\langle \nabla J(\pmb\theta_k), \widehat{\nabla} J^{R}(\pmb{x}^{(i,j)}, \pmb\theta_i,\pmb\theta_k)-\nabla J(\pmb\theta_k)\right\rangle \right] \right|  \nonumber\\
&\leq \frac{1}{|\mathcal{U}_k|}
 \sum_{{\pmb\theta_i}\in \mathcal{U}_k}
\left| \frac{1}{n}\sum^{n}_{j=1}\E\left[\Gamma\left(\pmb{x}^{(i,j)},\pmb\theta_i, \pmb\theta_k\right)-\Gamma\left(\pmb{x}^{(i,j)},\pmb\theta_i,\pmb\theta_{i-t}\right) \right]\right| \nonumber\\
&\quad +\frac{1}{|\mathcal{U}_k|n}
 \sum_{{\pmb\theta_i}\in \mathcal{U}_k}
\sum^{n}_{j=1}\left|\E\left[\Gamma\left(\pmb{x}^{(i,j)},\pmb\theta_i, \pmb\theta_{i-t}\right)-\Gamma\left(\tilde{\pmb{x}}^{(i,j)},\pmb\theta_i,\pmb\theta_{i-t}\right)\right]\right|\nonumber\\
& \quad +\frac{1}{|\mathcal{U}_k|n}
 \sum_{{\pmb\theta_i}\in \mathcal{U}_k}
 \sum^{n}_{j=1}\left| \E\left[\Gamma\left(\tilde{\pmb{x}}^{(i,j)},\pmb\theta_i, \pmb\theta_{i-t}\right)-\Gamma\left(\check{\pmb{x}}^{(i,j)},\pmb\theta_i,\pmb\theta_{i-t}\right)\right]\right|\nonumber\\
 & \quad +\frac{1}{|\mathcal{U}_k|n}
 \sum_{{\pmb\theta_i}\in \mathcal{U}_k}
\sum^{n}_{j=1}\left|\E\left[\Gamma\left(\check{\pmb{x}}^{(i,j)},\pmb\theta_i, \pmb\theta_{i-t}\right)\right]\right|. \nonumber
\end{align}
By applying Lemma~\ref{appendix lemma: supportive lemma 6}, Lemma~\ref{appendix lemma: supportive lemma 7}, Lemma~\ref{appendix lemma: supportive lemma 8}, and Lemma~\ref{appendix lemma: supportive lemma 9}, it follows that
\begin{align}
&\left|\frac{1}{n}\sum^{n}_{j=1}\E\left[\left\langle \nabla J(\pmb\theta_k), \widehat{\nabla} J^{R}(\pmb{x}^{(i,j)}, \pmb\theta_i,\pmb\theta_k)-\nabla J(\pmb\theta_k)\right\rangle \right] \right| \nonumber \\
&\leq C^\Gamma_1 (k-i+t)\eta_{i-t} + C_2^\Gamma (k-i+t)\eta_{i-t}+2nM^3U_\pi U_f (t+1)t \eta_{i-t} +2M^2\varphi(nt)
\nonumber\\
&\leq   C_1 (k-i+t)\eta_{i-t}+{2nM^3U_\pi U_f} (t+1)t \eta_{i-t} +2M^2\varphi(nt), \label{eq1: conv theorem equation 1} 
\end{align}
where $C_1=C^\Gamma_1+C^\Gamma_2$. The second term of Step~\eqref{eq1: conv theorem equation 1} follows from noticing that $\sum^i_{\ell=i-t}\eta_\ell \leq t \eta_{i-t}$ . Let $C_2={2nM^3U_\pi U_f}$. Then term $I_1$ becomes
\begin{align}
I_1 &= \frac{1}{|\mathcal{U}_k|}
 \sum_{{\pmb\theta_i}\in \mathcal{U}_k} \left|\frac{1}{n}\sum^{n}_{j=1}\E\left[\left\langle \nabla J(\pmb\theta_k), \widehat{\nabla} J^{R}(\pmb{x}^{(i,j)}, \pmb\theta_i,\pmb\theta_k)-\nabla J(\pmb\theta_k)\right\rangle \right] \right| \nonumber\\
& \leq \frac{1}{|\mathcal{U}_k|}\sum_{{\pmb\theta_i}\in \mathcal{U}_k}
 \left(C_1(k-i+t)\eta_{i-t}+C_2 (t+1)t\eta_{i-t} +2M^2\varphi(nt)\right). \nonumber
\end{align}

\noindent For term $I_2$, by noticing $\E\left[\left\Vert\widehat{\nabla} J^{PG}_{k}\right\Vert^2\right]\leq\frac{1}{n}\sum^n_{j=1}\E\left[\left\Vert g(\pmb{s}^{(k,j)},\pmb{a}^{(k,j)}|\pmb\theta_k)\right\Vert^2\right]\leq M^2$, we have
\begin{align}
  I_2&\leq \frac{cLM^2\eta_k}{|\mathcal{U}_k|^2}\sum_{\pmb\theta_i\in\mathcal{U}_k}\sum_{\pmb\theta_{i^\prime}\in\mathcal{U}_k}\max_{\ell=1,\ldots,d}|\Corr^{(\ell)}_{i,i^\prime,k}|.
\end{align}

\noindent For term $I_3$, by applying Theorem~\ref{theorem: bounded bias for LR graident}, we have
\begin{equation}\label{eq-1: I3 in lemma 9}
I_3= 2L\eta_k\left\Vert\E\left[\widehat{\nabla} J^{R}_{k}\right]-\E\left[\nabla J(\pmb\theta_k)\right]\right\Vert^2 =2L\eta_k \Vert\Bias_k \Vert^2,
\end{equation}
where the bias is a convergent sequence given by Theorem~\ref{theorem: bounded bias for LR graident}, 
i.e.,
\begin{align*}
\Vert\Bias_k \Vert &\leq\frac{Z_1^{R} (t+1)}{|\mathcal{U}_k|}
 \sum_{{\pmb\theta_i}\in \mathcal{U}_k}\sum^i_{\ell = i-t}\eta_\ell + 2M\varphi(nt) + \frac{Z_2^{R}}{|\mathcal{U}_k|}
 \sum_{{\pmb\theta_i}\in \mathcal{U}_k} \sum^k_{\ell=i-t}\eta_\ell.
\end{align*}

Consider $\eta_k$ small enough that $1-2L\eta_k \geq \frac{1}{2}$ or equivalently $\eta_k \leq \frac{1}{4L}$. As the learning rate is non-increasing, this condition can be simplified as the initial learning rate $\eta_1 \leq\frac{1}{4L}$. Then it proceeds with
\begin{equation}
\left(1-2L\eta_k\right)\E\left[\Vert\nabla  J(\pmb\theta_k)\Vert^2\right] \geq\frac{1}{2}\E\left[\Vert\nabla  J(\pmb\theta_k)\Vert^2\right].
 \nonumber 
\end{equation}
Thus, we conclude the proof with the following bound:
\begin{align}
\frac{1}{2}\E\left[\Vert\nabla  J(\pmb\theta_k)\Vert^2\right] &\leq \frac{1}{\eta_k}\left(\E\left[ J(\pmb\theta_{k+1})\right]-\E\left[ J(\pmb\theta_k)\right]\right) + I_1 + I_2 + I_3.
\label{eq.mid7}
\end{align}
\endproof

\begin{sloppypar}
\noindent\textbf{Theorem \ref{convergence theorem}}
\textit{
Suppose Assumptions \ref{assumption 2} and \ref{assumption 3} hold. Let $\eta_k=\eta_1 k^{-r}$ denote the learning rate used in the $k$-th iteration with two constants $\eta_1\in(0, \frac{1}{4L}]$ and $r\in (0,1)$. By running Algorithm ~\ref{algo: online} with the restart sampler in Step~1(a) and with the replay buffer of size $B_K$, for both LR/CLR policy gradient estimators in Eq.~\eqref{eq.LR-gradient} and \eqref{eq.CLR-gradient} and for $K \ge 2t+2B_K$, we have the rate of convergence
 \begin{align*}
    \frac{1}{K}\sum^K_{k=1}\E\left[\Vert\nabla  J(\pmb\theta_k)\Vert^2\right]  & \leq   \frac{8U_ J /\eta_1}{K^{1-r}} + \frac{4cLM^2}{K}\sum_{k=1}^K{\eta_k} \bar{\rho}_k+
    {4M^2}\varphi(nt) + \frac{2^{r+1}C_3}{(1-r)K^r}  \nonumber\\
    &\quad  + \frac{2^{r+1}C_2\eta_1 (t+1)t}{(1-r)K^r} + \frac{2^{r+1}C_1\eta_1}{1-r}\frac{B_K+t}{K^r} + M^2\frac{B_K+t}{K}
 \end{align*}
where $C_1=C^\Gamma_1+C^\Gamma_2$, $C_2={2nM^3U_\pi  U_f}$, $C_3=\sup_{k\geq 1}\Vert \Bias_k\Vert^2$ with $\Bias_k=\E\left[\widehat{\nabla} J^{R}_{k}\right]-\E\left[\nabla J(\pmb\theta_k)\right]$ and $\bar{\rho}_{k}=\frac{1}{|\mathcal{U}_k|^2}\sum_{\pmb\theta_i\in\mathcal{U}_k}\sum_{\pmb\theta_{i^\prime}\in\mathcal{U}_k}\max_{\ell=1,\ldots,d}|\Corr^{(\ell)}_{i,i^\prime,k}|$. Here $R\in\{LR,CLR\},$  $C_1^\Gamma = {LM^2\sqrt{2c+1} (\sqrt{2c+1}+2)}$, $C_2^\Gamma = M^2(U_f L_g+M U_\pi)$,
and $c$ is the selection constant defined in Theorem~\ref{thm:Online selection rule}. 
Using $\mathcal{O}(\cdot)$ notation gives
 \begin{equation*}
    \frac{1}{K}\sum^K_{k=1}\E\left[\Vert\nabla  J(\pmb\theta_k)\Vert^2\right] 
    = \mathcal{O}\left(\frac{1}{K^{1-r}}\right) +
     \mathcal{O}\left(\frac{\sum_{k=1}^K{\eta_k}\bar{\rho}_k}{K} \right)+
    \mathcal{O}\left(\varphi(nt)\right) + \mathcal{O}\left(\frac{t^2}{K^r}\right) + \mathcal{O}\left(\frac{B_K+t}{K^r}\right)  
 \end{equation*}
\noindent where $n$ is the number of steps in each iteration. The notation
$\mathcal{O}(\cdot)$ hides constants $c$, $L$, $\eta_1$, $M$, $U_ J$, $n$, $U_f$, $U_\pi$, $\kappa_0$, $\kappa$ and $L_g$.
}
\end{sloppypar}

\proof{}
\noindent Consider a dynamic buffer of size $B_k > 0$. For any policy $\pmb\theta_i \in \mathcal{U}_k$, the index condition $k-i \leq B_k$ holds. Furthermore, we assume the buffer size $B_k$ increases monotonically with $k$.
Regarding the initial term on the right-hand side (RHS) of Lemma~\ref{lemma: main lemma for convergence analysis}, we expand the sum as follows:
\begin{align}
&\sum^K_{k=B_K+t}\frac{1}{\eta_k}\left(\E\left[ J(\pmb\theta_{k+1})\right]-\E\left[ J(\pmb\theta_k)\right]\right) \nonumber\\
&= -\frac{1}{\eta_{B_K+t}}\E\left[ J(\pmb\theta_{B_K+t})\right] +\frac{1}{\eta_{K}}\E\left[ J(\pmb\theta_{K+1})\right]+ \sum_{k=B_K+t+1}^K\left(\frac{1}{\eta_{k-1}}-\frac{1}{\eta_k}\right)\E\left[ J(\pmb\theta_k)\right]
\nonumber \\
 &\leq \frac{1}{\eta_{B_K+t}} U_ J+\frac{1}{\eta_{K}}U_ J + \sum_{k=B_K+t+1}^K\left(\frac{1}{\eta_{k}}-\frac{1}{\eta_{k-1}}\right)U_ J \leq \frac{2}{\eta_{K}}U_ J. 
\end{align}
Next, from the bound $| J(\pmb\theta_k)|\leq U_ J$ and summing the terms from $k = B_K+t$ through $K$:
\begin{align}
&\frac{1}{2}\sum_{k=B_K+t}^K\E\left[\Vert\nabla  J(\pmb\theta_k)\Vert^2\right] \leq \frac{{2}}{\eta_{K}}U_ J \nonumber\\
&\quad+\underbrace{\sum^K_{k=B_K+t}\frac{1}{|\mathcal{U}_k|}\sum_{{\pmb\theta_i}\in \mathcal{U}_k}
 \left(C_1(k-i+t)\eta_{i-t}+C_2 (t+1)t\eta_{i-t} +2M^2\varphi(nt)\right)}_{I_1}\nonumber\\
&\quad +\underbrace{cLM^2\sum_{k=B_K+t}^K\frac{\eta_k}{|\mathcal{U}_k|^2}\sum_{\pmb\theta_i\in\mathcal{U}_k}\sum_{\pmb\theta_{i^\prime}\in\mathcal{U}_k}\max_{\ell=1,\ldots,d}|\Corr^{(\ell)}_{i,i^\prime,k}|}_{I_2} +\underbrace{2L\sum^K_{k=B_K+t}\eta_k\Vert\Bias_k\Vert^2}_{I_3}.
\end{align} 

\noindent For term $I_1$, notice that (1) $k-i\leq B_k\leq B_K$ for any $K\geq k\geq i$, and (2) $\eta_{i-t}=\frac{\eta_1}{(i-t)^r}\leq \frac{\eta_1}{(k-B_K-t)^r}$. Thus, it holds that
\begin{align}
I_1&=\sum^K_{k=B_K+t}\frac{1}{|\mathcal{U}_k|}\sum_{{\pmb\theta_i}\in \mathcal{U}_k}
 \left(\frac{C_1\eta_1(B_K+t)}{(k-B_K-t)^r}+\frac{C_2\eta_1 t^2}{(k-B_K-t)^r} +2M^2\varphi(nt)\right) \nonumber\\
 &\leq \sum^K_{k=B_K+t}
 \left(\frac{C_1\eta_1(B_K+t)}{(k-B_K-t)^r}+\frac{C_2\eta_1 t^2}{(k-B_K-t)^r} +2M^2\varphi(nt)\right) \nonumber\\
 &\leq\frac{C_1\eta_1(B_K+t)}{1-r}(K-B_K-t)^{1-r}+\frac{C_2\eta_1 t^2}{1-r}(K-B_K-t)^{1-r} \nonumber\\
 & \quad +{2M^2}\varphi(nt)(K-B_K-t)\nonumber
\end{align}
where the last step holds due to the fact that $$\sum_{k=B_K+t}^K\frac{1}{(k-B_K-t)^r}\leq 
\int^{K-B_K-t}_0 \frac{\eta_1}{k^r}\dd k =\frac{\eta_1}{1-r}(K-B_K-t)^{1-r}.$$
\endproof

\noindent For term $I_2$, notice that the inequality $\max_{\ell=1,\ldots,d}|\Corr^{(\ell)}_{i,i^\prime,k}| \leq 1$, which leads to:
\begin{equation}\label{eq0: theorem 4}
    I_2=cLM^2\sum_{k=B_K+t}^K\frac{\eta_k}{|\mathcal{U}_k|^2}\sum_{\pmb\theta_i\in\mathcal{U}_k}\sum_{\pmb\theta_{i^\prime}\in\mathcal{U}_k}\max_{\ell=1,\ldots,d}|\Corr^{(\ell)}_{i,i^\prime,k}| \leq cLM^2\sum_{k=B_K+t}^K{\eta_k} \bar{\rho}_k.
\end{equation}
Here, $\bar{\rho}_{k}=\frac{1}{|\mathcal{U}_k|^2}\sum_{\pmb\theta_i\in\mathcal{U}_k}\sum_{\pmb\theta_{i^\prime}\in\mathcal{U}_k}\max_{\ell=1,\ldots,d}|\Corr^{(\ell)}_{i,i^\prime,k}|$.

\noindent For term $I_3$, notice that the bias term $\Vert \Bias_k\Vert$ is upper bounded by a convergent sequence based on Theorem~\ref{theorem: bounded bias for LR graident} and Eq.~\eqref{eq: big O of bias with buffer size}. For $R\in\{LR, CLR\}$
\begin{align}
\Vert\Bias_k \Vert &\leq\frac{Z_1^{R} (t+1)}{|\mathcal{U}_k|}
 \sum_{{\pmb\theta_i}\in \mathcal{U}_k}\sum^i_{\ell = i-t}\eta_\ell + 2M\varphi(nt) + \frac{Z_2^{R}}{|\mathcal{U}_k|}
 \sum_{{\pmb\theta_i}\in \mathcal{U}_k} \sum^k_{\ell=i-t}\eta_\ell \nonumber\\
 & \leq Z_1^{R}\eta_1\frac{(t+1)t}{(k-B_k-t)^r}+Z_2^{R}\eta_1\frac{B_k+t}{(k-B_k-t)^r}+ 2M\varphi(nt) \xrightarrow{k\rightarrow\infty} 0\label{eq: bias bound theorem 4}
\end{align}
where the step \eqref{eq: bias bound theorem 4} holds because $\sum^i_{\ell=i-t}\eta_\ell \leq \eta_1 t (i-t)^{-r}\leq \eta_1 t (k-B_k-t)^{-r}$ and $\sum^k_{\ell=i-t}\eta_\ell \leq \eta_1 (k-i+t) (i-t)^{-r}\leq \eta_1 (B_k+t)(k-B_k-t)^{-r}$. The convergence of a sequence implies that it is bounded. Thus, we set $C_3=\sup_{k\geq 1} \Vert\Bias_k\Vert^2$.
By applying $\sum^K_{k=B_K+t}\eta_k = \sum^K_{k=B_K+t}\frac{\eta_1}{k^r} \leq \int^{K-B_K-t}_{0}\frac{\eta_1}{k^r}\dd k \leq\frac{\eta_1}{1-r}(K-B_K-t)^{1-r}$ and the bias bound, we can show
\begin{align*}
    I_3 \leq {C_3}(K-B_K-t)^{1-r}.
\end{align*}
\endproof

\noindent Combining these components and normalizing by the factor $\frac{1}{2}(K-B_K-t+1)$ yields:
\begin{align}
  & \frac{\sum_{k=B_K+t}^K\E\left[\Vert\nabla  J(\pmb\theta_k)\Vert^2\right]}{K-B_K-t+1} 
   \leq \frac{4\eta_1^{-1} U_ J K^{r}}{K-B_K-t+1} + \frac{2cLM^2}{K-B_k-t+1}\sum_{k=B_K+t}^K{\eta_k} \bar{\rho}_k \nonumber\\
&\quad + \frac{2C_3\eta_1}{1-r}(K-B_K-t)^{-r}+ {4M^2}\varphi(nt) +\frac{2C_2\eta_1(t+1) t}{1-r}(K-B_K-t)^{-r} \nonumber\\
& \quad + \frac{2C_1\eta_1}{1-r}(B_K+t)(K-B_K-t)^{-r}. \label{eq1: theorem 4}
\end{align}
Since $K \ge 2t+2B_K$, we have $K-B_K-t \ge K/2$. Therefore,
\begin{align}
    \frac{4 U_ J K^r }{ K-B_K-t+1 }& \leq \frac{4U_ J K^r }{ K-B_K-t } \leq \frac{8U_ J}{K^{1-r}}, (K-B_K-t)^{-r}\leq \frac{2^r}{K^r}. \nonumber
\end{align}
Consequently, Equation~\eqref{eq1: theorem 4} simplifies to:
\begin{align}
   &\frac{\sum_{k=B_K+t}^K\E\left[\Vert\nabla  J(\pmb\theta_k)\Vert^2\right]}{K-B_K-t+1} \leq \frac{8U_J/\eta_1}{K^{1-r}} + \frac{2cLM^2}{K-B_k-t+1}\sum_{k=B_K+t}^K{\eta_k} \bar{\rho}_k\nonumber\\
&\quad + \frac{2^{r+1}C_3\eta_1}{1-r}K^{-r}+ {4 M^2}\varphi(nt)+\frac{2^{r+1}C_2\eta_1(t+1) t}{1-r}K^{-r}    + \frac{2^{r+1}C_1\eta_1}{1-r}(B_K+t)K^{-r}. \nonumber
\end{align}
Given that the learning rate $\eta_k$ is strictly positive and $\bar{\rho}_k > 0$, the condition $K\geq 2t+2B_K$ implies:
$$\frac{2cLM^2}{K-B_k-t+1}\sum_{k=B_K+t}^K{\eta_k} \bar{\rho}_k\leq \frac{4cLM^2}{K}\sum_{k=1}^K{\eta_k} \bar{\rho}_k.$$ 
The conclusion is obtained by observing that
\begin{equation*}
    \frac{1}{K}\sum^T_{k=1}a_k\leq \frac{1}{K}\left((\tau-1)U_a+\sum^K_{k=\tau}a_k\right)=\frac{\tau-1}{K}U_a+\frac{1}{K}\sum^K_{k=\tau}a_k \leq \frac{\tau-1}{K}U_a+\frac{1}{1+K-\tau}\sum^K_{k=\tau}a_k.
\end{equation*}
\endproof
Employing Lemma \ref{supportive lemma: mean sequence Big O} with parameters $f(K)=B_K+t$ and $U_a=M^2$, we derive:
\begin{align}
   \frac{1}{K}\sum_{k=1}^K\E\left[\Vert\nabla  J(\pmb\theta_k)\Vert^2\right] &\leq \frac{\sum_{k=B_K+t}^K\E\left[\Vert\nabla  J(\pmb\theta_k)\Vert^2\right]}{K-B_K-t+1} + M^2 \frac{B_K+t}{K} \nonumber\\
   &\leq \frac{8U_ J /\eta_1}{K^{1-r}} + \frac{4cLM^2}{K}\sum_{k=1}^K{\eta_k} \bar{\rho}_k+
    {4M^2}\varphi(nt) + \frac{2^{r+1}C_3\eta_1}{(1-r)K^r}  \nonumber\\
    &\quad  + \frac{2^{r+1}C_2\eta_1 (t+1)t}{(1-r)K^r} + \frac{2^{r+1}C_1\eta_1}{1-r}\frac{B_K+t}{K^r} + M^2\frac{B_K+t}{K}. \label{eq2: theorem 4}
\end{align}
Expressing Eq.~\eqref{eq2: theorem 4} in terms of
$\mathcal{O}(\cdot)$ leads to the secondary conclusion:
\begin{equation*}
    \frac{1}{K}\sum^K_{k=1}\E\left[\Vert\nabla  J(\pmb\theta_k)\Vert^2\right] 
    = \mathcal{O}\left(\frac{1}{K^{1-r}}\right) +
     \mathcal{O}\left(\frac{\sum_{k=1}^K{\eta_k}\bar{\rho}_k}{K} \right)+
    \mathcal{O}\left(\varphi(nt)\right) + \mathcal{O}\left(\frac{t^2}{K^r}\right) + \mathcal{O}\left(\frac{B_K+t}{K^r}\right). \\
\end{equation*}
This concludes the proof.
\endproof

\subsection{Proof of Corollary~\ref{cor: convergence rate}}
\label{appendix subsec: corollary convergence}

\noindent\textbf{Corollary \ref{cor: convergence rate}}
\textit{
Suppose Assumptions \ref{assumption 2} and \ref{assumption 3} hold. Under the same configurations as Theorem~\ref{convergence theorem}, by setting $t=\log_{\kappa}K^{-r/n}$, we have the rate of convergence
 \begin{equation*}
    \frac{1}{K}\sum^K_{k=1}\E\left[\Vert\nabla  J(\pmb\theta_k)\Vert^2\right] 
    \leq \mathcal{O}\left(\frac{1}{K^{1-r}}\right) + \mathcal{O}\left(\frac{1}{K^r}\right) + \mathcal{O}\left(\frac{t^2}{K^r}\right) + \mathcal{O}\left(\frac{B_K+t}{K^r}\right)  
 \end{equation*}
\noindent where the notation
$\mathcal{O}(\cdot)$ hides constants $c$, $L$, $\eta_1$, $M$, $U_ J$, $n$, $U_f$, $U_\pi$, $\kappa_0$, $\kappa$ and $L_g$.
}
\proof{} For the dynamic buffer with size $B_K>0$, it holds that $K-i\leq B_K$ for any $\pmb\theta_i\in\mathcal{U}_K$. Here the buffer size is an increasing function of $K$. By setting $t=\log_{\kappa}K^{-r/n}$, it holds that 
$$\varphi(nt)=\kappa_0 \kappa^{n\log_{\kappa} \left(\frac{1}{K^r}\right)^{1/n}}=\kappa_0 \kappa^{\log_{\kappa}\left(\frac{1}{K^r}\right)}=\kappa_0K^{-r}.$$
The conclusion is obtained by applying Theorem~\ref{convergence theorem}.
\endproof

\section{Proof of Technical Lemmas}
\label{appendix sec: proof of technical lemmas}

\begin{lemma}\label{appendix lemma: supportive lemma 1}
   Under Assumption~\ref{assumption 2}-\ref{assumption 3}, for any $t$ such that $t <i\leq k$ we have
   \begin{itemize}
     \item[\textbf{(i)}]  $\left\Vert\E\left[\widehat{\nabla} J^{LR}\left({\pmb{x}}^{(i,j)}, \pmb\theta_i,\pmb\theta_{k}\right)-\widehat{\nabla} J^{LR}\left({\pmb{x}}^{(i,j)}, \pmb\theta_i,\pmb\theta_{i-t}\right) \right] \right\Vert \leq 
     M(L_g + MU_\pi)  \sum^k_{\ell=i-t}\eta_\ell $
    \item[\textbf{(ii)}]    $\E\left[\left\Vert\widehat{\nabla} J^{LR}\left({\pmb{x}}^{(i,j)}, \pmb\theta_i,\pmb\theta_{k}\right)-\widehat{\nabla} J^{LR}\left({\pmb{x}}^{(i,j)}, \pmb\theta_i,\pmb\theta_{i-t}\right) \right\Vert\right] \leq 
    M(L_g + MU_\pi)  \sum^k_{\ell=i-t}\eta_\ell  $
\item[\textbf{(iii)}] $\E\left[\left\Vert\widehat{\nabla} J^{CLR}\left({\pmb{x}}^{(i,j)}, \pmb\theta_i,\pmb\theta_{k}\right)-\widehat{\nabla} J^{CLR}\left({\pmb{x}}^{(i,j)}, \pmb\theta_i,\pmb\theta_{i-t}\right)\right\Vert\right] \leq M(U_f L_g+M U_\pi) \sum^k_{\ell=i-t}\eta_\ell$
   \end{itemize}
\end{lemma}
\proof
\noindent\textbf{(i)} and \textbf{(ii)}: Let 
\begin{align*}
    \Delta_{\textbf{(i)}}&\defeq\left\Vert\E\left[\widehat{\nabla} J^{LR}\left({\pmb{x}}^{(i,j)}, \pmb\theta_i,\pmb\theta_{k}\right)-\widehat{\nabla} J^{LR}\left({\pmb{x}}^{(i,j)}, \pmb\theta_i,\pmb\theta_{i-t}\right)  \right] \right\Vert\\
\Delta_{\textbf{(ii)}}&\defeq\E\left[\left\Vert\widehat{\nabla} J^{LR}\left({\pmb{x}}^{(i,j)}, \pmb\theta_i,\pmb\theta_{k}\right)-\widehat{\nabla} J^{LR}\left({\pmb{x}}^{(i,j)}, \pmb\theta_i,\pmb\theta_{i-t}\right)\right\Vert\right]
\end{align*}
It follows that
\begin{align}
\Delta_{\textbf{(ii)}}&= \E\left[ \left\Vert\frac{\pi_{\pmb\theta_k}\left(\pmb{a}^{(i,j)}|\pmb{s}^{(i,j)}\right)}{\pi_{\pmb\theta_i}\left(\pmb{a}^{(i,j)}|\pmb{s}^{(i,j)}\right)} g\left(\pmb{x}^{(i,j)}|\pmb\theta_k\right)-\frac{\pi_{\pmb\theta_{i-t}}\left(\pmb{a}^{(i,j)}|\pmb{s}^{(i,j)}\right)}{\pi_{\pmb\theta_i}\left(\pmb{a}^{(i,j)}|\pmb{s}^{(i,j)}\right)} g\left(\pmb{x}^{(i,j)}|\pmb\theta_{i-t}\right)\right\Vert \right] \nonumber\\
&= 
\E\left[\left\Vert\frac{\pi_{\pmb\theta_{i-t}}\left(\pmb{a}^{(i,j)}|\pmb{s}^{(i,j)}\right)}{\pi_{\pmb\theta_i}\left(\pmb{a}^{(i,j)}|\pmb{s}^{(i,j)}\right)} \left[g\left(\pmb{x}^{(i,j)}|\pmb\theta_{i-t}\right)-g\left(\pmb{x}^{(i,j)}|\pmb\theta_{k}\right)\right] \right.\right.\nonumber\\
&\qquad \qquad +\left.\left.\frac{\pi_{\pmb\theta_{i-t}}\left(\pmb{a}^{(i,j)}|\pmb{s}^{(i,j)}\right)-\pi_{\pmb\theta_{k}}\left(\pmb{a}^{(i,j)}|\pmb{s}^{(i,j)}\right)}{\pi_{\pmb\theta_i}\left(\pmb{a}^{(i,j)}|\pmb{s}^{(i,j)}\right)} g\left(\pmb{x}^{(i,j)}|\pmb\theta_{k}\right) \right\Vert \right] \nonumber\\
&= \E\left[\frac{\pi_{\pmb\theta_{i-t}}\left(\pmb{a}^{(i,j)}|\pmb{s}^{(i,j)}\right)}{\pi_{\pmb\theta_i}\left(\pmb{a}^{(i,j)}|\pmb{s}^{(i,j)}\right)} \left\Vert g\left(\pmb{x}^{(i,j)}|\pmb\theta_k\right)-g\left(\pmb{x}^{(i,j)}|\pmb\theta_{i-t}\right)\right\Vert \right] \nonumber\\
&\quad +\E\left[\frac{\left|\pi_{\pmb\theta_{k}}\left(\pmb{a}^{(i,j)}|\pmb{s}^{(i,j)}\right)-\pi_{\pmb\theta_{i-t}}\left(\pmb{a}^{(i,j)}|\pmb{s}^{(i,j)}\right)\right|}{\pi_{\pmb\theta_i}\left(\pmb{a}^{(i,j)}|\pmb{s}^{(i,j)}\right)} \left\Vert g\left(\pmb{x}^{(i,j)}|\pmb\theta_{k}\right) \right\Vert\right] \nonumber\\
&\leq  \E\left[\E\left[\left. \left\Vert g\left(\pmb{x}|\pmb\theta_k\right)-g\left(\pmb{x}|\pmb\theta_{i-t}\right)\right\Vert \right| \pmb\theta_{i-t} \right]\right] \nonumber\\
& \quad +\E\left[\frac{\left|\pi_{\pmb\theta_{k}}\left(\pmb{a}^{(i,j)}|\pmb{s}^{(i,j)}\right)-\pi_{\pmb\theta_{i-t}}\left(\pmb{a}^{(i,j)}|\pmb{s}^{(i,j)}\right)\right|}{\pi_{\pmb\theta_i}\left(\pmb{a}^{(i,j)}|\pmb{s}^{(i,j)}\right)} \left\Vert g\left(\pmb{x}^{(i,j)}|\pmb\theta_{k}\right) \right\Vert\right] \label{eq0: appendix lemma: supportive lemma 1}.
\end{align}
By applying Lemma~\ref{lemma: lipchitz continuity of sample gradient estimate}, we have $\Vert g\left(\pmb{x}^{(i,j)}|\pmb\theta_k\right)-g\left(\pmb{x}^{(i,j)}|\pmb\theta_{i-t}\right)\Vert\leq L_g \Vert\pmb\theta_k-\pmb\theta_{i-t}\Vert$. By Lemma~\ref{lemma: bounded of policy gradient}, we have $\Vert g\left(\pmb{x}^{(i,j)}|\pmb\theta_{k}\right)\Vert \leq M$. Then the conclusion follows that
\begin{align}
\Delta_{\textbf{(i)}} \leq \Delta_{\textbf{(ii)}} &\leq 
\E\left[\E\left[\left. \left\Vert g\left(\pmb{x}|\pmb\theta_k\right)-g\left(\pmb{x}|\pmb\theta_{i-t}\right)\right\Vert \right| \pmb\theta_{i-t} \right]\right] \nonumber\\
&\quad +\E\left[\frac{|\pi_{\pmb\theta_{k}}\left(\pmb{a}^{(i,j)}|\pmb{s}^{(i,j)}\right)-\pi_{\pmb\theta_{i-t}}\left(\pmb{a}^{(i,j)}|\pmb{s}^{(i,j)}\right)|}{\pi_{\pmb\theta_i}\left(\pmb{a}^{(i,j)}|\pmb{s}^{(i,j)}\right)} \left\Vert g\left(\pmb{x}^{(i,j)}|\pmb\theta_{i-t}\right)\right\Vert  \right] \label{eq-0: appendix lemma: supportive lemma 1} \\
&\leq  L_g \E\left[\E\left[\Vert\pmb\theta_k-\pmb\theta_{i-t}\Vert| \pmb\theta_{i-t}\right] \right] + M  U_\pi \E[\Vert\pmb\theta_k-\pmb\theta_{i-t}\Vert]\label{eq-1: appendix lemma: supportive lemma 1}\\
&\leq  (L_g + M  U_\pi )\E[\Vert\pmb\theta_k-\pmb\theta_{i-t}\Vert]
\end{align}
where Step~\eqref{eq-0: appendix lemma: supportive lemma 1} follows \eqref{eq0: appendix lemma: supportive lemma 1}. The first term of Step~\eqref{eq-1: appendix lemma: supportive lemma 1} holds due to Lemma~\ref{lemma: lipchitz continuity of sample gradient estimate}. The second term of Step~\eqref{eq-1: appendix lemma: supportive lemma 1} holds because
\begin{align*}
&\E\left[\frac{|\pi_{\pmb\theta_{k}}\left(\pmb{a}^{(i,j)}|\pmb{s}^{(i,j)}\right)-\pi_{\pmb\theta_{i-t}}\left(\pmb{a}^{(i,j)}|\pmb{s}^{(i,j)}\right)|}{\pi_{\pmb\theta_i}\left(\pmb{a}^{(i,j)}|\pmb{s}^{(i,j)}\right)} \left\Vert g\left(\pmb{x}^{(i,j)}|\pmb\theta_{i-t}\right)\right\Vert \right]\\
&=M\E\left[\int_{\mathcal{A}}\left|\pi_{\pmb\theta_{k}}\left(\pmb{a}^{(i,j)}|\pmb{s}^{(i,j)}\right)-\pi_{\pmb\theta_{i-t}}\left(\pmb{a}^{(i,j)}|\pmb{s}^{(i,j)}\right)\right|\dd \pmb{a}^{(i,j)}\right]\\
&\leq M U_\pi \E\left[\Vert\pmb\theta_k - \pmb\theta_{i-t}\Vert\right]. \nonumber
\end{align*}
The last step holds by applying (\ref{eq: assumption item 3}) in Assumption~\ref{assumption 2}. Applying 
Lemma~\ref{auxillary lemma: parametric distance} to \eqref{eq-1: appendix lemma: supportive lemma 1} gives
\begin{equation*}
\Delta_{\textbf{(i)}} \leq \Delta_{\textbf{(ii)}} \leq  M(L_g + MU_\pi)  \sum^k_{\ell=i-t}\eta_\ell 
\end{equation*}

\noindent\textbf{(iii)} The proof of \textbf{(iii)} is similar to that of \textbf{(i)} and \textbf{(ii)}.

Bound the likelihood ratio by $\min\left\{\frac{\pi_{\pmb\theta_k}\left(\pmb{a}^{(i,j)}|\pmb{s}^{(i,j)}\right)}{\pi_{\pmb\theta_i}\left(\pmb{a}^{(i,j)}|\pmb{s}^{(i,j)}\right)}, U_f\right\}\leq U_f$, and notice that  for any $(\pmb{s},\pmb{a})\in\mathcal{S}\times\mathcal{A}$, it holds that

\begin{equation*}
    \min\left\{\frac{\pi_{\pmb\theta_k}\left(\pmb{a}|\pmb{s}\right)}{\pi_{\pmb\theta_i}\left(\pmb{a}|\pmb{s}\right)}, U_f\right\} - \min\left\{\frac{\pi_{\pmb\theta_{i-t}}\left(\pmb{a}|\pmb{s}\right)}{\pi_{\pmb\theta_i}\left(\pmb{a}|\pmb{s}\right)}, U_f\right\} \leq \frac{|\pi_{\pmb\theta_{k}}\left(\pmb{a}|\pmb{s}\right)-\pi_{\pmb\theta_{i-t}}\left(\pmb{a}|\pmb{s}\right)|}{\pi_{\pmb\theta_i}\left(\pmb{a}|\pmb{s}\right)}.
\end{equation*}
Then by applying Lemma~\ref{lemma: lipchitz continuity of sample gradient estimate} and Assumption~\ref{assumption 2}, the conclusion follows that
\begin{align}
& \E\left[\left\Vert\widehat{\nabla} J^{CLR}\left({\pmb{x}}^{(i,j)}, \pmb\theta_i,\pmb\theta_{k}\right)-\widehat{\nabla} J^{CLR}\left({\pmb{x}}^{(i,j)}, \pmb\theta_i,\pmb\theta_{i-t}\right)\right\Vert\right] \nonumber\\
&\leq \E\left[\min\left\{\frac{\pi_{\pmb\theta_k}\left(\pmb{a}^{(i,j)}|\pmb{s}^{(i,j)}\right)}{\pi_{\pmb\theta_i}\left(\pmb{a}^{(i,j)}|\pmb{s}^{(i,j)}\right)}, U_f\right\} \left\Vert g\left(\pmb{x}^{(i,j)}|\pmb\theta_k\right)-g\left(\pmb{x}^{(i,j)}|\pmb\theta_{i-t}\right)\right\Vert\right] \nonumber\\
&\quad +\E\left[\frac{|\pi_{\pmb\theta_{k}}\left(\pmb{a}^{(i,j)}|\pmb{s}^{(i,j)}\right)-\pi_{\pmb\theta_{i-t}}\left(\pmb{a}^{(i,j)}|\pmb{s}^{(i,j)}\right)|}{\pi_{\pmb\theta_i}\left(\pmb{a}^{(i,j)}|\pmb{s}^{(i,j)}\right)} \left\Vert g\left(\pmb{x}^{(i,j)}|\pmb\theta_{i-t}\right)\right\Vert  \right] \nonumber\\
&\leq U_f L_g \E\left[\Vert\pmb\theta_k-\pmb\theta_{i-t}\Vert\right] + M  U_\pi \E[\Vert\pmb\theta_k-\pmb\theta_{i-t}\Vert]\nonumber\\
&\leq (U_f L_g+M U_\pi) M \sum^k_{\ell=i-t}\eta_\ell.\nonumber
\end{align}
\endproof

\begin{lemma}\label{appendix lemma: supportive lemma 2}
   Under Assumptions~\ref{assumption 2}-\ref{assumption 3}, for any $t$ such that $t <i\leq k$, we have
   \begin{itemize}
      \item[\textbf{(i)}] $\left\Vert\E\left[\widehat{\nabla} J^{LR}\left({\pmb{x}}^{(i,j)}, \pmb\theta_i,\pmb\theta_{i-t}\right)-\widehat{\nabla} J\left(\tilde{\pmb{x}}^{(i,j)}, \pmb\theta_i,\pmb\theta_{i-t}\right)\right] \right\Vert \leq  2 nM^2U_\pi t  \sum^{i}_{\ell=i-t}\eta_{\ell}$ \\
    \item[\textbf{(ii)}] $\E\left[\left\Vert\widehat{\nabla} J^{LR}\left({\pmb{x}}^{(i,j)}, \pmb\theta_i,\pmb\theta_{i-t}\right)-\widehat{\nabla} J\left(\tilde{\pmb{x}}^{(i,j)}, \pmb\theta_i,\pmb\theta_{i-t}\right)\right\Vert\right]  \leq 2M + 2M^2U_\pi\sum^i_{\ell=i-t}\eta_\ell$ \\
 \item[\textbf{(iii)}] $\left\Vert\E\left[\widehat{\nabla} J^{CLR}\left({\pmb{x}}^{(i,j)}, \pmb\theta_i,\pmb\theta_{i-t}\right)-\widehat{\nabla} J\left(\tilde{\pmb{x}}^{(i,j)}, \pmb\theta_i,\pmb\theta_{i-t}\right)\right]\right\Vert\leq 
      2nM^2U_\pi U_f t\sum^i_{\ell=i-t}\eta_\ell$
    \item[\textbf{(iv)}] $\E\left[\left\Vert\widehat{\nabla} J^{CLR}\left({\pmb{x}}^{(i,j)}, \pmb\theta_i,\pmb\theta_{i-t}\right)-\widehat{\nabla} J\left(\tilde{\pmb{x}}^{(i,j)}, \pmb\theta_i,\pmb\theta_{i-t}\right)\right\Vert\right]\leq 
       (3U_f+1)M$
   \end{itemize}
\end{lemma}
\proof

\noindent \textbf{(i)}: Let $\Delta_{\textbf{(i)}} = \left\Vert \E\left[\left.\widehat{\nabla} J^{LR}\left({\pmb{x}}^{(i,j)}, \pmb\theta_i,\pmb\theta_{i-t}\right)-\widehat{\nabla} J\left(\tilde{\pmb{x}}^{(i,j)}, \pmb\theta_i,\pmb\theta_{i-t}\right)\right|\pmb{s}^{(i-t,j)}, \pmb\theta_{i-t}\right]\right\Vert$.

Conditioning on $\pmb\theta_{i-t}$ and $\pmb{s}^{(i-t,j)}$, we have
\begin{align}
    \Delta_{\textbf{(i)}}&=\left\Vert \E\left[\left.\frac{\pi_{\pmb\theta_{i-t}}\left(\pmb{a}^{(i,j)}|\pmb{s}^{(i,j)}\right)}{\pi_{\pmb\theta_i}\left(\pmb{a}^{(i,j)}|\pmb{s}^{(i,j)}\right)} g\left(\pmb{x}^{(i,j)}|\pmb\theta_{i-t}\right)- g\left(\tilde{\pmb{x}}^{(i,j)}|\pmb\theta_{i-t}\right)\right| \pmb{s}^{(i-t,j)}, \pmb\theta_{i-t}\right]\right\Vert\nonumber\\
&=\left\Vert \int \P\left(\pmb{s}^{(i,j)}=\dd \pmb{s}|\pmb{s}^{(i-t,j)},\pmb\theta_{i-t}\right)\pi_{\pmb\theta_{i-t}}\left(\pmb{a}|\pmb{s}\right)g\left(\pmb{s}, \pmb{a}|\pmb\theta_{i-t}\right)\dd \pmb{a} \right. \nonumber\\ 
& \qquad \left. - \int \P\left(\tilde{\pmb{s}}^{(i,j)}=\dd \pmb{s}|\pmb{s}^{(i-t,j)},\pmb\theta_{i-t}\right)\pi_{\pmb\theta_{i-t}}\left(\pmb{a}|\pmb{s}\right)g\left(\pmb{s}, \pmb{a}|\pmb\theta_{i-t}\right)\dd \pmb{a} \right\Vert \nonumber\\ 
&=\left\Vert \int \left(\P(\pmb{s}^{(i,j)}=\dd \pmb{s}|\pmb{s}^{(i-t,j)},\pmb\theta_{i-t})-\P(\tilde{\pmb{s}}^{(i,j)}=\dd \pmb{s}|\pmb{s}^{(i-t,j)},\pmb\theta_{i-t}) \right)\pi_{\pmb\theta_{i-t}}\left(\pmb{a}|\pmb{s}\right)g\left(\pmb{s}, \pmb{a}|\pmb\theta_{i-t}\right)\dd \pmb{a}\right\Vert \nonumber\\
&\leq \int \left|\P(\pmb{s}^{(i,j)}=\dd \pmb{s}|\pmb{s}^{(i-t,j)},\pmb\theta_{i-t})-\P(\tilde{\pmb{s}}^{(i,j)}=\dd \pmb{s}|\pmb{s}^{(i-t,j)},\pmb\theta_{i-t}) \right|\pi_{\pmb\theta_{i-t}}\left(\pmb{a}|\pmb{s}\right)\left\Vert g\left(\pmb{s}, \pmb{a}|\pmb\theta_{i-t}\right)\right\Vert\dd \pmb{a} \nonumber\\
&\leq M \int\left|\P(\pmb{s}^{(i,j)}=\dd \pmb{s}|\pmb{s}^{(i-t,j)},\pmb\theta_{i-t})-\P(\tilde{\pmb{s}}^{(i,j)}=\dd \pmb{s}|\pmb{s}^{(i-t,j)},\pmb\theta_{i-t}) \right| \int\pi_{\pmb\theta_{i-t}}\left(\pmb{a}|\pmb{s}\right)\dd \pmb{a}\label{eq1: appendix lemma: supportive lemma 2}\\
&\leq 2M\left\Vert \P\left(\pmb{s}^{(i,j)}\in \cdot|\pmb{s}^{(i-t,j)},\pmb\theta_{i-t}\right)-\P\left(\tilde{\pmb{s}}^{(i,j)}\in \cdot|\pmb{s}^{(i-t,j)},\pmb\theta_{i-t}\right)\right\Vert_{TV}\label{eq: appendix lemma: supportive lemma 2}
\end{align}
where Step~\eqref{eq1: appendix lemma: supportive lemma 2} holds due to Lemma~\ref{lemma: bounded of policy gradient} and the last inequality is by the definition of total variation. Let $$\Delta^{(i,j)}=2M\left\Vert \P\left(\pmb{s}^{(i,j)}\in \cdot|\pmb{s}^{(i-t,j)},\pmb\theta_{i-t}\right)-\P\left(\tilde{\pmb{s}}^{(i,j)}\in \cdot|\pmb{s}^{(i-t,j)},\pmb\theta_{i-t}\right)\right\Vert_{TV}.$$ By Lemma~\ref{lemma: relations of total variation measures}, we have 
\begin{align}
\Delta^{(i,j)}&\leq 2M\left\Vert \P\left(
\left(\pmb{s}^{(i,j-1)},\pmb{a}^{(i,j-1)}  \right)\in \cdot|\pmb{s}^{(i-t,j)},\pmb\theta_{i-t}\right)-\P\left(\left(\tilde{\pmb{s}}^{(i,j-1)},\tilde{\pmb{a}}^{(i,j-1)} \right)\in \cdot|\pmb{s}^{(i-t,j)},\pmb\theta_{i-t}\right)\right\Vert_{TV} \nonumber\\
&\leq 2M\left\Vert \P\left(\left.\pmb{s}^{(i,j-1)}\in \cdot\right|\pmb{s}^{(i-t,j)},\pmb\theta_{i-t}\right)-\P\left(\left.\tilde{\pmb{s}}^{(i,j-1)}\in \cdot\right|\pmb{s}^{(i-t,j)},\pmb\theta_{i-t}\right)\right\Vert_{TV} \nonumber\\
&\quad +2M U_\pi \E\left[\Vert \pmb\theta_{i}-\pmb\theta_{i-t}\Vert|\pmb\theta_{i-t}\right] \nonumber\\
&={\Delta^{(i,j-1)}} + 2M U_\pi \E\left[\Vert \pmb\theta_{i}-\pmb\theta_{i-t}\Vert|\pmb\theta_{i-t}\right]\label{eq2: appendix lemma: supportive lemma 2}\\
& \cdots\nonumber\\
&={\Delta^{(i-1,n)}} + 2M U_\pi \sum^j_{j^\prime=1}\E\left[\Vert \pmb\theta_{i}-\pmb\theta_{i-t}\Vert|\pmb\theta_{i-t}\right]. \nonumber
\end{align}
In sum, repeating applying the inequality~\eqref{eq2: appendix lemma: supportive lemma 2} from $(i,j)$ to $(i-t,1)$ gives
\begin{align}
    \Delta^{(i,j)}\leq \ldots &\leq \Delta^{(i-t,1)} + 2M U_\pi \sum^j_{j^\prime=1}\E\left[\Vert \pmb\theta_{i}-\pmb\theta_{i-t}\Vert|\pmb\theta_{i-t}\right]+2M U_\pi\sum^{i-1}_{i^\prime=i-t}\sum^{n}_{j=1}\E[\Vert \pmb\theta_{i^\prime}-\pmb\theta_{i-t}\Vert|\pmb\theta_{i-t}] \nonumber\\
    &\leq \Delta^{(i-t,1)} + 2n M U_\pi \sum^{i}_{i^\prime=i-t}\E\left[\Vert \pmb\theta_{i^\prime}-\pmb\theta_{i-t}\Vert|\pmb\theta_{i-t}\right] \nonumber
\end{align}
where the last step holds as $\sum^j_{j^\prime=1}\E\left[\Vert \pmb\theta_{i}-\pmb\theta_{i-t}\Vert|\pmb\theta_{i-t}\right] \leq n \E\left[\Vert \pmb\theta_{i}-\pmb\theta_{i-t}\Vert|\pmb\theta_{i-t}\right]$. By noting $$\Delta^{(i-t,j)}=2M\left\Vert \P\left(\pmb{s}^{(i-t,j)}\in \cdot|\pmb{s}^{(i-t,j)}\right)-\P\left(\tilde{\pmb{s}}^{(i-t,j)}\in \cdot|\pmb{s}^{(i-t,j)}\right)\right\Vert_{TV}=0 \text{  for any $j\in [n]$}.$$ 
Thus we have  $\Delta^{(i,j)}\leq 2n M U_\pi \sum^{i}_{i^\prime=i-t}\E[\Vert \pmb\theta_{i^\prime}-\pmb\theta_{i-t}\Vert|\pmb\theta_{i-t}]$. 

By Lemma~\ref{auxillary lemma: parametric distance}, we have $\E[\Vert \pmb\theta_{i^\prime}-\pmb\theta_{i-t}\Vert|\pmb\theta_{i-t}] \leq M\sum^{i^\prime}_{\ell=i-t}{\eta_{\ell}}$ and thus it holds that
\begin{equation}\label{eq4: appendix lemma: supportive lemma 2}
\Delta_{\textbf{(i)}}\leq\Delta^{(i,j)} \leq 2nU_\pi M^2\sum^{i}_{i^\prime=i-t}  \sum^{i^\prime}_{\ell=i-t}\eta_{\ell} \leq 2nU_\pi M^2t  \sum^{i}_{\ell=i-t}\eta_{\ell}.
\end{equation}
Then the conclusion follows from the fact that $\E[\Delta_{\textbf{(i)}}]\leq 2nU_\pi M^2t  \sum^{i}_{\ell=i-t}\eta_{\ell}$.

\vspace{0.1in}
\noindent \textbf{(ii)}: Different from \textbf{(i)}, we define the conditional expectation $$\Delta_{\textbf{(ii)}} = \E\left[\left.\left\Vert \widehat{\nabla} J^{LR}\left({\pmb{x}}^{(i,j)}, \pmb\theta_i,\pmb\theta_{i-t}\right)-\widehat{\nabla} J\left(\tilde{\pmb{x}}^{(i,j)}, \pmb\theta_i,\pmb\theta_{i-t}\right)\right\Vert\right|\pmb{s}^{(i-t,j)},\pmb\theta_{i-t}\right].$$
It follows that
\begin{align}
\Delta_{\textbf{(ii)}} &=\E\left[\left.\left\Vert \frac{\pi_{\pmb\theta_{i-t}}\left(\pmb{a}^{(i,j)}|\pmb{s}^{(i,j)}\right)}
{\pi_{\pmb\theta_i}\left(\pmb{a}^{(i,j)}|\pmb{s}^{(i,j)}\right)}g\left(\pmb{s}^{(i,j)},\pmb{a}^{(i,j)}|\pmb\theta_{i-t}\right) - g\left(\tilde{\pmb{s}}^{(i,j)},\tilde{\pmb{a}}^{(i,j)}|\pmb\theta_{i-t}\right)\right\Vert\right|\pmb{s}^{(i-t,j)},\pmb\theta_{i-t}\right] \nonumber\\
 %
%
&\leq \E\left[\left.\left\Vert \frac{\pi_{\pmb\theta_{i-t}}\left(\pmb{a}^{(i,j)}|\pmb{s}^{(i,j)}\right)}
{\pi_{\pmb\theta_i}\left(\pmb{a}^{(i,j)}|\pmb{s}^{(i,j)}\right)}g\left(\pmb{s}^{(i,j)},\pmb{a}^{(i,j)}|\pmb\theta_{i-t}\right) - g\left({\pmb{s}}^{(i,j)},{\pmb{a}}^{(i,j)}|\pmb\theta_{i-t}\right) \right\Vert\right|\pmb{s}^{(i-t,j)},\pmb\theta_{i-t} \right]\nonumber\\
&\quad +\E\left[\left.\left\Vert g\left({\pmb{s}}^{(i,j)},{\pmb{a}}^{(i,j)}|\pmb\theta_{i-t}\right) - g\left(\tilde{\pmb{s}}^{(i,j)},\tilde{\pmb{a}}^{(i,j)}|\pmb\theta_{i-t}\right)\right\Vert\right|\pmb{s}^{(i-t,j)},\pmb\theta_{i-t}\right]. \nonumber
\end{align}
Let $p\left(\pmb{\tau}\right)=p\left(\pmb{s}^{(i-t,j)},\pmb{a}^{(i-t,j)}, \pmb{s}^{(i-t,j+1)},\ldots,\pmb{s}^{(i,j)}\right)$ denote the probability density of sample path from $\pmb{s}^{(i-t,j)}$ to $\pmb{s}^{(i,j)}$. Conditioning on $\pmb{s}^{(i-t,j)}$ and $\pmb\theta_{i-t}$, we have
\begin{align}
\Delta_{\textbf{(ii)}} &\leq \E\left[\left.\frac{|\pi_{\pmb\theta_{i-t}}\left(\pmb{a}^{(i,j)}|\pmb{s}^{(i,j)}\right)-\pi_{\pmb\theta_{i}}\left(\pmb{a}^{(i,j)}|\pmb{s}^{(i,j)}\right)|}
{\pi_{\pmb\theta_i}\left(\pmb{a}^{(i,j)}|\pmb{s}^{(i,j)}\right)} \left\Vert g\left(\pmb{s}^{(i,j)},\pmb{a}^{(i,j)}|\pmb\theta_{i-t}\right)\right\Vert \right|\pmb{s}^{(i-t,j)},\pmb\theta_{i-t}\right] +2M\nonumber\\
&\leq M \int p\left(\pmb{\tau}\right) \left(\int_{\mathcal{A}}\pi_{\pmb\theta_i}\left(\pmb{a}^{(i,j)}|\pmb{s}^{(i,j)}\right)\frac{|\pi_{\pmb\theta_{i-t}}\left(\pmb{a}^{(i,j)}|\pmb{s}^{(i,j)}\right)-\pi_{\pmb\theta_{i}}\left(\pmb{a}^{(i,j)}|\pmb{s}^{(i,j)}\right)|}
{\pi_{\pmb\theta_i}\left(\pmb{a}^{(i,j)}|\pmb{s}^{(i,j)}\right)} \dd \pmb{a}^{(i,j)}\right)\dd \pmb{\tau} +2M\nonumber\\
&\leq M \int p\left(\pmb{\tau}\right) \left(\int_{\mathcal{A}}\left|\pi_{\pmb\theta_{i-t}}\left(\pmb{a}|\pmb{s}^{(i,j)}\right)-\pi_{\pmb\theta_{i}}\left(\pmb{a}|\pmb{s}^{(i,j)}\right)\right| \dd \pmb{a}\right)\dd \pmb{\tau} +2M\nonumber\\
&\leq 2M \E\left[\left.\left\Vert\pi_{\pmb\theta_{i-t}}\left(\cdot|\pmb{s}^{(i,j)}\right)-\pi_{\pmb\theta_{i}}\left(\cdot|\pmb{s}^{(i,j)}\right)\right\Vert_{TV}\right|\pmb{s}^{(i-t,j)},\pmb\theta_{i-t}\right] +2M\label{eq5: appendix lemma: supportive lemma 2} \\
&\leq 2M U_\pi\E\left[\left.\left\Vert \pmb\theta_{i}-\pmb\theta_{i-t}\right\Vert\right|\pmb{s}^{(i-t,j)},\pmb\theta_{i-t}\right] +2M\label{eq6: appendix lemma: supportive lemma 2}\\
&\leq 2M^2U_\pi\sum^i_{\ell=i-t}\eta_\ell +2M \label{eq7: appendix lemma: supportive lemma 2}
\end{align}
where the expectation in Step~\eqref{eq5: appendix lemma: supportive lemma 2} is taken over the sample path from $\pmb{s}^{(i-t,j)}$ to $\pmb{s}^{(i,j)}$. Step~\eqref{eq6: appendix lemma: supportive lemma 2} holds by Assumption~\ref{assumption 2} and Step~\eqref{eq7: appendix lemma: supportive lemma 2} follows Lemma~\ref{auxillary lemma: parametric distance}. The conclusion follows by the fact that $$\E\left[\left\Vert\widehat{\nabla} J^{LR}\left({\pmb{x}}^{(i,j)}, \pmb\theta_i,\pmb\theta_{i-t}\right)-\widehat{\nabla} J\left(\tilde{\pmb{x}}^{(i,j)}, \pmb\theta_i,\pmb\theta_{i-t}\right)\right\Vert\right] 
\leq 2M^2U_\pi\sum^i_{\ell=i-t}\eta_\ell +2M.$$

\noindent \textbf{(iii)}: We define conditional expectation and notice that $\min(1,U_f) \equiv 1$ as $U_f > 1$ \begin{align*}
    \Delta_{\textbf{(iii)}}=\left\Vert \E\left[\left.\min\left(\frac{\pi_{\pmb\theta_{i-t}}\left(\pmb{a}^{(i,j)}|\pmb{s}^{(i,j)}\right)}{\pi_{\pmb\theta_i}\left(\pmb{a}^{(i,j)}|\pmb{s}^{(i,j)}\right)}, U_f \right)g\left(\pmb{x}^{(i,j)}|\pmb\theta_{i-t}\right)- g\left(\tilde{\pmb{x}}^{(i,j)}|\pmb\theta_{i-t}\right)\right|\pmb{s}^{(i-t,j)},\pmb\theta_{i-t}\right]\right\Vert\\
=\left\Vert \E\left[\left.\min\left(\frac{\pi_{\pmb\theta_{i-t}}\left(\pmb{a}^{(i,j)}|\pmb{s}^{(i,j)}\right)}{\pi_{\pmb\theta_i}\left(\pmb{a}^{(i,j)}|\pmb{s}^{(i,j)}\right)}, U_f \right)g\left(\pmb{x}^{(i,j)}|\pmb\theta_{i-t}\right)- \min(1,U_f) g\left(\tilde{\pmb{x}}^{(i,j)}|\pmb\theta_{i-t}\right)\right|\pmb{s}^{(i-t,j)},\pmb\theta_{i-t}\right]\right\Vert.
\end{align*}
By the Lipschitz property of the function $\min(\cdot, U_f)$, \begin{equation}
    \min\left(\frac{\pi_{\pmb\theta_{i-t}}\left(\pmb{a}^{(i,j)}|\pmb{s}^{(i,j)}\right)}{\pi_{\pmb\theta_i}\left(\pmb{a}^{(i,j)}|\pmb{s}^{(i,j)}\right)}, U_f \right) - \min(1, U_f) \leq \left|\frac{\pi_{\pmb\theta_{i-t}}\left(\pmb{a}^{(i,j)}|\pmb{s}^{(i,j)}\right)}{\pi_{\pmb\theta_i}\left(\pmb{a}^{(i,j)}|\pmb{s}^{(i,j)}\right)}- 1 \right|. \label{eq9: appendix lemma: supportive lemma 2}
\end{equation}
Conditioning on $\pmb{s}^{(i-t,j)}$ and $\pmb\theta_{i-t}$, it holds
\begin{align}
\Delta_{\textbf{(iii)}}&\leq\left\Vert \E\left[\left.\min\left(\frac{\pi_{\pmb\theta_{i-t}}\left(\pmb{a}^{(i,j)}|\pmb{s}^{(i,j)}\right)}{\pi_{\pmb\theta_i}\left(\pmb{a}^{(i,j)}|\pmb{s}^{(i,j)}\right)}, U_f \right)\left[g\left(\pmb{x}^{(i,j)}|\pmb\theta_{i-t}\right)- g\left(\tilde{\pmb{x}}^{(i,j)}|\pmb\theta_{i-t}\right)\right]\right|\pmb{s}^{(i-t,j)},\pmb\theta_{i-t}\right]\right\Vert \nonumber\\
&\quad +\left\Vert \E\left[\left.\left(\min\left(\frac{\pi_{\pmb\theta_{i-t}}\left(\pmb{a}^{(i,j)}|\pmb{s}^{(i,j)}\right)}{\pi_{\pmb\theta_i}\left(\pmb{a}^{(i,j)}|\pmb{s}^{(i,j)}\right)}, U_f \right) - \min(1, U_f)\right)g\left(\tilde{\pmb{x}}^{(i,j)}|\pmb\theta_{i-t}\right)\right|\pmb{s}^{(i-t,j)},\pmb\theta_{i-t}\right]\right\Vert\nonumber\\
&\leq\left\Vert \E\left[\left.U_f\left[g\left(\pmb{x}^{(i,j)}|\pmb\theta_{i-t}\right)- g\left(\tilde{\pmb{x}}^{(i,j)}|\pmb\theta_{i-t}\right)\right]\right|\pmb{s}^{(i-t,j)},\pmb\theta_{i-t}\right] \right\Vert \label{eq8: appendix lemma: supportive lemma 2}\\
&\quad +\E\left[\left.\frac{\left|\pi_{\pmb\theta_{i-t}}\left(\pmb{a}^{(i,j)}|\pmb{s}^{(i,j)}\right)-\pi_{\pmb\theta_i}\left(\pmb{a}^{(i,j)}|\pmb{s}^{(i,j)}\right)\right|}{\pi_{\pmb\theta_i}\left(\pmb{a}^{(i,j)}|\pmb{s}^{(i,j)}\right)} \left\Vert g\left(\tilde{\pmb{x}}^{(i,j)}|\pmb\theta_{i-t}\right)\right\Vert\right|\pmb{s}^{(i-t,j)},\pmb\theta_{i-t}\right] \label{eq10: appendix lemma: supportive lemma 2}
\end{align}
where the term~\eqref{eq10: appendix lemma: supportive lemma 2} holds due to the inequality \eqref{eq9: appendix lemma: supportive lemma 2}.

\noindent \textbf{First}, for term \eqref{eq8: appendix lemma: supportive lemma 2}, we have
\begin{align}
&\left\Vert \E\left[\left.U_f\left[g\left(\pmb{x}^{(i,j)}|\pmb\theta_{i-t}\right)- g\left(\tilde{\pmb{x}}^{(i,j)}|\pmb\theta_{i-t}\right)\right]\right|\pmb{s}^{(i-t,j)},\pmb\theta_{i-t}\right] \right\Vert \nonumber\\
&\quad \leq U_f \left\Vert\int_{\mathcal{S}}\int_{\mathcal{A}} \P\left(\pmb{s}^{(i,j)}=\dd \pmb{s}|\pmb{s}^{(i-t,j)},\pmb\theta_{i-t}\right) \pi_{\pmb\theta_i}(\pmb{a}|\pmb{s})g\left(\pmb{s},\pmb{a}|\pmb\theta_{i-t}\right) \right.\nonumber \\
&\qquad\qquad\quad  - \left.\P\left(\tilde{\pmb{s}}^{(i,j)}=\dd \pmb{s}|\pmb{s}^{(i-t,j)},\pmb\theta_{i-t}\right)\pi_{\pmb\theta_{i-t}}(\pmb{a}|\pmb{s})g\left(\pmb{s},\pmb{a}|\pmb\theta_{i-t}\right) \dd\pmb{a} \right\Vert \nonumber\\
&\quad \leq U_f  \int_{\mathcal{S}}\int_{\mathcal{A}} \left|\P\left(\pmb{s}^{(i,j)}=\dd \pmb{s}|\pmb{s}^{(i-t,j)},\pmb\theta_{i-t}\right) \pi_{\pmb\theta_i}(\pmb{a}|\pmb{s}) -\P\left(\tilde{\pmb{s}}^{(i,j)}=\dd \pmb{s}|\pmb{s}^{(i-t,j)},\pmb\theta_{i-t}\right)\pi_{\pmb\theta_{i-t}}(\pmb{a}|\pmb{s})\right|\nonumber \\
&\qquad\qquad\quad  \Vert g\left(\pmb{s},\pmb{a}|\pmb\theta_{i-t}\right) \Vert \dd\pmb{a} \nonumber\\
&\quad \leq 2MU_f\left\Vert \P\left((\pmb{s}^{(i,j)},\pmb{a}^{(i,j)})\in \cdot|\pmb{s}^{(i-t,j)},\pmb\theta_{i-t}\right)-\P\left((\tilde{\pmb{s}}^{(i,j)},\tilde{\pmb{a}}^{(i,j)})\in \cdot|\pmb{s}^{(i-t,j)},\pmb\theta_{i-t}\right)\right\Vert_{TV}\label{eq11: appendix lemma: supportive lemma 2}
\end{align}
where Step~\eqref{eq11: appendix lemma: supportive lemma 2} holds due to Lemma~\ref{lemma: bounded of policy gradient}. Let $$\Delta^{(i,j)}=2M\left\Vert \P\left((\pmb{s}^{(i,j)},\pmb{a}^{(i,j)})\in \cdot|\pmb{s}^{(i-t,j)},\pmb\theta_{i-t}\right)-\P\left((\tilde{\pmb{s}}^{(i,j)},\tilde{\pmb{a}}^{(i,j)})\in \cdot|\pmb{s}^{(i-t,j)},\pmb\theta_{i-t}\right)\right\Vert_{TV}.$$ By Lemma~\ref{lemma: relations of total variation measures}, we have 
\begin{align}
\Delta^{(i,j)}
&\leq 2M\left\Vert \P\left(\left.\pmb{s}^{(i,j)}\in \cdot\right|\pmb{s}^{(i-t,j)},\pmb\theta_{i-t}\right)-\P\left(\left.\tilde{\pmb{s}}^{(i,j)}\in \cdot\right|\pmb{s}^{(i-t,j)},\pmb\theta_{i-t}\right)\right\Vert_{TV} \nonumber\\
&\quad +2M U_\pi \E\left[\Vert \pmb\theta_{i}-\pmb\theta_{i-t}\Vert|\pmb\theta_{i-t}\right] \nonumber\\
&\leq\Delta^{(i,j-1)} + 2M U_\pi \E\left[\Vert \pmb\theta_{i}-\pmb\theta_{i-t}\Vert|\pmb\theta_{i-t}\right].\label{eq12: appendix lemma: supportive lemma 2}
\end{align}
Repeat the inequality~\eqref{eq12: appendix lemma: supportive lemma 2} from $(i,j)$ to $(i-t,j)$, we have
\begin{align}
    \Delta^{(i,j)}\leq \ldots &\leq \Delta^{(i-t,j)} + 2M U_\pi\sum^{i}_{i^\prime=i-t}\sum^{n}_{j=1}\E[\Vert \pmb\theta_{i^\prime}-\pmb\theta_{i-t}\Vert|\pmb\theta_{i-t}] \nonumber\\
    &= \Delta^{(i-t,j)} + 2n M U_\pi \sum^{i}_{i^\prime=i-t}\E\left[\Vert \pmb\theta_{i^\prime}-\pmb\theta_{i-t}\Vert|\pmb\theta_{i-t}\right]. \nonumber
\end{align}
Notice that $$\Delta^{(i-t,j)}=2M\left\Vert \P\left(\pmb{s}^{(i-t,j)}\in \cdot|\pmb{s}^{(i-t,j)}\right)\pi_{\pmb\theta_{i-t}}(\cdot|\pmb{s})-\P\left(\tilde{\pmb{s}}^{(i-t,j)}\in \cdot|\pmb{s}^{(i-t,j)}\right)\pi_{\pmb\theta_{i-t}}(\cdot|\pmb{s})\right\Vert_{TV}=0.$$ Thus we have  $\Delta^{(i,j)}\leq 2n M U_\pi \sum^{i}_{i^\prime=i-t}\E[\Vert \pmb\theta_{i^\prime}-\pmb\theta_{i-t}\Vert|\pmb\theta_{i-t}]$. 
By Lemma~\ref{auxillary lemma: parametric distance}, we have $\E[\Vert \pmb\theta_{i^\prime}-\pmb\theta_{i-t}\Vert|\pmb\theta_{i-t}] \leq M\sum^{i^\prime}_{\ell=i-t}{\eta_{\ell}}$ and thus it holds that
\begin{equation*}
\Delta^{(i,j)} \leq 2nU_\pi M^2\sum^{i}_{i^\prime=i-t}  \sum^{i^\prime}_{\ell=i-t}\eta_{\ell} \leq 2nU_\pi M^2t  \sum^{i}_{\ell=i-t}\eta_{\ell}.
\end{equation*}
Thus, the term \eqref{eq8: appendix lemma: supportive lemma 2} is bounded by
\begin{equation} \label{eq8: bound appendix lemma: supportive lemma 2}
\left\Vert \E\left[\left.U_f\left[g\left(\pmb{x}^{(i,j)}|\pmb\theta_{i-t}\right)- g\left(\tilde{\pmb{x}}^{(i,j)}|\pmb\theta_{i-t}\right)\right]\right|\pmb{s}^{(i-t,j)},\pmb\theta_{i-t}\right]\right\Vert \leq 2 U_f nU_\pi M^2t  \sum^{i}_{\ell=i-t}\eta_{\ell}.
\end{equation}

\noindent \textbf{Second}, for term \eqref{eq10: appendix lemma: supportive lemma 2}, based on the similar steps as Eq.~\eqref{eq7: appendix lemma: supportive lemma 2}, we have
\begin{align}
&\E\left[\left.\frac{\left|\pi_{\pmb\theta_{i-t}}\left(\pmb{a}^{(i,j)}|\pmb{s}^{(i,j)}\right)-\pi_{\pmb\theta_i}\left(\pmb{a}^{(i,j)}|\pmb{s}^{(i,j)}\right)\right|}{\pi_{\pmb\theta_i}\left(\pmb{a}^{(i,j)}|\pmb{s}^{(i,j)}\right)} \left\Vert g\left(\tilde{\pmb{x}}^{(i,j)}|\pmb\theta_{i-t}\right)\right\Vert\right|\pmb{s}^{(i-t,j)},\pmb\theta_{i-t}\right]\nonumber\\
&\leq 2M^2U_\pi\sum^i_{\ell=i-t}\eta_\ell. 
\label{eq10: bound appendix lemma: supportive lemma 2}
\end{align}
By putting \eqref{eq8: bound appendix lemma: supportive lemma 2} and \eqref{eq10: bound appendix lemma: supportive lemma 2} all together and taking full expectation, we have
\begin{equation}\label{eq14: appendix lemma: supportive lemma 2}
   \E[\Delta_{\textbf{(iii)}}] \leq 2nM^2U_\pi U_f t\sum^i_{\ell=i-t}\eta_\ell.
\end{equation}

\noindent \textbf{(iv):} Notice that for any $\pmb{s},\tilde{\pmb{s}}\in\mathcal{S}$ and $\pmb{a},\tilde{\pmb{a}}\in\mathcal{A}$, it holds
\begin{align}
    &\left\Vert\min \left(\frac{\pi_{\pmb\theta_{i-t}}\left(\pmb{a}|\pmb{s}\right)}{\pi_{\pmb\theta_{i}}\left(\pmb{a}^|\pmb{s}\right)},U_f\right)g\left(\pmb{x}|\pmb\theta_{i-t}\right)- g\left(\tilde{\pmb{x}}|\pmb\theta_{i-t}\right) \right\Vert  \nonumber\\
    &\leq  \left\Vert\min \left(\frac{\pi_{\pmb\theta_{i-t}}\left(\pmb{a}|\pmb{s}\right)}{\pi_{\pmb\theta_{i}}\left(\pmb{a}|\pmb{s}\right)},U_f\right)g\left(\pmb{x}|\pmb\theta_{i-t}\right)- \min \left(\frac{\pi_{\pmb\theta_{i-t}}\left(\pmb{a}|\pmb{s}\right)}{\pi_{\pmb\theta_{i}}\left(\pmb{a}|\pmb{s}\right)},U_f\right)g\left(\tilde{\pmb{x}}|\pmb\theta_{i-t}\right)\right\Vert \nonumber\\
    &+ \left \Vert \min \left(\frac{\pi_{\pmb\theta_{i-t}}\left(\pmb{a}|\pmb{s}\right)}{\pi_{\pmb\theta_{i}}\left(\pmb{a}|\pmb{s}\right)},U_f\right)g\left(\tilde{\pmb{x}}|\pmb\theta_{i-t}\right)-g\left(\tilde{\pmb{x}}|\pmb\theta_{i-t}\right) \right\Vert \nonumber\\
&\leq  \min \left(\frac{\pi_{\pmb\theta_{i-t}}\left(\pmb{a}|\pmb{s}\right)}{\pi_{\pmb\theta_{i}}\left(\pmb{a}|\pmb{s}\right)},U_f\right)\left\Vert g\left(\pmb{x}|\pmb\theta_{i-t}\right)- g\left(\tilde{\pmb{x}}|\pmb\theta_{i-t}\right)\right\Vert + \left | \min \left(\frac{\pi_{\pmb\theta_{i-t}}\left(\pmb{a}|\pmb{s}\right)}{\pi_{\pmb\theta_{i}}\left(\pmb{a}|\pmb{s}\right)},U_f\right)-1 \right| \Vert g\left(\tilde{\pmb{x}}|\pmb\theta_{i-t}\right)\Vert \nonumber\\
&\leq 2U_f M + (U_f+1)M \nonumber
\end{align}
from which the conclusion immediately follows.
\endproof

\begin{lemma}\label{appendix lemma: supportive lemma 3}
   Under Assumption~\ref{assumption 2} and \ref{assumption 3}, for any $t$ such that $t <i\leq k$, we have
   \begin{itemize}
       \item[\textbf{(i)}] $\left\Vert \E\left[\widehat{\nabla} J\left(\tilde{\pmb{x}}^{(i,j)}, \pmb\theta_i,\pmb\theta_{i-t}\right)-\widehat{\nabla} J\left(\check{\pmb{x}}^{(i,j)}, \pmb\theta_i,\pmb\theta_{i-t}\right)\right] \right\Vert\leq 2M \varphi(nt)$
       \item[\textbf{(ii)}] $ \E\left[\left\Vert\widehat{\nabla} J\left(\tilde{\pmb{x}}^{(i,j)}, \pmb\theta_i,\pmb\theta_{i-t}\right)-\widehat{\nabla} J\left(\check{\pmb{x}}^{(i,j)}, \pmb\theta_i,\pmb\theta_{i-t}\right)\right\Vert\right] \leq 2M.$
   \end{itemize}
\end{lemma}
\proof
\noindent \textbf{(i)} Let $\Delta = \left\Vert\E\left[\left.\widehat{\nabla} J\left(\tilde{\pmb{x}}^{(i,j)}, \pmb\theta_i,\pmb\theta_{i-t}\right)-\widehat{\nabla} J\left(\check{\pmb{x}}^{(i,j)}, \pmb\theta_i,\pmb\theta_{i-t}\right)\right|\pmb{s}^{(i-t,j)},\pmb\theta_{i-t}\right]\right\Vert$. It holds
\begin{align}
\Delta&=
\left\Vert \E\left[\left. g\left(\tilde{\pmb{s}}^{(i,j)}, \tilde{\pmb{a}}^{(i,j)}|\pmb\theta_{i-t}\right) - g\left(\check{\pmb{s}}^{(i,j)}, \check{\pmb{a}}^{(i,j)}|\pmb\theta_{i-t}\right) \right|\pmb{s}^{(i-t,j)},\pmb\theta_{i-t}\right]\right\Vert \nonumber\\
&=\left\Vert \int \P\left(\tilde{\pmb{s}}^{(i,j)}=\dd \pmb{s}|\pmb{s}^{(i-t,j)},\pmb\theta_{i-t}\right)\pi_{\pmb\theta_{i-t}}\left(\pmb{a}|\pmb{s}\right)g\left(\pmb{s}, \pmb{a}|\pmb\theta_{i-t}\right) \dd \pmb{a} \right. \nonumber\\
& \qquad - \left. \int d^{\pi_{\pmb\theta_{i-t}}}(\pmb{s})\pi_{\pmb\theta_{i-t}}\left(\pmb{a}|\pmb{s}\right)g\left(\pmb{s}, \pmb{a}|\pmb\theta_{i-t}\right) \dd \pmb{s} \dd\pmb{a} \right\Vert \nonumber\\
&=\left\Vert \int \left(p\left(\tilde{\pmb{s}}^{(i,j)}= \pmb{s}|\pmb{s}^{(i-t,j)},\pmb\theta_{i-t}\right)-d^{\pi_{\pmb\theta_{i-t}}}(\pmb{s})\right)\pi_{\pmb\theta_{i-t}}\left(\pmb{a}|\pmb{s}\right)g\left(\pmb{s}, \pmb{a}|\pmb\theta_{i-t}\right) \dd \pmb{s} \dd\pmb{a} \right\Vert \nonumber\\
&\leq \int \left|p\left(\tilde{\pmb{s}}^{(i,j)}= \pmb{s}|\pmb{s}^{(i-t,j)},\pmb\theta_{i-t}\right)-d^{\pi_{\pmb\theta_{i-t}}}(\pmb{s})\right|\pi_{\pmb\theta_{i-t}}\left(\pmb{a}|\pmb{s}\right) \left\Vert 
 g\left(\pmb{s}, \pmb{a}|\pmb\theta_{i-t}\right)\right\Vert \dd \pmb{s} \dd\pmb{a} \nonumber\\
 &\leq M\int \left|p\left(\tilde{\pmb{s}}^{(i,j)}= \pmb{s}|\pmb{s}^{(i-t,j)},\pmb\theta_{i-t}\right)-d^{\pi_{\pmb\theta_{i-t}}}(\pmb{s})\right|\pi_{\pmb\theta_{i-t}}\left(\pmb{a}|\pmb{s}\right)  \dd \pmb{s} \dd\pmb{a} \label{eq1: appendix lemma: supportive lemma 3}\\
 &= M\int \left|p\left(\tilde{\pmb{s}}^{(i,j)}= \pmb{s}|\pmb{s}^{(i-t,j)},\pmb\theta_{i-t}\right)-d^{\pi_{\pmb\theta_{i-t}}}(\pmb{s})\right| \left(\int \pi_{\pmb\theta_{i-t}}\left(\pmb{a}|\pmb{s}\right) \dd\pmb{a} \right) \dd \pmb{s}\nonumber\\
 &= 2M \left\Vert \P\left(\tilde{\pmb{s}}^{(i,j)}\in \cdot|\pmb{s}^{(i-t,j)},\pmb\theta_{i-t}\right)-d^{\pi_{\pmb\theta_{i-t}}}(\cdot)\right\Vert_{TV}\nonumber\\
 &\leq 2M \varphi(nt). \label{eq2: appendix lemma: supportive lemma 3}
\end{align}
Then the conclusion follows from the fact that 
$$\left\Vert \E\left[\widehat{\nabla} J\left(\tilde{\pmb{x}}^{(i,j)}, \pmb\theta_i,\pmb\theta_{i-t}\right)-\widehat{\nabla} J\left(\check{\pmb{x}}^{(i,j)}, \pmb\theta_i,\pmb\theta_{i-t}\right)\right] \right\Vert
\leq 2M\varphi(nt).$$

\noindent \textbf{(ii)} The conclusion can be obtained by observing that $$\E\left[\left\Vert g\left(\tilde{\pmb{s}}^{(i,j)}, \tilde{\pmb{a}}^{(i,j)}|\pmb\theta_{i-t}\right) - g\left(\check{\pmb{s}}^{(i,j)}, \check{\pmb{a}}^{(i,j)}|\pmb\theta_{i-t}\right)\right\Vert\right]\leq \E[2M]=2M,$$
which concludes the proof.
\endproof

\begin{lemma}\label{appendix lemma: supportive lemma 4}
   Under Assumptions~\ref{assumption 2} and \ref{assumption 3}, for any $t$ such that $t <i\leq k$, we have
   \begin{equation*}
       \Vert\E[\Delta]\Vert\leq \E\left[\Vert\Delta\Vert\right]\leq L_g M \sum^k_{\ell=i-t}\eta_\ell
   \end{equation*}
   where $\Delta=\widehat{\nabla} J\left(\check{\pmb{x}}^{(i,j)}, \pmb\theta_i,\pmb\theta_{i-t}\right)-\widehat{\nabla} J\left(\check{\pmb{x}}^{(i,j)}, \pmb\theta_i,\pmb\theta_{k}\right)$.
\end{lemma}
\proof
By Lemma~\ref{lemma: lipchitz continuity of sample gradient estimate},
\begin{equation*}
\E[\Vert\Delta \Vert] = \E\left[\Vert g(\check{\pmb{x}}^{(i,j)}|\pmb\theta_{i-t})-g(\check{\pmb{x}}^{(i,j)}|\pmb\theta_{k})\Vert \right] \leq L_g\E[\Vert\pmb\theta_k-\pmb\theta_{i-t}\Vert].
\end{equation*}
The conclusion follows by applying Lemma~\ref{auxillary lemma: parametric distance}.
\endproof

\begin{lemma}\label{appendix lemma: supportive lemma 5}
   Under Assumption~\ref{assumption 2} and \ref{assumption 3}, for any $t$ such that $t <i\leq k$, we have
   \begin{itemize}
     \item[\textbf{(i)}] $\left\Vert \E\left[\widehat{\nabla} J\left(\check{\pmb{x}}^{(i,j)}, \pmb\theta_i,\pmb\theta_{k}\right)-\nabla  J\left(\pmb\theta_k\right)\right] \right\Vert \leq 2M^2 C_d\sum^{k}_{\ell=i-t}\eta_{\ell}$
     \item[\textbf{(ii)}] $\E\left[\left\Vert \widehat{\nabla} J\left(\check{\pmb{x}}^{(i,j)}, \pmb\theta_i,\pmb\theta_{k}\right)-\nabla  J\left(\pmb\theta_k\right)\right\Vert\right]  \leq 2M$
   \end{itemize}
\end{lemma}
\proof
\noindent \textbf{(i)}:
Let $\Delta = \E\left[\left.\widehat{\nabla} J\left(\check{\pmb{x}}^{(i,j)}, \pmb\theta_i,\pmb\theta_{k}\right)-\nabla  J\left(\pmb\theta_k\right)\right|\pmb{s}^{i-t,j},\pmb\theta_{i-t}\right]$. For notational simplicity, we omit the conditions on $(\pmb{s}^{i-t,j},\pmb\theta_{i-t})$ below. Then we have
\begin{align}
\Vert\Delta \Vert&= \left\Vert\E\left[\E_{(\pmb{s},\pmb{a})\sim d^{\pi_{\pmb\theta_{i-t}}}(\cdot,\cdot)}[g(\pmb{s},\pmb{a}|\pmb\theta_k)]\right] - \E\left[\E_{(\pmb{s},\pmb{a})\sim d^{\pi_{\pmb\theta_{k}}}(\cdot,\cdot)}[g(\pmb{s},\pmb{a}|\pmb\theta_k)]\right] \right\Vert\nonumber\\
&\leq \E\left[\left\Vert \int d^{\pi_{\pmb\theta_{i-t}}}(\pmb{s},\pmb{a}) g(\pmb{s}, \pmb{a}|\pmb\theta_k)\dd \pmb{s}\dd \pmb{a} - \int d^{\pi_{\pmb\theta_{k}}}(\pmb{s},\pmb{a}) g(\pmb{s}, \pmb{a}|\pmb\theta_k)\dd \pmb{s}\dd \pmb{a}\right\Vert\right] \nonumber\\
&\leq \E\left[ \int | d^{\pi_{\pmb\theta_{i-t}}}(\pmb{s},\pmb{a}) - d^{\pi_{\pmb\theta_{k}}}(\pmb{s},\pmb{a}) |\left\Vert g(\pmb{s}, \pmb{a}|\pmb\theta_k) \right\Vert\dd \pmb{s}\dd \pmb{a}\right] \nonumber \\
&\leq 2M \E\left[\left\Vert d^{\pi_{\pmb\theta_{i-t}}}(\pmb{s},\pmb{a}) - d^{\pi_{\pmb\theta_{k}}}(\pmb{s},\pmb{a}) \right\Vert_{TV}\right] \label{eq1: appendix lemma: supportive lemma 5} \\
&\leq 2M C_d \E\left[\left\Vert \pmb\theta_{k}-\pmb\theta_{i-t}\right\Vert \right] \label{eq2: appendix lemma: supportive lemma 5}
\end{align}
where $C_d=U_\pi(1+\lceil\log_{\kappa}\kappa_0^{-1}\rceil+\frac{1}{1-\kappa})$ with $\kappa_0=1$ and $\kappa=\gamma\in(0,1)$. Step~\eqref{eq1: appendix lemma: supportive lemma 5} holds due to Lemma~\ref{lemma: bounded of policy gradient} and Step~\eqref{eq2: appendix lemma: supportive lemma 5} is by Lemma~\ref{lemma: lipchitz continuity of stationary distribution}. The conclusion is obtained by applying Lemma~\ref{auxillary lemma: parametric distance}, i.e.
\begin{equation*}
  \left\Vert \E\left[\widehat{\nabla} J\left(\check{\pmb{x}}^{(i,j)}, \pmb\theta_i,\pmb\theta_{k}\right)-\nabla  J\left(\pmb\theta_k\right)\right] \right\Vert=\Vert\E[\Delta]\Vert\leq \E[\Vert\Delta\Vert]\leq 2M^2 C_d\sum^{k}_{\ell=i-t}\eta_{\ell}.
\end{equation*}

\noindent \textbf{(ii)}: The conclusion is obtained by Minkowski's inequality and Lemma~\ref{lemma: bounded of policy gradient}:
\begin{align*}
    &\E\left[\left\Vert \widehat{\nabla} J\left(\check{\pmb{x}}^{(i,j)}, \pmb\theta_i,\pmb\theta_{k}\right)-\nabla  J\left(\pmb\theta_k\right)\right\Vert\right] \\
    &=\E\left[\left\Vert g(\check{\pmb{x}}^{(i,j)}|\pmb\theta_k)-g(\pmb{x}|\pmb\theta_k)\right\Vert\right] \leq \E\left[\left\Vert g(\check{\pmb{x}}^{(i,j)}|\pmb\theta_k)\right\Vert +\left \Vert g(\pmb{x}|\pmb\theta_k)\right\Vert\right] \leq 2M
\end{align*}
which completes the proof.
\endproof

\begin{lemma}\label{appendix lemma: supportive lemma 6}
  For a learning rate $\eta_k=\eta_1 k^{-r}$ with $\eta_1 \leq \frac{1}{4L}$ and $r\in(0,1)$, under Assumption~\ref{assumption 2} and \ref{assumption 3}, by following the selection rule defined in Theorem~\ref{thm:Online selection rule}, we have
    \begin{equation*}
        \left|\frac{1}{n}\sum^n_{j=1}\E\left[\Gamma\left(\pmb{x}^{(i,j)},\pmb\theta_i, \pmb\theta_k\right)-\Gamma\left(\pmb{x}^{(i,j)},\pmb\theta_i,\pmb\theta_{i-t}\right) \right]\right|\leq  C^\Gamma_1 (k-i+t)\eta_{i-t} + C_2^\Gamma (k-i+t)\eta_{i-t}
    \end{equation*}
where $C_1^\Gamma = {LM^2\sqrt{2c+1} (\sqrt{2c+1}+2)}$ and $C_2^\Gamma = M^2(U_f L_g+M U_\pi)$.
\end{lemma}
\proof
Let $\Delta=\frac{1}{n}\sum^n_{j=1}\left(\Gamma\left(\pmb{x}^{(i,j)},\pmb\theta_i, \pmb\theta_k\right)-\Gamma\left(\pmb{x}^{(i,j)},\pmb\theta_i,\pmb\theta_{i-t}\right)\right)$. 
Recall that $$\widehat{\nabla} J^{R}_{i,k}=\frac{1}{n}\sum^n_{j=1}\widehat{\nabla} J^{R}(\pmb{x}^{(i,j)},\pmb\theta_i,\pmb\theta_k)\text{  and  } \widehat{\nabla} J^{R}_{i,i-t}=\frac{1}{n}\sum^n_{j=1}\widehat{\nabla} J^{R}(\pmb{x}^{(i,j)},\pmb\theta_i,\pmb\theta_{i-t})$$
where $R\in\{LR,CLR\}$. Then we have
\begin{align}
    \E[\Delta] &= \E\left[\left\langle \nabla J(\pmb\theta_k),\frac{1}{n}\sum^n_{j=1}\widehat{\nabla} J^{R}(\pmb{x}^{(i,j)},\pmb\theta_i,\pmb\theta_k)- \nabla J(\pmb\theta_k)\right\rangle\right.\nonumber\\
    &\qquad -\left. \left\langle \nabla J(\pmb\theta_{i-t}),\frac{1}{n}\sum^n_{j=1}\widehat{\nabla} J^{R}(\pmb{x}^{(i,j)},\pmb\theta_i,\pmb\theta_{i-t})-\nabla J(\pmb\theta_{i-t})\right\rangle\right] \nonumber\\
&=\Delta_1+\Delta_2 +\Delta_3
\end{align}
where 
\begin{align}
\Delta_1 &= \E\left[\left\langle \nabla J(\pmb\theta_k),\widehat{\nabla} J^{R}_{i,k}-\nabla J(\pmb\theta_k)\right\rangle \right. 
    -\left.\left\langle \nabla J(\pmb\theta_{i-t}),\widehat{\nabla} J^{R}_{i,k}-\nabla J(\pmb\theta_{k})\right\rangle\right]; \nonumber\\
\Delta_2 &=\E\left[\left\langle \nabla J(\pmb\theta_{i-t}),\widehat{\nabla} J^{R}_{i,k}-\nabla J(\pmb\theta_k)\right\rangle \right. 
 -\left.\left\langle \nabla J(\pmb\theta_{i-t}),\widehat{\nabla} J^{R}_{i,k}-\nabla J(\pmb\theta_{i-t})\right\rangle\right]; \nonumber\\
\Delta_3 &=\E\left[\left\langle \nabla J(\pmb\theta_{i-t}),\widehat{\nabla} J^{R}_{i,k}-\nabla J(\pmb\theta_{i-t})\right\rangle \right. 
    -\left.\left\langle \nabla J(\pmb\theta_{i-t}),\widehat{\nabla} J^{R}_{i,i-t}-\nabla J(\pmb\theta_{i-t})\right\rangle\right]. \nonumber
\end{align}
By Minkowski's inequality, the norm $|\E[\Delta]|\leq |\Delta_1|+|\Delta_2|+|\Delta_3|$.

\noindent\textbf{(i)} For the term $|\Delta_1|$: 
\begin{align}
| \Delta_1 | &= \left| \E\left[\left\langle \nabla J(\pmb\theta_k),\widehat{\nabla} J^{R}_{i,k}-\nabla J(\pmb\theta_k)\right\rangle \right.\right. 
    -\left.\left.\left\langle \nabla J(\pmb\theta_{i-t}),\widehat{\nabla} J^{R}_{i,k}-\nabla J(\pmb\theta_{k})\right\rangle\right] \right| 
    \nonumber\\
&= \left| \E\left[\left\langle \nabla J(\pmb\theta_k)-\nabla J(\pmb\theta_{i-t}),\widehat{\nabla} J^{R}_{i,k}-\nabla J(\pmb\theta_k)\right\rangle\right] \right| \nonumber \\
&\leq  \E\left[\left\Vert\nabla J(\pmb\theta_k)-\nabla J(\pmb\theta_{i-t})\right\Vert^2 \right]^{1/2} \E\left[\left\Vert\widehat{\nabla} J^{R}_{i,k}-\nabla J(\pmb\theta_k)\right\Vert^2 \right]^{1/2} \label{eq1: appendix lemma: supportive lemma 6} \\
&\leq  L \E[\Vert \pmb\theta_{k}-\pmb\theta_{i-t}\Vert^2]^{1/2}\E\left[\left\Vert\widehat{\nabla} J^{R}_{i,k}-\nabla J(\pmb\theta_k)\right\Vert^2 \right]^{1/2},
\label{eq2: appendix lemma: supportive lemma 6}
\end{align}
where Step~\eqref{eq1: appendix lemma: supportive lemma 6} follows by Cauchy–Schwarz inequality and Step~\eqref{eq2: appendix lemma: supportive lemma 6} holds due to Lemma~\ref{lemma: Lipschitz continuity}. For the first term of \eqref{eq2: appendix lemma: supportive lemma 6}, applying  Minkowski's inequality and Lemma~\ref{lemma: bounded of policy gradient} gives
\begin{equation}\label{eq-0: appendix lemma: supportive lemma 6}
{\begin{aligned}
\E\Vert\pmb\theta_k-\pmb\theta_{i-t}\Vert^2
&=\E\left\Vert\sum_{\ell=i-t}^{k-1}\eta_\ell\widehat{\nabla}J^R_\ell\right\Vert^2\\
&=\sum_{\ell=i-t}^{k-1}\eta_\ell^2\E\Vert\widehat{\nabla}J^R_\ell\Vert^2
+2\sum_{i-t\leq\ell<\ell'<k}\eta_\ell\eta_{\ell'}
\E\langle\widehat{\nabla}J^R_\ell,\widehat{\nabla}J^R_{\ell'}\rangle\\
&\leq\left(\sum_{\ell=i-t}^{k-1}\eta_\ell\right)
\left(\sum_{\ell=i-t}^{k-1}\eta_\ell\E\Vert\widehat{\nabla}J^R_\ell\Vert^2\right).
\end{aligned}}
\end{equation}
Theorem~\ref{thm:Online selection rule} implies that
 \begin{equation}\label{eq-1: appendix lemma: supportive lemma 6}
    \E\left[\left\Vert\widehat{\nabla} J^{R}_\ell\right\Vert^2\right]=\Tr\left(\Var\left[\widehat{\nabla} J^{R}_\ell\right]\right)+ \left\Vert\E\left[\widehat{\nabla} J^{R}_\ell\right]\right\Vert^2\leq c\Tr\left(\Var\left[\widehat{\nabla} J^{PG}_\ell\right]\right)+ \left\Vert\E\left[\widehat{\nabla} J^{R}_\ell\right]\right\Vert^2.
\end{equation}
Also by Minkowski's inequality, it holds for any $\ell$ that
\begin{align}
  \left\Vert\E\left[\widehat{\nabla} J^{R}_\ell\right]\right\Vert  &\leq\E\left[\frac{1}{|\mathcal{U}_\ell|n}
 \sum_{{\pmb\theta_h}\in \mathcal{U}_\ell}
 \sum^{n}_{j=1}
 \frac{\pi_{\pmb\theta_\ell}\left(\pmb{a}^{(h,j)}|\pmb{s}^{(h,j)}\right)}
 {\pi_{\pmb\theta_h}\left(\pmb{a}^{(h,j)}|\pmb{s}^{(h,j)}\right)}
\left\Vert g\left(\pmb{x}^{(h,j)}|\pmb\theta_{\ell}\right)\right\Vert \right] \nonumber\\
&\leq  \frac{M}{|\mathcal{U}_\ell|n}
 \sum_{{\pmb\theta_h}\in \mathcal{U}_\ell}
 \sum^{n}_{j=1}\E\left[
 \frac{\pi_{\pmb\theta_\ell}\left(\pmb{a}^{(h,j)}|\pmb{s}^{(h,j)}\right)}
 {\pi_{\pmb\theta_h}\left(\pmb{a}^{(h,j)}|\pmb{s}^{(h,j)}\right)}
 \right] \nonumber\\
& \leq M.
\label{eq-2: appendix lemma: supportive lemma 6}
\end{align}
By Lemma~\ref{lemma: bounded of policy gradient}, it follows that \begin{equation}\label{eq-3: appendix lemma: supportive lemma 6}
    \Tr\left(\Var\left[\widehat{\nabla} J^{PG}_\ell\right]\right)=\E[\Vert g(\pmb{x};\pmb\theta_\ell)\Vert^2]-\Vert\E[ g(\pmb{x};\pmb\theta_\ell)]\Vert^2\leq 2M^2.
\end{equation}
By collecting \eqref{eq-1: appendix lemma: supportive lemma 6}, \eqref{eq-2: appendix lemma: supportive lemma 6} and \eqref{eq-3: appendix lemma: supportive lemma 6}, Eq.~\eqref{eq-0: appendix lemma: supportive lemma 6} becomes 
\begin{equation*}
    \E\Vert\pmb\theta_k-\pmb\theta_{i-t}\Vert^2\leq(2c+1)M^2\left(\sum_{\ell=i-t}^{k-1}\eta_\ell\right)^2\leq(2c+1)M^2(k-i+t)^2\eta_{i-t}^2,
\end{equation*}
where the last inequality holds because $\{\eta_k\}$ is a non-increasing sequence. 
Then the first term of Step~\eqref{eq2: appendix lemma: supportive lemma 6} becomes
\begin{equation}\label{eq-4: appendix lemma: supportive lemma 6}
   \E\left[\Vert\pmb\theta_k-\pmb\theta_{i-t}\Vert^2\right]^{1/2}\leq\sqrt{2c+1}M(k-i+t)\eta_{i-t}. 
\end{equation}

\noindent For the second term of Step~\eqref{eq2: appendix lemma: supportive lemma 6}, applying Minkowski's inequality and Lemma~\ref{lemma: bounded of policy gradient} gives
\begin{align*}
   \E\left[\left\Vert\widehat{\nabla} J^{R}_{i,k}-\nabla J(\pmb\theta_k)\right\Vert^2 \right]^{1/2} &\leq \E\left[\left\Vert\widehat{\nabla} J^{R}_{i,k}\right\Vert^2\right]^{1/2}+\E\left[\left\Vert\nabla J(\pmb\theta_k)\right\Vert^2 \right]^{1/2} \nonumber\\
    &\leq \E\left[\left\Vert\widehat{\nabla} J^{R}_{i,k}\right\Vert^2\right]^{1/2}+M.
\end{align*}
By applying \textbf{Selection Rule 1} defined in Theorem~\ref{thm:Online selection rule}, we have
\begin{equation}
   \E\left[\left\Vert\widehat{\nabla} J^{R}_{i,k}\right\Vert^2\right] 
   {=}
   \Tr\left(\Var\left[\widehat{\nabla} J^{R}_{i,k}\right]\right)+\left\Vert\E\left[\widehat{\nabla} J^{R}_{i,k}\right]\right\Vert^2 \leq c\Tr\left(\Var\left[\widehat{\nabla} J^{PG}_k\right]\right)+\left\Vert\E\left[\widehat{\nabla} J^{R}_{i,k}\right]\right\Vert^2. \label{eq-5: appendix lemma: supportive lemma 6}\\
\end{equation}
By substituting \eqref{eq-2: appendix lemma: supportive lemma 6} and \eqref{eq-3: appendix lemma: supportive lemma 6} into \eqref{eq-5: appendix lemma: supportive lemma 6}, we have $\E\left[\left\Vert\widehat{\nabla} J^{R}_{i,k}\right\Vert^2\right] \leq (2c+1)M^2$ and consequently, the second term of Step~\eqref{eq2: appendix lemma: supportive lemma 6} becomes
\begin{equation*}
    \E\left[\left\Vert\widehat{\nabla} J^{R}_{i,k}-\nabla J(\pmb\theta_k)\right\Vert^2 \right]^{1/2} \leq M(\sqrt{2c+1} +1).
\end{equation*}
Then we have
\begin{equation*}
    | \Delta_1 | \leq L(2c+1+\sqrt{2c+1}) M^2 (k-i+t)^{1/2} \eta_{i-t}.
\end{equation*}

\noindent\textbf{(ii)} For the term $|\Delta_2|$: 
\begin{align}
| \Delta_2| &=\left|\E\left[\left\langle \nabla J(\pmb\theta_{i-t}),\widehat{\nabla} J^{R}(\pmb{x}^{(i,j)},\pmb\theta_i,\pmb\theta_k)-\nabla J(\pmb\theta_k)\right\rangle \right. \right.\nonumber\\
 &\qquad -\left.\left.\left\langle \nabla J(\pmb\theta_{i-t}),\widehat{\nabla} J^{R}(\pmb{x}^{(i,j)},\pmb\theta_i,\pmb\theta_{k})-\nabla J(\pmb\theta_{i-t})\right\rangle\right]\right| \nonumber\\
&= \left| \E\left[\left\langle \nabla J(\pmb\theta_{i-t}),\nabla J(\pmb\theta_k)-\nabla J(\pmb\theta_{i-t})\right\rangle\right] \right| \nonumber \\
&\leq \E\left[\left\Vert \nabla J(\pmb\theta_{i-t})\right\Vert^2\right]^{1/2}\E\left[\left\Vert\nabla J(\pmb\theta_k)-\nabla J(\pmb\theta_{i-t})\right\Vert^2\right]^{1/2} \nonumber\\
&\leq LM\E\left[\Vert\pmb\theta_k-\pmb\theta_{i-t}\Vert^2\right]^{1/2} \label{eq.1: appendix lemma: supportive lemma 6}\\
&\leq L\sqrt{2c+1} M^2 (k-i+t)^{1/2} \eta_{i-t}, \nonumber
\end{align}
where the first inequality follows by applying Cauchy–Schwarz inequality and Lemma~\ref{lemma: Lipschitz continuity} and Step~\eqref{eq.1: appendix lemma: supportive lemma 6} is due to Lemma~\ref{lemma: bounded of policy gradient}. The last step is held by using the bound~\eqref{eq-4: appendix lemma: supportive lemma 6}.

\noindent\textbf{(iii)} For the term $|\Delta_3|$:
\begin{align}
|\Delta_3 | &\leq \left|\E\left[\E\left[\left\langle \nabla J(\pmb\theta_{i-t}),\widehat{\nabla} J^{R}(\pmb{x}^{(i,j)},\pmb\theta_i,\pmb\theta_k)-\nabla J(\pmb\theta_{i-t})\right\rangle \right.\right.\right. \nonumber\\
    &\qquad -\left.\left.\left.\left.\left\langle \nabla J(\pmb\theta_{i-t}),\widehat{\nabla} J^{R}(\pmb{x}^{(i,j)},\pmb\theta_i,\pmb\theta_{i-t})-\nabla J(\pmb\theta_{i-t})\right\rangle \right| \pmb{s}^{(i-t,j)},\pmb\theta_{i-t}\right] \right]\right| \nonumber\\
& \leq \left.\E\left[\left| \left\langle \nabla J(\pmb\theta_{i-t}),\E\left[\widehat{\nabla} J^{R}(\pmb{x}^{(i,j)},\pmb\theta_i,\pmb\theta_{k})-\widehat{\nabla} J^{R}(\pmb{x}^{(i,j)},\pmb\theta_i,\pmb\theta_{i-t}) \right| \pmb{s}^{(i-t,j)},\pmb\theta_{i-t} \right]\right\rangle\right| \right]\label{eq3-1: appendix lemma: supportive lemma 6}\\
& \leq \E\left[\left\Vert \nabla J(\pmb\theta_{i-t})\right\Vert \left\Vert\E\left[\left.\widehat{\nabla} J^{R}(\pmb{x}^{(i,j)},\pmb\theta_i,\pmb\theta_{k})-\widehat{\nabla} J^{R}(\pmb{x}^{(i,j)},\pmb\theta_i,\pmb\theta_{i-t}) \right| \pmb{s}^{(i-t,j)},\pmb\theta_{i-t} \right] \right\Vert \right]\label{eq3: appendix lemma: supportive lemma 6}\\
&\leq M \E\left[\left.\left\Vert \E\left[ \widehat{\nabla} J^{R}(\pmb{x}^{(i,j)},\pmb\theta_i,\pmb\theta_{k})-\widehat{\nabla} J^{R}(\pmb{x}^{(i,j)},\pmb\theta_i,\pmb\theta_{i-t})  \right| \pmb{s}^{(i-t,j)},\pmb\theta_{i-t}\right] \right\Vert\right] \label{eq4: appendix lemma: supportive lemma 6}\\
&\leq M \E\left[\left\Vert \widehat{\nabla} J^{R}(\pmb{x}^{(i,j)},\pmb\theta_i,\pmb\theta_{k})-\widehat{\nabla} J^{R}(\pmb{x}^{(i,j)},\pmb\theta_i,\pmb\theta_{i-t}) \right\Vert\right] \label{eq5: appendix lemma: supportive lemma 6}
\end{align}
where Step~\eqref{eq3-1: appendix lemma: supportive lemma 6} and \eqref{eq5: appendix lemma: supportive lemma 6} both follows Jensen's inequality, Step~\eqref{eq3: appendix lemma: supportive lemma 6} follows Cauchy-Schwarz inequality and Step~\eqref{eq4: appendix lemma: supportive lemma 6} holds by applying Lemma~\ref{lemma: bounded of policy gradient}.

By applying Lemma~\ref{appendix lemma: supportive lemma 1}, we have
\begin{equation}
    | \Delta_3 | \leq\begin{cases}
M^2 (L_g+M U_\pi) \sum^k_{\ell=i-t}\eta_\ell & \text{R $=$ LR}; \\
M^2 (U_f L_g+M U_\pi) \sum^k_{\ell=i-t}\eta_\ell & \text{R $=$ CLR}.
\end{cases}
\end{equation}
Because $U_f\geq 1$ and $\{\eta_k\}$ is a non-increasing sequence, it holds that $\sum^k_{\ell=i-t}\eta_\ell \leq (k-i+t)\eta_{i-t}$ and thus we have $$| \Delta_3 |\leq M^2(U_f L_g+M U_\pi) \sum^k_{\ell=i-t}\eta_\ell\leq M^2(U_f L_g+M U_\pi) (k-i+t)\eta_{i-t}.$$
The conclusion is obtained by summing $|\Delta_1|$, $|\Delta_2|$ and $|\Delta_3|$.
\endproof

\begin{lemma}\label{appendix lemma: supportive lemma 7}
   Under Assumptions~\ref{assumption 2} and \ref{assumption 3}, we have
\begin{align*}
&\left|\E\left[\Gamma\left(\pmb{x}^{(i,j)},\pmb\theta_i, \pmb\theta_{i-t}\right)-\Gamma\left(\tilde{\pmb{x}}^{(i,j)},\pmb\theta_i,\pmb\theta_{i-t}\right) \right]\right|\nonumber\\
&\leq \begin{cases}
 2 nM^3 U_\pi t  \sum^{i}_{\ell=i-t}\eta_{\ell} &\text{LR estimator};\\
 2nM^3U_\pi U_f t\sum^i_{\ell=i-t}\eta_\ell &\text{CLR estimator.}
\end{cases}
\end{align*}
\end{lemma}
\proof
Let $\Delta =\Gamma\left(\pmb{x}^{(i,j)},\pmb\theta_i, \pmb\theta_{i-t}\right)-\Gamma\left(\tilde{\pmb{x}}^{(i,j)},\pmb\theta_i,\pmb\theta_{i-t}\right)$. For both LR and CLR policy gradient estimators, i.e. $R\in\{LR,CLR\}$, we have
\begin{align}
    |\E[\Delta]|&= \left|\E\left[\E\left[\left\langle \nabla J(\pmb\theta_{i-t}),\widehat{\nabla} J^{R}(\pmb{x}^{(i,j)},\pmb\theta_i,\pmb\theta_{i-t})- \nabla J(\pmb\theta_{i-t})\right\rangle\right.\right.\right.\nonumber\\
    &\qquad -\left.\left. \left.\left. \left\langle \nabla J(\pmb\theta_{i-t}),\widehat{\nabla} J(\tilde{\pmb{x}}^{(i,j)},\pmb\theta_i,\pmb\theta_{i-t})-\nabla J(\pmb\theta_{i-t})\right\rangle\right| \pmb{s}^{(i-t,j)},\pmb\theta_{i-t}\right]\right]\right| \nonumber\\
&\leq \E\left[\left|\left\langle \nabla J(\pmb\theta_{i-t}),\E\left[\left.\widehat{\nabla} J^{R}\left(\pmb{x}^{(i,j)},\pmb\theta_i,\pmb\theta_{i-t}\right)- \widehat{\nabla} J\left(\tilde{\pmb{x}}^{(i,j)},\pmb\theta_i,\pmb\theta_{i-t}\right)\right| \pmb{s}^{(i-t,j)},\pmb\theta_{i-t}\right]\right\rangle \right|\right]\nonumber\\
&\leq M\E\left[\left\Vert \E\left[\left.\widehat{\nabla} J^{R}\left(\pmb{x}^{(i,j)},\pmb\theta_i,\pmb\theta_{i-t}\right)- \widehat{\nabla} J\left(\tilde{\pmb{x}}^{(i,j)},\pmb\theta_i,\pmb\theta_{i-t}\right)\right| \pmb{s}^{(i-t,j)},\pmb\theta_{i-t}\right]\right\Vert\right] \label{eq1: appendix lemma: supportive lemma 7}
\end{align}
where the first inequality holds by applying Jensen's inequality and Step~\eqref{eq1: appendix lemma: supportive lemma 7} is by Cauchy–Schwarz inequality and Lemma~\ref{lemma: bounded of policy gradient}.
Then the conclusion follows by using Eq.~\eqref{eq4: appendix lemma: supportive lemma 2} and Eq.~\eqref{eq14: appendix lemma: supportive lemma 2} from Lemma~\ref{appendix lemma: supportive lemma 2}.
\endproof

\begin{lemma}\label{appendix lemma: supportive lemma 8}
   Under Assumption~\ref{assumption 2} and \ref{assumption 3}, we have
    $$\left|\E\left[\Gamma\left(\tilde{\pmb{x}}^{(i,j)},\pmb\theta_i, \pmb\theta_{i-t}\right)-\Gamma\left(\check{\pmb{x}}^{(i,j)},\pmb\theta_i,\pmb\theta_{i-t}\right)\right]\right|\leq 2M^2 \varphi(nt).$$
\end{lemma}
\proof
Let $\Delta=\Gamma\left(\tilde{\pmb{x}}^{(i,j)},\pmb\theta_i, \pmb\theta_{i-t}\right)-\Gamma\left(\check{\pmb{x}}^{(i,j)},\pmb\theta_i,\pmb\theta_{i-t}\right)$. We have
\begin{align}
    |\E[\Delta] |&= \left|\E\left[\E\left[\left\langle \nabla J(\pmb\theta_{i-t}),\widehat{\nabla} J(\tilde{\pmb{x}}^{(i,j)},\pmb\theta_i,\pmb\theta_{i-t})- \nabla J(\pmb\theta_{i-t})\right\rangle\right.\right.\right.\nonumber\\
    &\qquad -\left. \left.\left.\left.\left\langle \nabla J(\pmb\theta_{i-t}),\widehat{\nabla} J(\check{\pmb{x}}^{(i,j)},\pmb\theta_i,\pmb\theta_{i-t})-\nabla J(\pmb\theta_{i-t})\right\rangle\right| \pmb{s}^{(i-t,j)},\pmb\theta_{i-t}\right]\right]\right| \nonumber\\
&\leq\E\left[\left|\left\langle \nabla J(\pmb\theta_{i-t}),\E\left[\left.\widehat{\nabla} J\left(\tilde{\pmb{x}}^{(i,j)},\pmb\theta_i,\pmb\theta_{i-t}\right)- \widehat{\nabla} J\left(\check{\pmb{x}}^{(i,j)},\pmb\theta_i,\pmb\theta_{i-t}\right)\right| \pmb{s}^{(i-t,j)},\pmb\theta_{i-t}\right]\right\rangle\right|\right] \nonumber\\
&\leq M\E\left[\left\Vert \E\left[\left.\widehat{\nabla} J\left(\tilde{\pmb{x}}^{(i,j)},\pmb\theta_i,\pmb\theta_{i-t}\right)- \widehat{\nabla} J\left(\check{\pmb{x}}^{(i,j)},\pmb\theta_i,\pmb\theta_{i-t}\right) \right| \pmb{s}^{(i-t,j)},\pmb\theta_{i-t}\right]\right\Vert \right]\label{eq1: appendix lemma: supportive lemma 8}
\end{align}
where Step~\eqref{eq1: appendix lemma: supportive lemma 8} holds by applying Cauchy–Schwarz inequality and Lemma~\ref{lemma: bounded of policy gradient}.
Then the conclusion immediately follows by applying Lemma~\ref{appendix lemma: supportive lemma 3}.
\endproof
\begin{lemma}\label{appendix lemma: supportive lemma 9}
Under Assumptions~\ref{assumption 2} and~\ref{assumption 3},
$\E[\Gamma(\check{\pmb{x}}^{(i,j)},\pmb\theta_i,\pmb\theta_{i-t})]=0$.
\end{lemma}
\proof
Conditional on $\pmb\theta_{i-t}$, the auxiliary sample is drawn from
$d^{\pi_{\pmb\theta_{i-t}}}$, invariant for the frozen-policy restart kernel.
Equation~\eqref{eq: policy gradient} therefore gives
$\E[g(\check{\pmb{x}}^{(i,j)}|\pmb\theta_{i-t})\mid\pmb\theta_{i-t}]=\nabla J(\pmb\theta_{i-t})$,
and the stated centered expectation is zero.
\endproof

\section{Proof of Auxillary Lemmas}
\label{appendix sec: proof of auxillary lemmas}

\begin{lemma}\label{auxillary lemma: parametric distance}
     Under Assumption~\ref{assumption 2} and \ref{assumption 3}, for any $k\geq i \geq t$ we have
     $$\E[\Vert \pmb\theta_{k}-\pmb\theta_{i-t}\Vert]\leq M \sum^{k}_{\ell=i-t}{\eta_{\ell}}.$$
\end{lemma}
\proof By applying Minkowski inequality, it holds
\begin{align}
    \E[\Vert \pmb\theta_{k}-\pmb\theta_{i-t}\Vert]&\leq \sum^{k}_{\ell=i-t}\eta_{\ell} \E\left[\frac{1}{|\mathcal{U}_\ell|n}
 \sum_{{\pmb\theta_h}\in \mathcal{U}_\ell}
 \sum^{n}_{j=1}
 \frac{\pi_{\pmb\theta_\ell}\left(\pmb{a}^{(h,j)}|\pmb{s}^{(h,j)}\right)}
 {\pi_{\pmb\theta_h}\left(\pmb{a}^{(h,j)}|\pmb{s}^{(h,j)}\right)}
\left\Vert g\left(\pmb{x}^{(h,j)}|\pmb\theta_{\ell}\right)\right\Vert \right] \nonumber\\
&\leq M\sum^{k}_{\ell=i-t}\eta_{\ell} \frac{1}{|\mathcal{U}_\ell|n}
 \sum_{{\pmb\theta_h}\in \mathcal{U}_\ell}
 \sum^{n}_{j=1}\E\left[
 \frac{\pi_{\pmb\theta_\ell}\left(\pmb{a}^{(h,j)}|\pmb{s}^{(h,j)}\right)}
 {\pi_{\pmb\theta_h}\left(\pmb{a}^{(h,j)}|\pmb{s}^{(h,j)}\right)}
 \right] \label{eq3: appendix lemma: supportive lemma 2}\\
& \leq M\sum^{k}_{\ell=i-t}{\eta_{\ell}} \nonumber
\end{align}
where Step~\eqref{eq3: appendix lemma: supportive lemma 2} holds by Lemma~\ref{lemma: bounded of policy gradient}.
\endproof

\begin{lemma}\label{supportive lemma: mean sequence Big O}
Suppose $\{a_k\}$ is a non-negative, bounded sequence. 
Let $U_a$ denote the bound of the sequence. Then
for any large enough $K$, let $\tau= f(K) <K$ (i.e. $f(\cdot)$ is a positive increasing function) and we have:
\begin{equation*}
    \frac{1}{K}\sum^K_{k=1}a_k\leq\frac{\tau-1}{K}U_a+\frac{1}{1+K-\tau}\sum^K_{k=\tau}a_k.
\end{equation*}
\end{lemma}
\proof
The conclusion is obtained by observing that
\begin{equation*}
    \frac{1}{K}\sum^T_{k=1}a_k\leq \frac{1}{K}\left((\tau-1)U_a+\sum^K_{k=\tau}a_k\right)=\frac{\tau-1}{K}U_a+\frac{1}{K}\sum^K_{k=\tau}a_k \leq \frac{\tau-1}{K}U_a+\frac{1}{1+K-\tau}\sum^K_{k=\tau}a_k.
\end{equation*}
\endproof

\section{Justification of MBB Estimator of Gradient Variance}
\label{appendix sec: verification of variance estimator}
According to \citet[Theorem 3.1]{lahiri2003resampling}, there are two conditions required for the consistency of the MBB variance estimators \eqref{eq.LR-VarEst} and \eqref{eq.PG-VarEst}: there exists a $\delta >0$ such that (1) $\E[\Vert g_{i,j}\Vert^{2+\delta}]<\infty$ and that (2) $\sum^{\infty}_{j=1}\alpha(j)^{\delta/2+\delta} <\infty$ where $\alpha(j)$ represents the strong mixing coefficient; see the definition in \citet[Definition 3.1]{lahiri2003resampling}. To ensure the consistency of the MBB variance estimator, we introduce the following boundedness assumption on the likelihood ratio:
\begin{equation*}
    \frac{
\pi_{\pmb\theta_k}\left(\pmb{a}|\pmb{s}\right)
}{\pi_{\pmb\theta_i}\left(\pmb{a}|\pmb{s}\right)} \le U_L, \forall (\pmb{s},\pmb{a}) \in \mathcal{S} \times \mathcal{A}.
\end{equation*}
Under this assumption, we will show that both conditions hold when $\delta=2/3$.

\noindent \textbf{Condition (1)} holds due to $$
\E[\Vert g_{i,j}\Vert^{8/3}]\leq\E\left[\left|\frac{
\pi_{\pmb\theta_k}\left(\pmb{a}^{(i,j)}|\pmb{s}^{(i,j)}\right)
}{\pi_{\pmb\theta_i}\left(\pmb{a}^{(i,j)}|\pmb{s}^{(i,j)}\right)}\right|^{8/3}\Vert g\left(\pmb{s}^{(i,j)},\pmb{a}^{(i,j)}|\pmb\theta_k\right)\Vert^{8/3}\right]\leq U_L^{8/3} M^{8/3} < \infty$$ by the bounded-likelihood-ratio condition introduced above and Lemma~\ref{lemma: bounded of policy gradient}.

\vspace{0.05in}
\noindent \textbf{Condition (2)}: For fixed behavior policy $\pi_{\pmb\theta_i}$, initialize the state--action chain in $d^{\pi_{\pmb\theta_i}}$ and use transition kernel
$\widetilde{\P}(\dd\pmb{s}'|\pmb{s},\pmb{a})\pi_{\pmb\theta_i}(\dd\pmb{a}'|\pmb{s}')$.
This chain is stationary, and Assumption~\ref{assumption 3} gives the following mixing bound:
\small
$$\left\Vert \P(\pmb{s}^{(i,j)}\in \cdot |\pmb{s}^{(i,1)}=\pmb{s})\pi_{\pmb\theta_i}(\cdot|\pmb{s}^{(i,j)})-d^{\pi_{\pmb\theta_i}}(\cdot,\cdot) \right\Vert_{TV}= \left\Vert \P(\pmb{s}^{(i,j)}\in \cdot |\pmb{s}^{(i,1)}=\pmb{s})-d^{\pi_{\pmb\theta_i}}(\cdot) \right\Vert_{TV}\leq \kappa_0\kappa^{j-1}. $$\normalsize
It is known that for stationary Markov chains,
the geometric ergodicity implies the $\beta$-mixing \cite[Thm 21.19]{bradley2007introduction} and the $\beta$-mixing implies strong-mixing \citep{bradley2007basic}.

\begin{sloppypar}
Since the gradient estimates $g_{i,j}$ are measurable functions of state-action pair $(\pmb{s}^{(i,j)}, \pmb{a}^{(i,j)})$, for the LR gradient estimate based on samples from $i$-th iteration, the sigma-algebra $\sigma(\cdot)$, generated by the gradient estimates is a subset of that generated by the state-action samples, i.e., $\sigma(\{g_{i,j}\}_{j=1}^n)\subset\sigma(\{(\pmb{s}^{(i,j)},\pmb{a}^{(i,j)})\}_{j=1}^n)$. Then, the sequence $\{g_{i,j}\}_{j= 1}^n$ is strong-mixing due to the fact as $n\rightarrow \infty$
\begin{align}
    \alpha(n) &=\sup_j
\alpha\left(\sigma\left(\{g_{i,j^\prime}\}_{j^\prime=1}^j\right),\sigma\left(\{g_{i,j^\prime} \}_{j^\prime=j+n}^\infty\right)\right)\nonumber\\
&\leq \sup_j
\alpha\left(\sigma\left(\{(\pmb{s}^{(i,j)},\pmb{a}^{(i,j^\prime)})\}_{j^\prime=1}^j\right),\sigma\left(\{(\pmb{s}^{(i,j^\prime)},\pmb{a}^{(i,j^\prime)})\}_{j^\prime=j+n}^\infty\right)\right) \nonumber \\
&\leq \sup_j
\beta\left(\sigma\left(\{(\pmb{s}^{(i,j)},\pmb{a}^{(i,j^\prime)})\}_{j^\prime=1}^j\right),\sigma\left(\{(\pmb{s}^{(i,j^\prime)},\pmb{a}^{(i,j^\prime)})\}_{j^\prime=j+n}^\infty\right)\right)  \leq \kappa_0\kappa^n\label{eq: verification of variance estimator}
\end{align}
where first inequality in~\eqref{eq: verification of variance estimator} holds due to Eq.~(1.11) in \cite{bradley2007basic} and second inequality in~\eqref{eq: verification of variance estimator} holds due to \citet[Thm 21.19]{bradley2007introduction}. Therefore, \textbf{Condition (2)} follows that $$\lim_{n\rightarrow\infty}  \sum^{n}_{j=1}\alpha(j)
= \lim_{n\rightarrow\infty} \kappa_0 \frac{\kappa -\kappa^{n+1}}{1-\kappa} = \frac{\kappa_0 \kappa}{1-\kappa}< \infty.$$
\end{sloppypar}

\section{Justification for the Variance Ratio Approximation}\label{appendix sec: justification for vr approximation}
\textbf{Proposition~\ref{prop: approximation of variance ratio}.}
\textit{Under Assumptions \ref{assumption 2} and \ref{assumption 3}, for any $t$ such that $t <i\leq k$, if $t=o(k^{r/2})$ and $B_k=o(k^{r/2})$, the total variance ratio of the individual LR policy gradient estimator and PG estimator has the approximation 
\begin{equation*}
     \frac{\Tr\left(\Var
 \left[ \widehat{\nabla} J^{LR}_{i,k} \right]\right)}{\Tr\left(\Var\left[\widehat{\nabla} J^{PG}_k\right]\right)}\approx e^{\E\left[\KL\left(\pi_{\pmb\theta_k}(\cdot|\pmb{s})\Vert \pi_{\pmb\theta_i}(\cdot|\pmb{s})\right)\right]}\left(1+\zeta_k^{-1}\right)-\zeta^{-1}_k
\end{equation*}
where $\zeta_k={\Tr\left(\Var\left[\widehat{\nabla} J^{PG}_k\right]\right)}/{\left\Vert\E\left[\widehat{\nabla} J^{PG}_k\right]\right\Vert^2}$ is the relative variance.}

\proof{} 
Consider the total variance of LR policy gradient estimator for a state-action pair $(\pmb{s}^{i,j},\pmb{a}^{i,j})$. Let $\E_i[\cdot]=\E_{\pmb{s}\sim \P(\pmb{s}^{(i,j)}\in\cdot|\pmb{s}_1), \pmb{a}\sim\pi_{\pmb\theta_i}(\cdot|\pmb{s})}[\cdot]$. We have
 \begin{align}
     \Tr\left(\Var
 \left[ \widehat{\nabla} J^{LR}_{i,k} \right]\right) &= \E_{i}\left[\left\Vert \frac{\pi_{\pmb\theta_k}(\pmb{a}|\pmb{s})}{\pi_{\pmb\theta_i}(\pmb{a}|\pmb{s})} g(\pmb{s},\pmb{a}|\pmb\theta_k)\right\Vert^2\right] - \left\Vert\E_{i}\left[\frac{\pi_{\pmb\theta_k}(\pmb{a}|\pmb{s})}{\pi_{\pmb\theta_i}(\pmb{a}|\pmb{s})} g(\pmb{s},\pmb{a}|\pmb\theta_k) \right]\right\Vert^2 \nonumber\\
  &= \E_{\pmb{s}\sim \P\left(\pmb{s}^{(i,j)}\in\cdot|\pmb{s}_1\right), \pmb{a}\sim\pi_{\pmb\theta_k}(\cdot|\pmb{s})}
  \left[\frac{\pi_{\pmb\theta_k}(\pmb{a}|\pmb{s})}{\pi_{\pmb\theta_i}(\pmb{a}|\pmb{s})} \Vert 
 g(\pmb{s},\pmb{a}|\pmb\theta_k)\Vert^2\right] 
 \nonumber\\
 &\quad - \left\Vert\E_{\pmb{s}\sim \P\left(\pmb{s}^{(i,j)}\in\cdot|\pmb{s}_1\right), \pmb{a}\sim\pi_{\pmb\theta_k}(\cdot|\pmb{s})}[g(\pmb{s},\pmb{a}|\pmb\theta_k)]\right\Vert^2 \nonumber\\
&= \exp\left\{\log\left(\E_{\pmb{s}\sim \P\left(\pmb{s}^{(k,j)}\in\cdot|\pmb{s}_1\right), \pmb{a}\sim\pi_{\pmb\theta_k}(\cdot|\pmb{s})}\left[\frac{\pi_{\pmb\theta_k}(\pmb{a}|\pmb{s})}{\pi_{\pmb\theta_i}(\pmb{a}|\pmb{s})} \Vert 
 g(\pmb{s},\pmb{a}|\pmb\theta_k)\Vert^2\right] \right)\right\}\nonumber\\
 &\quad - \left\Vert\E_{\pmb{s}\sim \P\left(\pmb{s}^{(k,j)}\in\cdot|\pmb{s}_1\right), \pmb{a}\sim\pi_{\pmb\theta_k}(\cdot|\pmb{s})}[g(\pmb{s},\pmb{a}|\pmb\theta_k)]\right\Vert^2
 \label{eq: LR policy gradient variance}
 \end{align}
 .
 
By a second-order Taylor expansion of $\E[\log(X)]$ around $\mu = \E[X]$, we have
\begin{equation}\label{eq: taylor expansion of log}
    \E[\log(X)]\approx \E\left[\log \E[X] +\frac{X-\E[X]}{\E[X]} -\frac{(X-\E[X])^2}{2\E[X]^2}\right] =\log\E[X]-\frac{\Var[X]}{2\E[X]^2}.
\end{equation}
The Taylor approximation in~\eqref{eq: taylor expansion of log} is used as a local second-order approximation rather than as an exact identity. More precisely, for a positive random variable $X$ with mean $\mu=\E[X]$ bounded away from zero, Taylor's theorem gives $\E[\log (X)]=\log \E[X]-\frac{\Var(X)}{2\E[X]^2}+R_X$, where the remainder $R_X$ is determined by higher-order terms. In the present LR setting, this expansion identifies the leading dependence on the expectation 
$\E[X]$ 
and the second-order correction $\frac{\Var[X]}{2\E[X]^2}$, 
 but Assumptions~\ref{assumption 2} and~\ref{assumption 3} do not provide a uniform finite-sample bound on $R_X$. The approximation is therefore interpreted locally, for behavior policies retained close to the target policy, rather than as an exact identity. For the CLR estimator, an analogous local expansion can be applied after replacing the likelihood ratio by its clipped counterpart.

Let $\E_k[\cdot]=\E_{\pmb{s}\sim \P\left(\pmb{s}^{(k,j)}\in\cdot|\pmb{s}_1\right), \pmb{a}\sim\pi_{\pmb\theta_k}(\cdot|\pmb{s})}[\cdot]$, and let $\Var_k[\cdot]$ denote the variance under the same probability measure. Define $q_{i,k}(\pmb{s},\pmb{a})=\frac{\pi_{\pmb\theta_k}(\pmb{a}|\pmb{s})}{\pi_{\pmb\theta_i}(\pmb{a}|\pmb{s})} \Vert g(\pmb{s},\pmb{a}|\pmb\theta_k)\Vert^2$ and $q_{k,k}(\pmb{s},\pmb{a})=\Vert g(\pmb{s},\pmb{a}|\pmb\theta_k)\Vert^2$.
By applying Taylor approximation~\eqref{eq: taylor expansion of log} to the first term of \eqref{eq: LR policy gradient variance}, 
 we get
 \begin{align*}
     \Tr\left(\Var
 \left[ \widehat{\nabla} J^{LR}_{i,k} \right]\right) &\approx \exp\left\{\E_{k}\left[\log\left(\frac{\pi_{\pmb\theta_k}(\pmb{a}|\pmb{s})}{\pi_{\pmb\theta_i}(\pmb{a}|\pmb{s})}\right)\right] +\E_{k}\left[\log\left(\Vert 
 g(\pmb{s},\pmb{a}|\pmb\theta_k)\Vert^2 \right)\right]+\frac{\Var_k[q_{i,k}(\pmb{s},\pmb{a})]}{2\E_k[q_{i,k}(\pmb{s},\pmb{a})]^2}\right\} \nonumber\\
 &\quad - \left\Vert\E_{k}[g(\pmb{s},\pmb{a}|\pmb\theta_k)]\right\Vert^2.
 \end{align*}
Applying the Taylor approximation~\eqref{eq: taylor expansion of log} to $\E_{k}\left[\log\left(\Vert 
 g(\pmb{s},\pmb{a}|\pmb\theta_k)\Vert^2 \right)\right]$ gives
\begin{align*}
     \Tr\left(\Var
 \left[ \widehat{\nabla} J^{LR}_{i,k} \right]\right) &\approx \exp\left\{\E_{k}\left[\log\left(\frac{\pi_{\pmb\theta_k}(\pmb{a}|\pmb{s})}{\pi_{\pmb\theta_i}(\pmb{a}|\pmb{s})}\right)\right] +\log\left(\E_{k}\left[\Vert 
 g(\pmb{s},\pmb{a}|\pmb\theta_k)\Vert^2 \right]\right)\right\} \mathcal{G} - \left\Vert\E_{k}[g(\pmb{s},\pmb{a}|\pmb\theta_k)]\right\Vert^2 \nonumber\\
&\approx \exp\left\{\E_{k}\left[\log\left(\frac{\pi_{\pmb\theta_k}(\pmb{a}|\pmb{s})}{\pi_{\pmb\theta_i}(\pmb{a}|\pmb{s})}\right)\right] \right\} \E_{k}\left[\Vert 
 g(\pmb{s},\pmb{a}|\pmb\theta_k)\Vert^2 \right] \mathcal{G} - \left\Vert\E_{k}[g(\pmb{s},\pmb{a}|\pmb\theta_k)]\right\Vert^2
\end{align*}
where $\log\mathcal{G}=\frac{\Var_k[q_{i,k}(\pmb{s},\pmb{a})]}{2\E_k[q_{i,k}(\pmb{s},\pmb{a})]^2}-\frac{\Var_k[q_{k,k}(\pmb{s},\pmb{a})]}{2\E_k[q_{k,k}(\pmb{s},\pmb{a})]^2}=\frac{1}{2}\left(\frac{\E_k[q_{i,k}(\pmb{s},\pmb{a})^2]}{\E_k[q_{i,k}(\pmb{s},\pmb{a})]^2}-\frac{\E_k[q_{k,k}(\pmb{s},\pmb{a})^2]}{\E_k[q_{k,k}(\pmb{s},\pmb{a})]^2}\right)$ is the second-order approximation error. 

The convergence result in Theorem~\ref{convergence theorem}, together with Corollary~\ref{cor: convergence rate}, shows that under the stated choices of $t$, $B_K$, and $\eta_k$, $\frac{1}{K}\sum_{k=1}^K\E[\Vert \nabla J(\pmb\theta_k)\Vert^2]\to0$, i.e., PG-VRER converges to stationarity in the averaged expected squared-gradient sense. Thus, the late-stage training regime should be interpreted as one in which the iterates are close to stationary in this averaged sense, rather than as requiring pointwise convergence of every $\nabla J(\pmb\theta_k)$. 
Moreover, every reused policy satisfies $k-i\le B_k$. By Lemma~\ref{auxillary lemma: parametric distance},
$\E[\Vert\pmb\theta_k-\pmb\theta_i\Vert]\leq M\sum_{\ell=i}^{k}\eta_\ell$. Since $\eta_\ell=\eta_1\ell^{-r}$ and $B_k=o(k^{r/2})$ implies $B_k=o(k)$, the indices $\ell=i,\ldots,k$ are asymptotically of order $k$, yielding $\sum_{\ell=i}^{k}\eta_\ell=O(B_k k^{-r})$ and $\E[\Vert\pmb\theta_k-\pmb\theta_i\Vert]=O(B_k k^{-r})$. Thus, the behavior and target policies become locally close in expectation under the buffer-size condition in Proposition~\ref{prop: approximation of variance ratio}. This local closeness motivates treating the normalized second-order correction $\mathcal{G}$ as close to one in the late-stage regime. We emphasize, however, that $\mathcal{G}\approx1$ is a local approximation and does not follow uniformly from Assumptions~\ref{assumption 2} and~\ref{assumption 3}. 
Therefore, defining the relative variance $\zeta_k={\Tr\left(\Var\left[\widehat{\nabla} J^{PG}_k\right]\right)}/{\left\Vert\E\left[\widehat{\nabla} J^{PG}_k\right]\right\Vert^2}$, the total variance ratio can be approximated as
\begin{align*}
     \frac{\Tr\left(\Var
 \left[ \widehat{\nabla} J^{LR}_{i,k} \right]\right)}{\Tr\left(\Var\left[\widehat{\nabla} J^{PG}_k\right]\right)}& \approx e^{\E\left[\mbox{KL}\left(\pi_{\pmb\theta_k}(\cdot|\pmb{s})\Vert \pi_{\pmb\theta_i}(\cdot|\pmb{s})\right)\right]}\frac{\E_{k}[\left\Vert g(\pmb{s},\pmb{a}|\pmb\theta_k)\right\Vert^2]}{\Tr\left(\Var\left[\widehat{\nabla} J^{PG}_k\right]\right)} - \frac{\left\Vert\E_{k}[g(\pmb{s},\pmb{a}|\pmb\theta_k)]\right\Vert^2}{\Tr\left(\Var\left[\widehat{\nabla} J^{PG}_k\right]\right)} \nonumber\\
& \approx e^{\E\left[\mbox{KL}\left(\pi_{\pmb\theta_k}(\cdot|\pmb{s})\Vert \pi_{\pmb\theta_i}(\cdot|\pmb{s})\right)\right]}\left(1+\zeta_k^{-1}\right)-\zeta^{-1}_k \qed
\end{align*}

\end{appendices}

\end{document}